\documentclass[lettersize,journal]{IEEEtran}
\usepackage{amsmath,amsfonts,amssymb}
\usepackage{algorithmic}
\usepackage{graphicx}
\usepackage{array}
\usepackage{bm}
\usepackage{textcomp}
\usepackage{multirow}
\usepackage{stfloats}
\usepackage{url}
\usepackage{balance}
\usepackage{verbatim}
\usepackage{xcolor}
\usepackage{cite}
\usepackage[font={small}]{caption}
\usepackage[colorlinks, linkcolor=black, anchorcolor=black, citecolor=black]{hyperref}
\usepackage{siunitx}
\usepackage{subcaption} 
\usepackage{threeparttable}

\newtheorem{mydef}{Definition}
\newtheorem{theorem}{Theorem}
\newtheorem{remark}{Remark}
\newtheorem{assumption}{Assumption}
\newenvironment{proof}{{\indent \indent \it Proof:\quad}}{\hfill$\blacksquare$\par}

\begin{document}
\title{Dynamic Collision Avoidance Using Velocity Obstacle-Based Control Barrier Functions}
\author{Jihao Huang$^{1}$, Jun Zeng$^{3}$, Xuemin Chi$^{1}$, Koushil Sreenath$^{3}$, Zhitao Liu$^{2\dagger}$, Hongye Su$^{1}$
\thanks{This work was supported in part by National Key R\&D Program of China (Grant NO. 2021YFB3301000); National Natural Science Foundation of China (NSFC:62173297), Zhejiang Key R\&D Program (Grant NO. 2022C01035), Fundamental Research Funds for the Central Universities (NO.226-2022-00086). ($^\dagger$Corresponding author)}
\thanks{
$^{1}$ State Key Laboratory of Industrial Control Technology, Institute of Cyber-Systems and Control, Zhejiang University, Hangzhou, China {\tt\footnotesize \{jihaoh, chixuemin, hysu\}@zju.edu.cn}.
$^{2} $Institute of Intelligence Science and Engineering, Shenzhen Polytechnic University, Shenzhen, China {\tt\footnotesize ztliu@zju.edu.cn}.
$^{3}$ Hybrid Robotics Group at the Department of Mechanical Engineering, UC Berkeley, USA {\tt\footnotesize \{zengjunsjtu, koushils\}@berkeley.edu}.
}%
}

\markboth{Journal of \LaTeX\ Class Files,~Vol.~14, No.~1, February~2025}%
{Shell \MakeLowercase{\textit{et al.}}: A Sample Article Using IEEEtran.cls for IEEE Journals}

\maketitle
\begin{abstract}
Designing safety-critical controllers for acceleration-controlled unicycle robots is challenging, as control inputs may not appear in the constraints of control Lyapunov functions (CLFs) and control barrier functions (CBFs), leading to invalid controllers. 
Existing methods often rely on state-feedback-based CLFs and high-order CBFs (HOCBFs), which are computationally expensive to construct and fail to maintain effectiveness in dynamic environments with fast-moving, nearby obstacles.
To address these challenges, we propose constructing velocity obstacle-based CBFs (VOCBFs) in the velocity space to enhance dynamic collision avoidance capabilities, instead of relying on distance-based CBFs that require the introduction of HOCBFs.
Additionally, by extending VOCBFs using variants of VO, we enable reactive collision avoidance between robots.
We formulate a safety-critical controller for acceleration-controlled unicycle robots as a mixed-integer quadratic programming (MIQP), integrating state-feedback-based CLFs for navigation and VOCBFs for collision avoidance.
To enhance the efficiency of solving the MIQP, we split the MIQP into multiple sub-optimization problems and employ a decision network to reduce computational costs. 
Numerical simulations demonstrate that our approach effectively guides the robot to its target while avoiding collisions.
Compared to HOCBFs, VOCBFs exhibit significantly improved dynamic obstacle avoidance performance, especially when obstacles are fast-moving and close to the robot. 
Furthermore, we extend our method to distributed multi-robot systems.
\end{abstract}

\section{Introduction}
\subsection{Motivation}
\IEEEPARstart{W}{ith} advancements in robotics, robots are increasingly used in applications such as autonomous driving, delivery services, and industrial production~\cite{ferr2023dis}.
Ensuring reliable collision avoidance is essential for these applications to ensure safety and prevent potential losses.
Dynamic collision avoidance, which requires considering both the positions and the velocities of obstacles, has gained significant attention.
Recently, control barrier function (CBF)-based approaches are widely used in safety-critical controllers~\cite{ames2014control, ames2016control, zeng2021feasibility, xiao2021bridge, alan2022disturbance, molnar2023safe, huang2023obs, sun2024safety, dong2024safety}, which prioritize safety over other aspects, such as tracking.
When designing CBFs for acceleration-controlled unicycle robots based on position and Euclidean distance, it is necessary to introduce high-order CBFs (HOCBFs)~\cite{xiao2019control} to ensure that all control inputs explicitly appear in the CBF constraints.
However, designing appropriate HOCBFs is computationally expensive and is less effective for dynamic obstacle avoidance. 
Therefore, to avoid the use of HOCBFs and achieve better dynamic obstacle avoidance performance, designing CBFs in the velocity space is an effective solution.
The velocity obstacle (VO)~\cite{fiorini1998motion}-based methods are widely used for local collision avoidance between circular robots and obstacles by selecting velocities outside the VO, with several works~\cite{zhang2022velocity} showing promising results by formulating VO-based obstacle avoidance constraints in optimization problems. 
To achieve better dynamic obstacle avoidance performance while ensuring real-time capability, we propose constructing VO-based CBFs (VOCBFs).
Additionally, to mitigate VO deadlocks and enable navigation for acceleration-controlled robots, we design state-feedback-based control Lyapunov functions (CLFs). 
By combining VOCBFs and CLFs, the robot can move toward its target while ensuring reliable real-time collision avoidance.

\subsection{Related Works}
\subsubsection{Safety-Critical Controller}
Safety-critical controllers that unify CLFs for stability and CBFs for safety through quadratic programming (CLF-CBF-QP) are proposed in~\cite{ames2014control, ames2016control, zeng2021feasibility} to achieve adaptive cruise control (ACC).
These controllers prioritize safety over other objectives like tracking, relaxing CLF constraints to satisfy the CBF constraints when they conflict.
Furthermore, these safety-critical controllers have been applied to robotics, using CLFs for navigation and CBFs for safety guarantees\cite{wu2016safety, he2021rule}.
While many works focus on point-mass robot models with static obstacles, extensions to unicycle models are presented in~\cite{xiao2021bridge, huang2023obs}.
However, a nominal CBF cannot manage both linear and angular velocities of the unicycle model to avoid collisions, as the CBF is designed based on the position coordinates of the robot without considering the robot's orientation, resulting in lacking control over steering~\cite{huang2023obs}. 
To address this, Huang~\emph{et al.}~\cite{huang2023obs} propose designing the CBF based on the center of the robot's rear axle and its orientation, allowing the safety-critical controller to manage both linear and angular velocities to navigate the robot to its destination and avoid collisions with dynamic obstacles.
As a result, it overcomes the limitation of traditional CBFs, which cannot account for steering to avoid collisions.

The extensions mentioned above mainly focus on velocity-controlled unicycle robot models, where the control inputs are linear and angular velocities.
For acceleration-controlled unicycle robot models, safety-critical controllers are often designed using state-feedback-based CLFs and HOCBFs~\cite{xiao2019control, xiao2022control, xiao2022high, xiao2022ada}.
The difficulty in designing CLFs and CBFs for high-order systems lies in the fact that CLFs are usually invalid and CBFs with a relative degree greater than one.
The relative degree of a sufficiently differentiable function is defined as the number of times we need to differentiate it along the robot’s dynamics until the control inputs explicitly appear~\cite{xiao2019control}.
Xiao~\emph{et al.}~\cite{xiao2022ada} propose two CLFs to achieve navigation: one state-feedback-based CLF to adjust the robot's orientation, and another to adjust its velocity toward the desired velocity.
Moreover, since the relative degree of the CBF based on positions using Euclidean distance is two, \cite{xiao2022high} proposes using HOCBF to generate valid collision avoidance constraints for the acceleration-controlled unicycle model, ensuring that the nominal CBF with relative degree two is non-negative.
Since HOCBFs require all but the final derivative of the nominal CBF, where the control inputs explicitly appear, to be non-negative, the safety guarantee is only provided using a subset of the original safe set, leading to conservative control performance~\cite{thontepu2022control}. 
Moreover, computing proper constraint candidates with appropriate penalty weights and parameters for HOCBFs is costly and lacks geometrical intuition.
Additionally, Ames~\emph{et al.}~\cite{ames2021integral} propose the integral CBF (I-CBF), which guarantees safety in both states and inputs while minimally modifying a nominal dynamically defined controller.
When the control inputs do not appear in the derivative of the original CBF, I-CBF also constructs a valid CBF following the concept of HOCBF.
Consequently, it faces similar challenges as HOCBF.
To address these issues and avoid the need for HOCBFs, we propose constructing CBFs in the velocity space for the acceleration-controlled unicycle models, offering a efficient and effective solution for dynamic collision avoidance.
Moreover, unlike I-CBF~\cite{ames2021integral}, which relies on a dynamically defined tracking controller for navigation, our approach integrates stated-feedback-based CLFs for navigation and CBFs for collision avoidance within a unified framework.

\subsubsection{Velocity Obstacle}
Collision cone-based methods~\cite{fiorini1998motion,van2008reciprocal, snape2011hybrid, berg2011reciprocal} are commonly used for local collision avoidance between circular robots and obstacles by defining a geometric conic set in the robot's velocity space.
The collision cone considers the relative velocity between the robot and the obstacle, while VO is a specific type of collision cone that focuses on the robot’s velocity.
Additionally, VO-based methods has been extended to polytopic-shaped robots and obstacles~\cite{huang2023velocity}.
Under the assumption that both the robot and the obstacle maintain their current velocities, if the robot's current velocity is within the VO induced by the obstacle, then it will collide with the obstacle at some future moment; 
conversely, obstacle avoidance is guaranteed if the robot’s current velocity is outside the VO.
Therefore, the robot can avoid collisions with all obstacles in the environment by selecting a velocity outside any VO induced by all obstacles.
Since VO explicitly considers the velocity of the obstacle and defines the set of unsafe velocities for the robot, some works~\cite{Cheng2017decentralized,zhang2022velocity} use VO or its variants to formulate the collision avoidance constraints for acceleration-controlled robot models.
Zhang~\emph{et al.}~\cite{zhang2022velocity} formulate a constrained nonlinear model predictive control (NMPC) problem to achieve navigation and collision avoidance in distributed multi-robot systems.
This works employs the disjunction of two linear constraints to represent the safe velocity set of the robot, i.e., the complement set of the VO, with each linear constraint requiring that the robot’s velocity lies within a half-space. 
However, the collision avoidance constraint that the robot’s velocity should fall within the safe velocity set cannot be directly represented in terms of single constraint, so integer variables are introduced to ensure that at least one of the two linear constraints is satisfied, converting the problem into a mixed-integer nonlinear programming (MINLP) problem.
When the prediction horizon is long, solving the MINLP problem becomes time-consuming, compromising real-time performance guarantees.
To address this, we propose constructing VOCBFs based on the constraints in~\cite{zhang2022velocity}, transforming the optimization problem into a QP, thereby ensuring real-time performance.
In addition, \cite{thontepu2022control} proposes constructing a collision cone-based CBF (C3BF) for the acceleration-controlled unicycle robot model, and real-time performance is ensured as the optimization problem is formulated as a QP.
However, C3BF~\cite{thontepu2022control} is constructed based on the relative velocity between the robot and obstacles, and when applied to distributed multi-robot systems, it neglects the reactive nature between robots, treating them as obstacles.
In contrast, VOCBF is constructed based on the robot's velocity and allows for the incorporation of various VO variants which consider the reactive nature among robots, making it more suitable for distributed multi-robot systems.

\subsection{Contributions}
In this paper, we propose a safety-critical controller in the form of CLF-VOCBF-MIQP (mixed-integer QP) for the acceleration-controlled unicycle robot model. The key contributions of our work are as follows:
\begin{itemize}
    \item State-feedback-based CLFs are designed for navigation, with additional CLFs tailored for specific purposes.
    \item VOCBFs are designed for dynamic collision avoidance and are further extended with VO variants for improved applicability in distributed multi-robot systems.
    \item The constraints of CLFs, VOCBFs, and physical capabilities, along with the objective function, are formulated into a CLF-VOCBF-MIQP. Furthermore, we propose splitting the MIQP into multiple sub-optimization problems and employing a decision network to improve computational efficiency.
    \item Extensive numerical simulations are conducted to validate the effectiveness of our proposed approach, and better dynamic obstacle avoidance performance is demonstrated compared to HOCBF. 
    Furthermore, successful navigation for distributed multi-robot systems with guaranteed collision avoidance is achieved. 
\end{itemize}

The rest of this paper is organized as follows:
In Section~\ref{sec:pre}, we formally define the robot model and the problem studied in this paper, and review the concepts of CLF, CBF, and VO.
The detailed design of the safety-critical controller is outlined in Section~\ref{sec:con}, including the design of CLFs, VOCBFs, and the integrated controller.
Numerical simulation results are presented in Section~\ref{sec:num} to demonstrate the effectiveness of our approach.
Section~\ref{sec:conclu} concludes the paper.
\section{Preliminaries}
\label{sec:pre}
Consider an affine control system in the form of
\begin{equation}
    \dot{\bm{x}} = f(\bm{x}) + g(\bm{x})\bm{u},
    \label{eq:affine_system}
\end{equation}
where $\bm{x} \in \mathcal{D} \subset \mathbb{R}^n$, $\bm{u} \in \mathcal{U} \subset \mathbb{R}^m$, $\Delta \bm{u} \in \Delta \mathcal{U} \subset \mathbb{R}^m$
and $f:\mathcal{D} \to \mathbb{R}^n$ and $g:\mathcal{D} \to \mathbb{R}^{n \times m}$ are locally Lipschitz continuous on $\mathcal{D}$.
Moreover, $\mathcal{D}=\{\bm{x} \in \mathbb{R}^n \, \big| \, \bm{x}_{\text{min}} \leq \bm{x} \leq \bm{x}_{\text{max}}\}$, $\mathcal{U}=\{\bm{u} \in \mathbb{R}^m \, \big| \, \bm{u}_{\text{min}} \leq \bm{u} \leq \bm{u}_{\text{max}}\}$ and $\Delta \mathcal{U}=\{\Delta \bm{u} \in \mathbb{R}^m \, \big| \, \Delta \bm{u}_{\text{min}} \leq \Delta \bm{u} \leq \Delta \bm{u}_{\text{max}}\}$ denote the closed constraint sets of admissible states, control inputs and the changes in control inputs, respectively.
Here $\bm{x}_{\text{min}}$, $\bm{u}_{\text{min}}$, $\Delta \bm{u}_{\text{min}}$ and $\bm{x}_{\text{max}}$, $\bm{u}_{\text{max}}$, $\Delta \bm{u}_{\text{max}}$ represent the lower and upper bounds of $\bm{x}$, $\bm{u}$, $\Delta \bm{u}$, respectively.

\subsection{Problem Formulation}
\begin{figure} 
    \centering
    \includegraphics[width=0.7\linewidth]{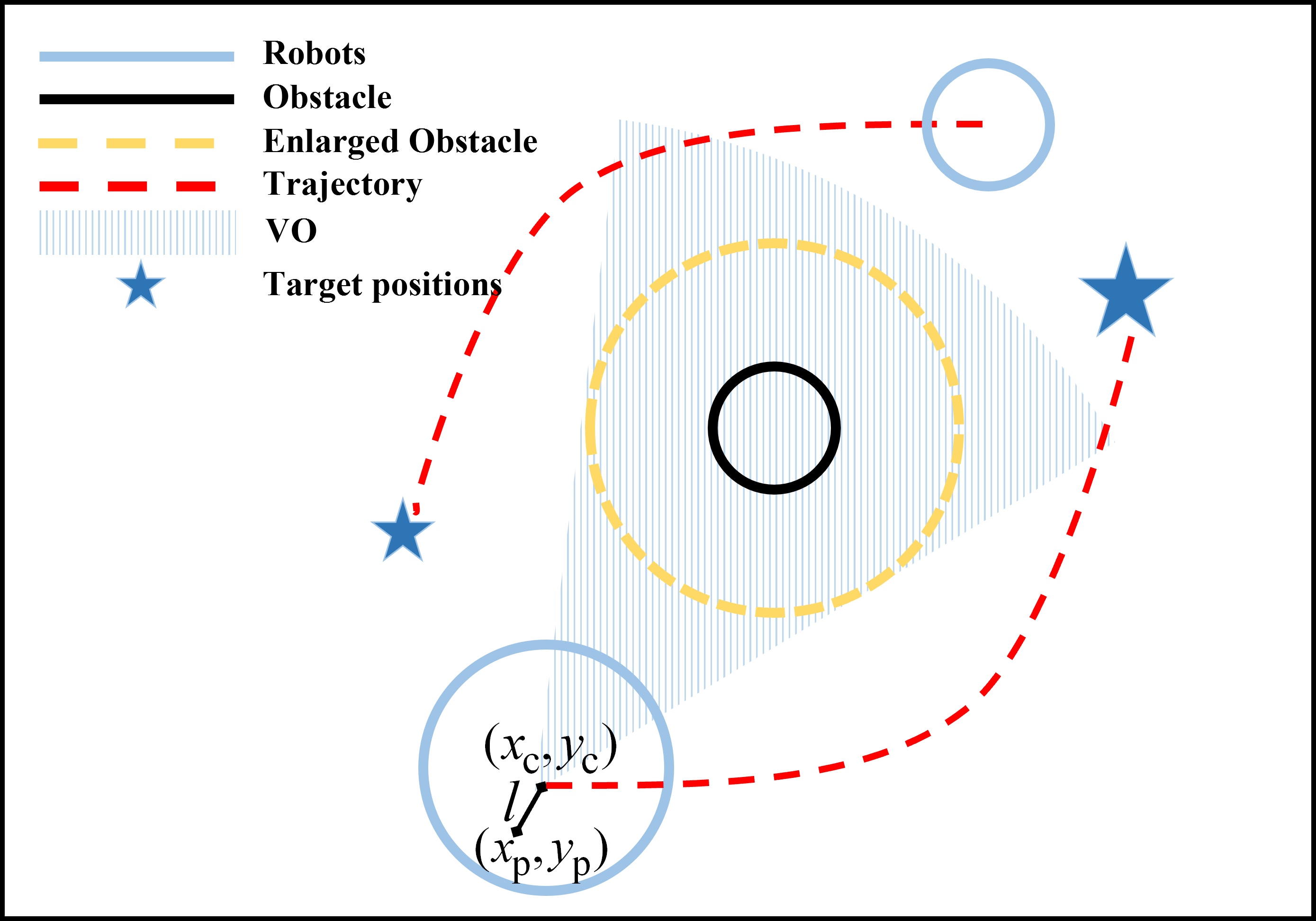}
    \caption{Collision avoidance between the robot and obstacles is achieved using a VO-based approach, where the VO is depicted as the shaded area.
    The obstacle is enlarged by inflating it with the robot's radius, which is then used to construct the VO. Here, $(x_\text{p}, y_\text{p})$ denote the coordinates of the rear axle axis, $(x_\text{c}, y_\text{c})$ represents the robot's center, and $l$ is the distance between them.}
    \label{fig:problem_formulation}
\end{figure}
Assume there is a set of $N$ robots sharing an environment with a set of $M$ static and dynamic obstacles, where both the robots and obstacles are circular-shaped.
For notations, subscripts $i$ and $j$ are used to distinguish different robots and obstacles, where each robot and obstacle are represented by $\text{R}_i, i \in \{0, 1, \dots, N - 1\}$ and $\text{O}_j, j\in \{0, 1, \dots, M-1\}$.
Moreover, we mainly focus on robots with an acceleration-controlled unicycle model in this work.
Since the circular-shaped robot has two axles, we model the acceleration-controlled unicycle model using the rear axle as:
\begin{equation} 
    \left[\begin{array}{c}
    \dot{x}_\text{p} \\
    \dot{y}_\text{p} \\
    \dot{\theta} \\
    \dot{v} \\
    \dot{w}
    \end{array}\right] = 
    \left[\begin{array}{c}
    v \cos\theta \\
    v \sin \theta \\
    w \\
    0 \\
    0
    \end{array}\right] + 
    \left[\begin{array}{cc}
    0 & 0 \\
    0 & 0 \\
    0 & 0 \\
    1 & 0 \\
    0 & 1 \\    
    \end{array}\right]
    \left[\begin{array}{c}
    a \\
    \alpha
    \end{array}\right],
    \label{eq:affine_robot_model}
\end{equation}
where $x_\text{p}, y_\text{p}$ denote the current coordinates of the rear axle axis, and the center $(x_\text{c}, y_\text{c})$ of the circular-shaped robots can be represented in terms of $(x_\text{p}, y_\text{p})$ as:
\begin{equation} 
    \left[\begin{array}{c}
        x_\text{c} \\
        y_\text{c} 
    \end{array}\right] = 
    \left[\begin{array}{c}
        x_\text{p} + l\cos\theta \\
        y_\text{p} + l\sin\theta 
    \end{array}\right],
    \label{eq:transform_relationship}
\end{equation}
where $l$ represents the distance between the rear axle axis and the center, as shown in Fig.~\ref{fig:problem_formulation}.
Moreover, $\theta$ denotes the orientation with respect to the $x$-axis; $v$ and $\omega$ denote the linear and angular velocities of the robot; and $a$ and $\alpha$ denote the linear and angular accelerations that control the motion of the robot.
Eq.~\eqref{eq:affine_robot_model} is in the form of~\eqref{eq:affine_system}, and we have $\bm{x}=[x_\text{p}, y_\text{p}, \theta, v, \omega]^{\top} \in \mathbb{R}^5$, $\bm{u} = [a, \alpha]^{\top} \in \mathbb{R}^2$.
$\bm{x}_{\text{min}}$, $\bm{u}_{\text{min}}$, $\Delta \bm{u}_{\text{min}}$, $\bm{x}_{\text{max}}$, $\bm{u}_{\text{max}}$ and $\Delta \bm{u}_{\text{max}}$ represent the physical limits of states and control inputs, such as maximum velocity, acceleration and the change rate of the acceleration.
Additionally, the double integrator dynamics model of the obstacle is defined as $\left[\begin{array}{c}
        \dot{x}_\text{o}, \;
        \dot{y}_\text{o}, \;
        \dot{v}_\text{ox},\;
        \dot{v}_\text{oy}
    \end{array}\right]^\top = 
    \left[\begin{array}{c}
        v_\text{ox}, \;
        v_\text{oy}, \;
        a_\text{ox}, \;
        a_\text{oy}
    \end{array}\right]^\top$,
where $x_\text{o}$ and $y_\text{o}$ denote the horizontal and vertical positions of the obstacle; $v_\text{ox}$, $v_\text{oy}$, $a_\text{ox}$ and $a_\text{oy}$ represent its horizontal and vertical velocities and accelerations.
The states and control inputs of the obstacle are denoted as $\bm{x}_{\text{O}} = [x_\text{o}, y_\text{o}, v_\text{ox}, v_\text{oy}]^{\top}$ and $\bm{u}_{\text{O}} = [a_\text{ox}, a_\text{oy}]^{\top}$, respectively.

In our work, each robot operates as an individual and can make decisions independently based on the environmental information to reach its target position while avoiding collisions with other robots and obstacles, as shown in Fig.~\ref{fig:problem_formulation}.
The task of each robot is formulated as a constrained optimal control problem (COCP):
\begin{equation}
\begin{aligned}
    J(\bm{u}(t)) &= \min_{\bm{u}}\int_{t_0}^{t_f} \frac{1}{2} \bm{u}(t)^{\top} H \bm{u}(t), \\
    \text{s.t.} \; & \dot{\bm{x}}(t) = f(\bm{x}(t)) + g(\bm{x}(t))\bm{u}(t), \\
    & \bm{x}(t) \in \mathcal{D}, \bm{x}(t) \in \mathcal{C}_\text{Safe}, \\
    & \bm{u}(t) \in \mathcal{U}, \Delta \bm{u}(t) \in \Delta \mathcal{U},
    \label{eq:ocp}
\end{aligned}
\end{equation}
where $H$ is a positive definite matrix, $\mathcal{C}_\text{Safe}$ represents the safe region in the configure space, with $t_0$ and $t_f$ denoting the initial and final times, respectively.
The objective function of the optimization problem is to minimize energy consumption.
Additionally, the COCP considers various constraints, including kinematic constraints such as the robot's dynamics, physical limits of states and control inputs and collision avoidance constraints. 

\subsection{Control Lyapunov Function}
\begin{mydef}[Class $\mathcal{K}$ and $\mathcal{K}_{\infty}$ functions\textnormal{~\cite{ames2019control}}]
    A Lipschitz continuous function $\mu: [0, a) \to [0, \infty), a > 0$ is said to belong to class $\mathcal{K}$ if it is strictly increasing and satisfies $\mu(0) = 0$.
    Moreover, this function is said to belong to class $\mathcal{K}_{\infty}$ if it belongs to class $\mathcal{K}$ and further satisfies $a = \infty$ and $\mu(b) \to \infty$ as $b \to \infty$.
\end{mydef}
\begin{mydef}[Control Lyapunov function\textnormal{~\cite{ames2019control}}]
    A continuously differentiable function $V: \mathcal{D} \to \mathbb{R}$ is a CLF for system \eqref{eq:affine_system} if it is positive definite and satisfies
    \begin{equation}
        \inf_{\bm{u} \in \mathcal{U}} [L_f V(\bm{x}) + L_g V(\bm{x}) \bm{u}] \leq -\gamma(V(\bm{x})),
        \label{eq:define_clf}
    \end{equation}
    where $L_f V(\bm{x}) := \frac{\partial V}{\partial \bm{x}} f(\bm{x})$ and $L_g V(\bm{x}) := \frac{\partial V}{\partial \bm{x}}g(\bm{x})$ are Lie-derivatives of $V(\mathbf{x})$, $\gamma(\cdot)$ belongs to class $\mathcal{K}$.
    \label{def:clf}
\end{mydef}
\begin{theorem}
Given a CLF $V(\bm{x})$ defined in Definition~\ref{def:clf}, any Lipschitz continuous controller $\bm{u} \in K_{\text{clf}}(\bm{x})$, with 
\begin{equation}
    K_{\text{clf}} (\bm{x}):=\{\bm{u} \in \mathcal{U}, L_f V(\bm{x}) + L_g V(\bm{x}) \bm{u} \leq -\gamma(V(\bm{x})) \},
    \label{eq:kclf}
\end{equation}
can stabilize the system \eqref{eq:affine_system} to its desired states~\cite{ames2019control}.
\end{theorem}

\subsection{Control Barrier Function}
\label{sec:cbf}
\begin{mydef}[Forward invariance set]
    A set $\mathcal{C} \subset \mathcal{D}$ is forward invariant with respect to \eqref{eq:affine_system} if for every initial state $\bm{x}_0 \in \mathcal{C}$, $\bm{x}(t) \in \mathcal{C}$, for $\bm{x}(t_0) = \bm{x}_0, \forall t \geq t_0$.
    The system~\eqref{eq:affine_system} is safe with respect to the set $\mathcal{C}$ if $\mathcal{C}$ is forward invariant.
    \label{def:safe_set}
\end{mydef}

Compared to CLFs, which are proposed to stabilize the system~\eqref{eq:affine_system} to an equilibrium point, CBFs are proposed to ensure the forward invariance of a set, as defined in Definition~\ref{def:safe_set}.
Consider a set $\mathcal{C} \subset \mathcal{D}$ defined as a zero-superlevel set of a continuously differentiable function $h$: $\mathcal{D} \to \mathbb{R}$, yielding
\begin{equation}
\begin{aligned}
    \mathcal{C} & = \{\bm{x} \in \mathcal{D} \subset \mathbb{R}^n : h(\bm{x}) \geq 0 \}, \\
    \partial \mathcal{C} & = \{\bm{x} \in \mathcal{D} \subset \mathbb{R}^n : h(\bm{x}) = 0 \}, \\
    \rm Int(\mathcal{C}) & = \{\bm{x} \in \mathcal{D} \subset \mathbb{R}^n : h(\bm{x}) > 0 \}.
    \label{eq:safe_set}
\end{aligned} 
\end{equation}
Throughout this paper, we refer to $\mathcal{C}$ as the safe set.
\begin{mydef}[Control barrier function\textnormal{~\cite{ames2019control}}]
    Suppose the set $\mathcal{C}$ defined in~\eqref{eq:safe_set} is the superlevel set of a continuously differentiable function $h$: $\mathcal{D} \to \mathbb{R}$, then $h$ is a CBF if there exists an extended class $\mathcal{K}_{\infty}$ function $\mu(\cdot)$ such that for the control system~\eqref{eq:affine_system}
    \begin{equation}
        \sup_{\bm{u} \in \mathcal{U}} [L_f h(\bm{x}) + L_g h(\bm{x})\bm{u}] \geq -\mu(h(\bm{x})),
        \label{eq:define_cbf}
    \end{equation}
    where $L_f h(\bm{x}) = \frac{\partial h(\bm{x})}{\partial \bm{x}}f(\bm{x})$ and $L_g h(\bm{x}) = \frac{\partial h(\bm{x})}{\partial \bm{x}}g(\bm{x})$ are Lie-derivatives of $h(\bm{x})$.
    \label{def:cbf}
\end{mydef}
\begin{theorem}
    If $h(\bm{x})$ is a CBF on $\mathcal{C}$ and $\frac{\partial h(\bm{x})}{\partial \bm{x}} \not= 0, \forall \bm{x} \in \partial \mathcal{C}$, then any Lipschitz continuous controller $\bm{u} \in K_{\text{cbf}}(\bm{x})$, with 
    \begin{equation}
        K_{\text{cbf}}(\bm{x}):=\{\bm{u} \in \mathcal{U}, L_f h(\bm{x}) + L_g h(\bm{x})\bm{u} \geq -\mu(h(\bm{x}))\}
        \label{eq:kcbf}
    \end{equation}
    can ensure the forward invariance of $\mathcal{C}$ and thus the safety of system~\eqref{eq:affine_system}~\cite{ames2019control}.
\end{theorem}

Since constraints~\eqref{eq:kclf} and~\eqref{eq:kcbf} are affine in the control inputs $\bm{u}$, many studies~\cite{ames2014control, ames2016control} propose a safety-critical controller which unifies CLFs for stability and CBFs for safety through a quadratic programming (CLF-CBF-QP), formulated as:
\noindent\rule{\columnwidth}{0.8pt}
\textbf{CLF-CBF-QP:}
\begin{subequations}
\begin{align}
    \min_{(\bm{u}, \delta) \in \mathbb{R}^{m + 1}} & \frac{1}{2}\bm{u}^{\top} H \bm{u} + p \delta^2 \label{eq:optimal_problem1} \\
    \text{s.t.} ~& L_f V(\bm{x}) + L_g V(\bm{x}) \bm{u} + \gamma(V(\bm{x})) \leq \delta, \label{eq:cons_clf1} \\ 
    & L_f h(\bm{x}) + L_g h(\bm{x})\bm{u} \geq -\mu(h(\bm{x})), \label{eq:cons_cbf1} \\
    & \bm{u} \in \mathcal{U}, \Delta \bm{u} \in \Delta \mathcal{U}, \label{eq:cons_u1}
\end{align}
\label{eq:clf_cbf_qp_optimal_problem}
\end{subequations}
\noindent\rule{\columnwidth}{0.8pt}
\noindent
where $H$ is a predefined positive definite matrix, $p > 0$ is a weighting factor to minimize $\delta$, where $\delta$ is a relaxation variable to relax the CLF constraint~\eqref{eq:cons_clf1} to ensure safety when the CLF constraint conflicts with the CBF constraint~\eqref{eq:cons_cbf1}. 
The objective function~\eqref{eq:optimal_problem1} aims to minimize energy consumption and reduce the additional quadratic cost associated with the relaxation variable.

The original COCP~\eqref{eq:ocp} can be reformulated as a sequence of QPs in the form of~\eqref{eq:clf_cbf_qp_optimal_problem}, i.e., the time interval $[t_0, t_f]$ is partitioned into a set of equal time intervals $\{ [t_0, t_0 + \Delta t), [t_0 + \Delta t, t_0 + 2\Delta t), \dots\}$, where $\Delta t > 0$.
Within each interval $[t_0 + q\Delta t, t_0 + (q + 1)\Delta t)$, for $q = 0, 1, 2, \dots$, the state maintains its value at the beginning of the interval, and the control inputs are obtained by solving~\eqref{eq:clf_cbf_qp_optimal_problem}.
The state is then updated based on the applied control input, and the procedure is repeated for subsequent intervals. 
Therefore, this QP-based approach~\eqref{eq:clf_cbf_qp_optimal_problem} is sub-optimal compared to the original COCP~\eqref{eq:ocp}, since the optimizations are performed pointwise in time.

\subsection{Velocity Obstacle}
\label{sec:vo}
\begin{figure} 
    \centering
    \includegraphics[width=0.7\linewidth]{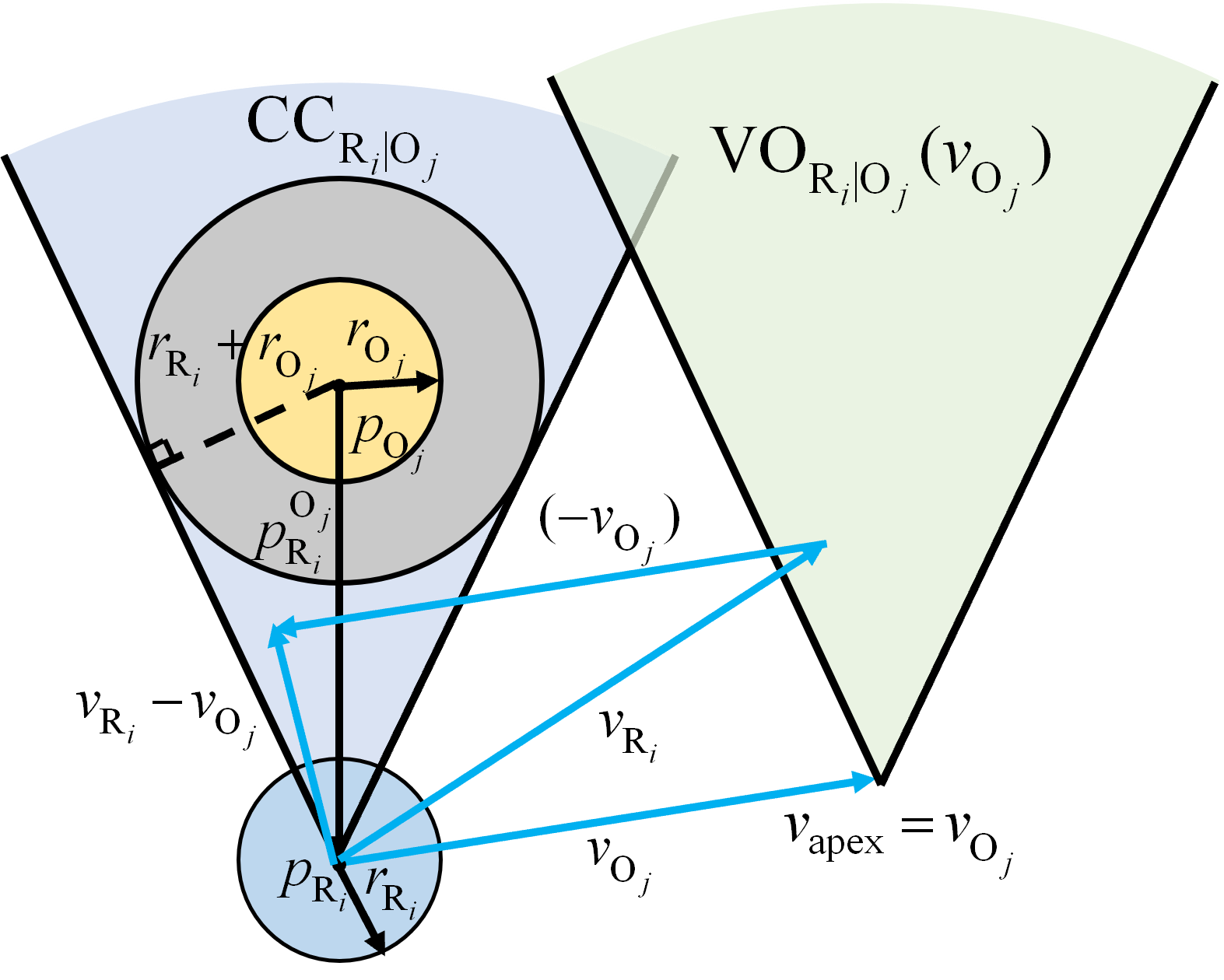}
    \caption{Velocity obstacle $\text{VO}_{\text{R}_i|\text{O}_j}(\bm{v}_{\text{O}_j})$ of robot $\text{R}_i$ induced by the obstacle $\text{O}_j$, and $\text{CC}_{\text{R}_i|\text{O}_j}$ represents the collision cone between them. If the relative velocity $\bm{v}_{\text{R}_i} - \bm{v}_{\text{O}_j} \in \text{CC}_{\text{R}_i|\text{O}_j}$ or the robot's velocity $\bm{v}_{\text{R}_i} \in \text{VO}_{\text{R}_i|\text{O}_j}(\bm{v}_{\text{O}_j})$, then a collision will occur between $\text{R}_i$ and $\text{O}_j$ at some future time with the assumption that both maintain their current velocities.}
    \label{fig:vo_explain}
\end{figure}
Assume there is a robot $\text{R}_i$ and an obstacle $\text{O}_j$ with radii $r_{\text{R}_i}$ and $r_{\text{O}_j}$.
Their positions and velocities are denoted as $\bm{p}_{\text{R}_i}$, $\bm{p}_{\text{O}_j}$, $\bm{v}_{\text{R}_i}$, $\bm{v}_{\text{O}_j}$, as shown in Fig.~\ref{fig:vo_explain}.
The VO of $\text{R}_i$ induced by $\text{O}_j$ is denoted as $\text{VO}_{\text{R}_i|\text{O}_j}(\bm{v}_{\text{O}_j})$, which contains all velocities for $\text{R}_i$ that would result in collisions with the obstacle in the future, assuming both maintain their current constant velocities.
Let $A \oplus B = \{ a + b \big| a \in A, b \in B\}$ be the Minkowski sum of sets $A$ and $B$, and let $\lambda(\bm{p}, \bm{v}) = \{\bm{p} + t\bm{v} \big| t>0\}$ denote a ray starting at position $\bm{p}$ and in the direction of vector $\bm{v}$.
If a ray starting at $\bm{p}_{\text{R}_i}$ and heading in the direction of the relative velocity $\bm{v}_{\text{R}_i} - \bm{v}_{\text{O}_j}$ intersects the Minkowski sum of $\text{O}_j \oplus -\text{R}_i$ centered at $\bm{p}_{\text{O}_j}$, 
then $\bm{v}_{\text{R}_i}$ is in the VO of $\text{R}_i$ induced by $\text{O}_j$.
Hence, $\text{VO}_{\text{R}_i|\text{O}_j}(\bm{v}_{\text{O}_j})$ is defined as~\cite{fiorini1998motion}:
\begin{equation*}
\resizebox{1.0\linewidth}{!}{$
    \text{VO}_{\text{R}_i|\text{O}_j}(\bm{v}_{\text{O}_j}) = \big\{
    \bm{v}_{\text{R}_i} \big| \lambda(\bm{p}_{\text{R}_i}, \bm{v}_{\text{R}_i}-\bm{v}_{\text{O}_j}) \cap \text{O}_j \oplus -\text{R}_i \neq \emptyset 
    \big\} $}.
\label{eq:vo_define}
\end{equation*}
Therefore, under the assumption that both $\text{R}_i$ and $\text{O}_j$ maintain their current constant velocities, if $\bm{v}_{\text{R}_i} \in \text{VO}_{\text{R}_i|\text{O}_j}(\bm{v}_{\text{O}_j})$, a collision will occur between $\text{R}_i$ and $\text{O}_j$ at some future moment.
Conversely, collision avoidance is achieved if $\bm{v}_{\text{R}_i} \notin \text{VO}_{\text{R}_i|\text{O}_j}(\bm{v}_{\text{O}_j})$.
In fact, $\text{VO}_{\text{R}_i|\text{O}_j}(\bm{v}_{\text{O}_j})$ is a velocity cone with its apex at $\bm{v}_{\text{O}_j}$, as shown in Fig.~\ref{fig:vo_explain}.
It can be transformed from the collision cone (CC) defined as:
\begin{equation*}
    \text{CC}_{\text{R}_i|\text{O}_j} = \big\{\bm{v}_{\text{R}_i}-\bm{v}_{\text{O}_j} \big| \lambda(\bm{p}_{\text{R}_i}, \bm{v}_{\text{R}_i}-\bm{v}_{\text{O}_j}) \cap \text{O}_j \oplus -\text{R}_i \neq \emptyset 
    \big\}
\end{equation*}
with the transformation $\text{VO}_{\text{R}_i|\text{O}_j}(\bm{v}_{\text{O}_j}) = \text{CC}_{\text{R}_i|\text{O}_j} \oplus \bm{v}_{\text{O}_j}$.
The difference between these two cones lies in the position of their apexes: $\text{CC}_{\text{R}_i|\text{O}_j}$ has its apex at the origin since it considers the relative velocity, while $\text{VO}_{\text{R}_i|\text{O}_j}(\bm{v}_{\text{O}_j})$ has its apex at $\bm{v}_{\text{O}_j}$.

VO is widely used to enable a robot to avoid collisions with multiple obstacles by having the robot select a velocity which lies outside any VO induced by each obstacle~\cite{fiorini1998motion}.
Several approaches have also been proposed to overcome the limitations of VO in avoiding collisions between robots, such as reciprocal velocity obstacle (RVO)~\cite{van2008reciprocal}, hybrid reciprocal velocity obstacle (HRVO)~\cite{snape2011hybrid} and optimal reciprocal collision avoidance (ORCA)~\cite{berg2011reciprocal}.
These methods reduce unnecessary oscillations by explicitly considering the reactive nature among robots.
The challenges of VO lie in its reliance on the assumption that the velocity of the obstacle remains constant and its inability to explicitly consider the robot's physical limitation constraints on velocity and acceleration.
In this work, we fully leverage the dynamic obstacle avoidance capabilities of VO to construct VO-based CBFs in the velocity space, instead of constructing CBF based on positions and introducing HOCBF for the acceleration-controlled unicycle model.
For more details, readers can refer to Section~\ref{sec:design_cbf}.
\section{Controller Design}
\label{sec:con}
In this section, we first present the design of state-feedback-based CLFs for the acceleration-controlled unicycle model~\eqref{eq:affine_robot_model} to achieve navigation.
Next, we demonstrate the construction of VOCBFs in the velocity space to ensure dynamic collision avoidance.
Additionally, we impose CBF constraints to account for the physical limitations of states.
Finally, we demonstrate how to formulate the constraints of CLFs, VOCBFs, along with other physical limitation constraints on states and control inputs in the form of CLF-VOCBF-MIQP and describe how it can be efficiently solved using a decision network.
For brevity, we omit the subscript $i$ of robots and reserve $i$ for other purposes.

\subsection{Design of Control Lyapunov Functions}
CLFs are commonly used to stabilize states to their desired equilibrium points.
When applying CLFs to the field of robotics, it is common to leverage their stabilizing properties for navigation purposes.
In this section, we mainly demonstrate the design of state-feedback-based CLFs~\cite{wang2023learning} to achieve navigation for the acceleration-controlled unicycle model~\eqref{eq:affine_robot_model}.
Specifically, we design four CLFs, each serving a different purpose: reducing the distance between the robot and its target position, adjusting the robot’s orientation to ensure it moves toward the target, and enabling efficient and smooth motion.

Assume the robot's goal state is denoted as $(x_\text{gr}, y_\text{gr}, \theta_\text{g})$, where $(x_\text{gr}, y_\text{gr})$ represents the coordinates of the rear axle axis.
Additionally, the center position $(x_\text{g}, y_\text{g})$ of the goal state is defined as $x_\text{g} = x_\text{gr} + l\cos\theta_\text{g}, y_\text{g} = y_\text{gr} + l\sin\theta_\text{g}$.
Our primary objective is to stabilize the robot’s position to its target position $(x_\text{g}, y_\text{g})$ without explicitly considering other states such as $v$ and $\omega$ in the dynamics~\eqref{eq:affine_robot_model}.
If we design a CLF such as $V(\bm{x}) = (x_\text{c} - x_\text{g})^2 + (y_\text{c} - y_\text{g})^2$, then $\forall \bm{x}, L_g V(\bm{x}) = \bm{0}$. 
This results in a high relative degree issue~\cite{xiao2019control}, preventing us from obtaining a valid controller. 
To address this problem, we adopt state-feedback-based CLFs~\cite{wang2023learning}, which mitigate the relative degree issue and facilitate controller design.
\begin{remark}
    State-feedback-based CLFs stabilize a state to its desired state by making its higher-order derivative inversely dependent on all lower-order derivatives, ensuring all lower-order derivatives eventually converge to their desired states. This method is especially effective for high-order systems.
\end{remark}

Suppose $\bm{x} = [x_1, x_2, \dots, x_n]^\top \in \mathbb{R}^n$, and we aim to stabilize only the partial states $[x_1, x_2, \dots, x_s]^\top \in \mathbb{R}^s, s < n$ to their desired states $[x_1^*, x_2^*, \dots, x_s^*]^\top$, with $\bm{u} \in \mathbb{R}^m$.
Assume $\forall x_i, i \in \{1, 2, \dots, s\}$, control inputs $\bm{u}$ can influence the dynamics.
Furthermore, the relative degree of $x_i, i \in \{1, 2, \dots, s\}$ is defined as $r_i \in \mathbb{N}$, and $n_0 := \sum_{i=1}^{s} r_i$ represents the sum of the relative degree.
We recursively define new states by taking derivatives of $x_i$: the first derivative is $x_{i,1} \in \mathbb{R}$, the second derivative is $x_{i,2} \in \mathbb{R}$, and so on, until reaching $x_{i, r_i - 1} \in \mathbb{R}$, where the first derivative of $x_{i, r_i - 1}$ becomes an affine function of $\bm{u}$, i.e., the relative degree of $x_{i, r_i - 1}$ is 1.
In this work, we extend the concept of state-feedback-based CLFs, which were originally designed to stabilize $x_i$ to the origin, to stabilize $x_i$ to its desired state $x_i^*$.
The desired state of $x_{i, r_i - 1}$ is defined as $x_{i, r_i - 1}^{*}:= -l_i(x_i - x_i^*) - l_{i, 1}x_{i, 1} - \dots - l_{i, r_i - 2}x_{i, r_i - 2}$, where $l_i > 0$, $l_{i, 1} > 0$, $\dots$, $l_{i, r_i - 2} > 0$.
Since the desired state of $x_{i, r_i - 1}$ depends inversely on all lower relative degree states $x_{i, k}, k \in \{0, 1, \dots, r_i - 2\}$, stabilizing $x_{i, r_i - 1}$ to its desired state $x_{i, r_i - 1}^{*}$ will indirectly stabilize all lower relative degree states to their desired states, including $x_i$.
Thus, we can design a CLF to stabilize $x_{i, r_i - 1}$ to its desired state $x_{i, r_i - 1}^{*}$, which in turn stabilizes $x_i$ to its desired states.
This is referred to as a state-feedback-based CLF~\cite{wang2023learning}.
To integrate this design, we define an integrated state $\bm{y} := [x_1, \dots, x_s, x_{1, 1}, \dots, x_{1, r_1 - 1}, \dots, x_{s,1}, \dots, x_{s, r_s - 1}]^\top \in \mathbb{R}^{n_0}$ with its desired state $\bm{y}^*:= [x_1^*, \dots, x_s^*, 0, \dots, 0]^\top$. 
The valid state-feedback-based CLF with relative degree 1 is defined as:
\begin{equation}
\begin{aligned}
    & V(\bm{y}, \bm{y}^*) = \sum_{i=1}^{s} c_i (\frac{1}{l_i}x_{i, r_i - 1} - \frac{1}{l_i}x_{i, r_i - 1}^{*})^2 \\
    & = \sum_{i=1}^{s}c_i(x_i - x_i^* + k_{i, 1}x_{i, 1} + \dots + k_{i, r_i - 1}x_{i, r_i - 1})^2,
    \label{eq:hoclf}
\end{aligned}
\end{equation}
where $c_i \geq 0, i \in \{1, \dots, s\}$, and $k_{i, 1} = \frac{l_{i, 1}}{l_i}, \dots, k_{i, r_i - 2} = \frac{l_{i, r_i - 2}}{l_i}, k_{i, r_i - 1} = \frac{1}{l_i} > 0$.
With this state-feedback-based CLF, we can stabilize the states $[x_1, \dots, x_s]^\top$ to their desired states $[x_1^*, \dots, x_s^*]^\top$.
Additionally, \eqref{eq:hoclf} can be rewritten in matrix form as:
\begin{equation}
    V(\bm{y}, \bm{y}^*) = (Q(\bm{y} - \bm{y}^*))^\top \Lambda (Q(\bm{y} - \bm{y}^*)),
    \label{eq:hoclf_matrix_form}
\end{equation}
where 
\begin{equation*}
\begin{aligned}
    Q &= 
    \left[\begin{array}{cc}
        \mathbf{I}_{s \times s} & \mathcal{K} \\
        \mathbf{0}_{(n_0 - s) \times s} & \mathbf{0}_{(n_0 - s) \times (n_0 - s)}
    \end{array}\right] \in \mathbb{R}^{n_0 \times n_0}, \\
    \mathcal{K} &= 
    \left[\begin{array}{cccc}
        \bm{k}_1 & \bm{0} & \cdots & \bm{0} \\
        \bm{0} & \bm{k}_2 & \cdots & \bm{0} \\
        \vdots & \vdots & \ddots & \vdots \\
        \bm{0} & \bm{0} & \cdots & \bm{k}_s
    \end{array}\right] \in \mathbb{R}^{s \times (n_0 - s)}, \\
    \Lambda &= 
    \left[\begin{array}{cc}
        \mathbf{c}_{s \times s} & \mathbf{0}_{s \times (n_0 - s)} \\
        \mathbf{0}_{(n_0 - s) \times s} & \mathbf{0}_{(n_0 - s) \times (n_0 - s)}
    \end{array}\right] \in \mathbb{R}^{n_0 \times n_0}, 
\end{aligned}
\end{equation*}
where $\mathcal{K}$ is a block diagonal matrix composed of vectors $\bm{k}_i = [k_{i, 1}, \dots, k_{i, r_i-1}], i \in \{1, \dots, s\}$, $\mathbf{c}_{s \times s}$ is a diagonal matrix composed of $c_i, i \in \{1, \dots, s\}$, $\mathbf{0}_{s \times (n_0 - s)}$ is a zero matrix of size $s \times (n_0 - s)$, and $\mathbf{I}_{s \times s}$ is an identity matrix of size $s \times s$.
In addition, the CLF constraint~\eqref{eq:cons_clf1} can be derived for the state-feedback-based CLF~\cite{wang2023learning}.

To achieve navigation for the acceleration-controlled unicycle model~\eqref{eq:affine_robot_model}, we design four distinct CLFs, each serving a specific purpose:
First, we design a valid state-feedback-based CLF to reduce the distance between the robot's current position and its target position as:
\begin{equation}
\begin{aligned}
    V_\text{d}(\bm{x}) = & c_1[x_\text{c} - x_\text{g} + k_1 (v\cos\theta - l\sin\theta\omega)]^2 +  \\
    & c_2[y_\text{c} - y_\text{g} + k_2 (v\sin\theta + l\cos\theta\omega)]^2,
    \label{eq:v_distance}
\end{aligned}
\end{equation}
where $x_c$ and $y_c$ denote the center position of the robot, as defined in~\eqref{eq:transform_relationship}.
This CLF allows the controller to effectively manage both linear and angular accelerations, guiding the robot to its destination.
\begin{remark}
If $(x_\text{p}, y_\text{p})$ represents the robot's center position instead of the coordinates of the rear axle axis, the state-feedback-based CLF becomes $V_\text{d}(\bm{x}) = c_1(x_\text{p} - x_\text{g} + k_1 v\cos\theta)^2 + c_2(y_\text{p} - y_\text{g} + k_2 v\sin\theta)^2$.
Additionally, $L_g V_\text{d}(\bm{x}) = [*, 0]$ in this case, meaning the controller's performance is limited since it cannot control the angular acceleration $\alpha$.
To address this limitation, we propose controlling $(x_\text{p}, y_\text{p})$ directly instead of $(x_\text{c}, y_\text{c})$.
\end{remark}

Additionally, to ensure the robot always moves toward its destination, we design another state-feedback-based CLF for the robot's orientation. 
Since the relative degree of $\theta - \arctan(\frac{y_\text{g} - y_\text{c}}{x_\text{g} - x_\text{c}})$ is 2, the CLF is expressed as:
\begin{equation}
    V_\theta(\bm{x}) = (\theta - \arctan(\frac{y_\text{g} - y_\text{c}}{x_\text{g} - x_\text{c}}) + k_\theta q_\theta)^2,
    \label{eq:v_theta}
\end{equation}
where $q_\theta$ is the $1^\text{st}$ derivative of $\theta - \arctan(\frac{y_\text{g} - y_\text{c}}{x_\text{g} - x_\text{c}})$.
With this design, the relative degree of $V_\theta(\bm{x})$ is 1, ensuring control inputs explicitly appear in the constraints of $V_\theta(\bm{x})$.

Besides, we design a CLF that enables the robot to adjust its velocity to approach its desired velocity.
When far from the target, the robot moves faster to reduce travel time; as it nears the target, it slows down. 
The CLF is defined as:
\begin{equation}
     V_v(\bm{x}) = (v - v_\text{d})^2,
     \label{eq:v_v}
\end{equation}
where the desired velocity $v_\text{d}$ is proportional to the distance between the robot's current position and its target.
Moreover, to ensure smooth motion, we design a CLF as:
\begin{equation}
    V_\omega(\bm{x}) = \omega^2.
    \label{eq:v_omega}
\end{equation}
Since $L_g V_v(\bm{x})=[2(v - v_\text{d}), 0]$ and $L_g  V_\omega(\bm{x})=[0, 2\omega]$, both CLFs are valid without requiring state-feedback-based designs.

In conclusion, we design four CLFs $V_\text{d}(\bm{x})$, $V_\theta(\bm{x})$, $V_v(\bm{x})$, and $V_\omega(\bm{x})$ to achieve navigation.
Among these, $V_\text{d}(\bm{x})$ and $V_\theta(\bm{x})$ serve as the primary CLFs to guide the robot to its destination, while $V_v(\bm{x})$ and $V_\omega(\bm{x})$ play auxiliary roles.
This prioritization is managed by adjusting the weight factors $p$ of the relaxation variables in~\eqref{eq:optimal_problem1} for each CLF constraint.
Specifically, a small weight factor increases the level of relaxation, making the CLF auxiliary.
A large weight factor reduces relaxation, giving the CLF a primary role.

\subsection{Design of Control Barrier Functions}
\label{sec:design_cbf}
In this section, we mainly demonstrate how to construct VOCBFs for the acceleration-controlled unicycle model~\eqref{eq:affine_robot_model} to achieve dynamic collision avoidance.
Consider a CBF constructed based on the Euclidean distance between robot $\text{R}_i$ and obstacle $\text{O}_j$, defined as:
\begin{equation}
    h(\bm{x}) = \| \bm{p}_{\text{R}_i} - \bm{p}_{\text{O}_j} \|^2 - (r_{\text{R}_i} + r_{\text{O}_j})^2.
    \label{eq:hocbf}
\end{equation}
Since the relative degree of~\eqref{eq:hocbf} is 2 w.r.t the robot's dynamics~\eqref{eq:affine_robot_model}, HOCBF is necessary to be introduced to generate valid collision avoidance constraints with the safety guarantee that~\eqref{eq:hocbf} is always non-negative~\cite{xiao2019control}.
However, HOCBF only provides a subset of the original safe set for safety guarantees, since it requires all but the final derivative of the nominal CBF, where the control inputs explicitly appear, to be non-negative, which can lead to the optimization problem being conservative or even infeasible in certain cases~\cite{thontepu2022control}.
Additionally, the calculation of proper constraint candidates of HOCBF with appropriate penalty weights and parameters is computationally expensive~\cite{xiao2022high}.
To overcome these limitations and design a valid CBF with a relative degree of 1, we propose constructing VOCBFs directly in the velocity space, eliminating the need for HOCBFs.

\begin{figure} 
    \centering
    \includegraphics[width=0.7\linewidth]{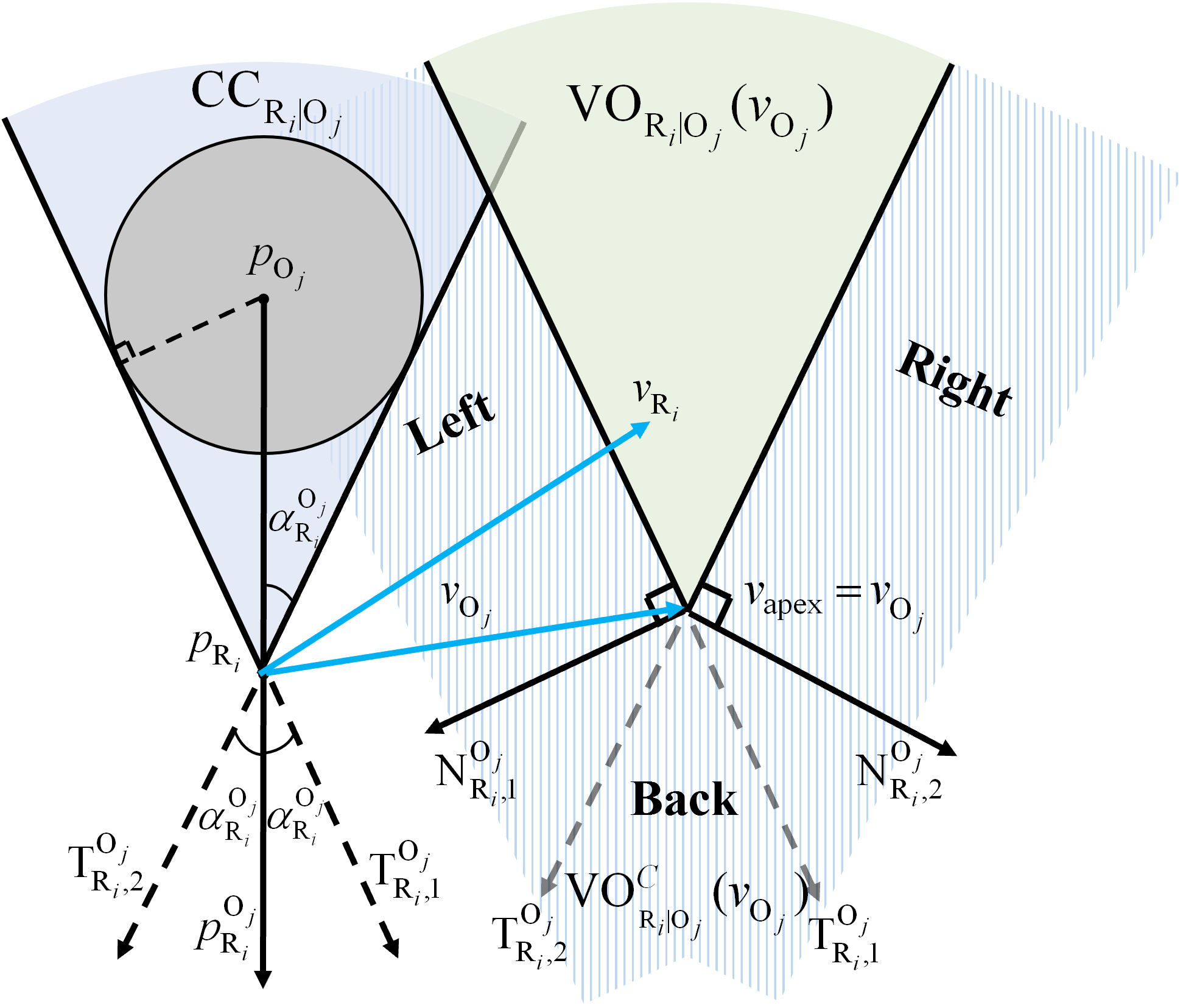}
    \caption{The safe range of the robot's velocity is denoted by $\text{VO}_{\text{R}_i|\text{O}_j}^{\text{C}}(\bm{v}_{\text{O}_j})$, and it can be represented by the disjunction of two linear constraints, where each linear constraint requires the robot's velocity is within a half-space. Each half-space corresponds to a direction in which the robot can navigate around the obstacle, and satisfying both leads to backward avoidance.}
    \label{fig:vocbf_explain}
\end{figure}
As mentioned in Section~\ref{sec:vo}, the VO of $\text{R}_i$ induced by $\text{O}_j$ is denoted as $\text{VO}_{\text{R}_i|\text{O}_j}(\bm{v}_{\text{O}_j})$.
Under the assumption that both $\text{R}_i$ and $\text{O}_j$ maintain their current velocities, $\text{R}_i$ will collide with $\text{O}_j$ at some future moment if $\bm{v}_{\text{R}_i} \in \text{VO}_{\text{R}_i|\text{O}_j}(\bm{v}_{\text{O}_j})$.
Conversely, collision avoidance is guaranteed if $\bm{v}_{\text{R}_i} \notin \text{VO}_{\text{R}_i|\text{O}_j}(\bm{v}_{\text{O}_j})$, i.e., $\bm{v}_{\text{R}_i} \in \text{VO}_{\text{R}_i|\text{O}_j}^{\text{C}}(\bm{v}_{\text{O}_j})$.
Here, $\text{VO}_{\text{R}_i|\text{O}_j}^{\text{C}}(\bm{v}_{\text{O}_j})$ is the complement of $\text{VO}_{\text{R}_i|\text{O}_j}(\bm{v}_{\text{O}_j})$ and represents the feasible velocity region for $\bm{v}_{\text{R}_i}$, as shown in the blue shaded area of Fig.~\ref{fig:vocbf_explain}.
Let $\bm{p}_{\text{R}_i}^{\text{O}_j}:= \bm{p}_{\text{R}_i} - \bm{p}_{\text{O}_j}$ and $\bm{v}_{\text{R}_i}^{\text{O}_j}:= \bm{v}_{\text{R}_i} - \bm{v}_{\text{O}_j}$ represent the relative position and velocity between $\text{R}_i$ and $\text{O}_j$, respectively.
The angle between $\bm{p}_{\text{R}_i}^{\text{O}_j}$ and the boundary line of $\text{CC}_{\text{R}_i|\text{O}_j}$ is defined as $\alpha_{\text{R}_i}^{\text{O}_j}$, where $\alpha_{\text{R}_i}^{\text{O}_j} \in (0, \frac{\pi}{2})$, and $\sin\alpha_{\text{R}_i}^{\text{O}_j} = 
\frac{r_{\text{R}_i} + r_{\text{O}_j}}{\Vert \bm{p}_{\text{R}_i}^{\text{O}_j} \Vert}$, as shown Fig.~\ref{fig:vocbf_explain}.
The two tangent directional vectors of $\text{CC}_{\text{R}_i|\text{O}_j}$, denoted as $\bm{T}_{\text{R}_i, 1}^{\text{O}_j}$ and $\bm{T}_{\text{R}_i, 2}^{\text{O}_j}$, can be obtained by rotating the vector $\bm{p}_{\text{R}_i}^{\text{O}_j}$ counterclockwise and clockwise by the angle $\alpha_{\text{R}_i}^{\text{O}_j}$, respectively
\begin{equation}
    \bm{T}_{\text{R}_i, 1}^{\text{O}_j} = \bm{R}( \alpha_{\text{R}_i}^{\text{O}_j}) \bm{p}_{\text{R}_i}^{\text{O}_j}, \;
    \bm{T}_{\text{R}_i, 2}^{\text{O}_j} = \bm{R}(-\alpha_{\text{R}_i}^{\text{O}_j}) \bm{p}_{\text{R}_i}^{\text{O}_j},
    \label{eq:tangent_vecs}
\end{equation}
where 
\begin{equation*}
    \bm{R}(\theta) = \left[\begin{array}{cc}
    \cos\theta & -\sin\theta \\
    \sin\theta & \cos\theta
    \end{array}\right]
\end{equation*}
represents the rotation matrix that rotates a vector counterclockwise by the angle $\theta$.
Additionally, two outer normal vectors $\bm{N}_{\text{R}_i, 1}^{\text{O}_j}$ and $\bm{N}_{\text{R}_i, 2}^{\text{O}_j}$ can be obtained by rotating $\bm{T}_{\text{R}_i, 1}^{\text{O}_j}$ clockwise with the angle $\frac{\pi}{2}$ and rotating $\bm{T}_{\text{R}_i, 2}^{\text{O}_j}$ counterclockwise with the angle $\frac{\pi}{2}$ as:
\begin{equation}
    \bm{N}_{\text{R}_i, 1}^{\text{O}_j} = \bm{R}(-\frac{\pi}{2}) 
    \bm{T}_{\text{R}_i, 1}^{\text{O}_j}, \;
    \bm{N}_{\text{R}_i, 2}^{\text{O}_j} = \bm{R}( \frac{\pi}{2}) \bm{T}_{\text{R}_i, 2}^{\text{O}_j}.
    \label{eq:normal_vecs}
\end{equation}
The outer normal vectors and the tangent directional vectors of $\text{VO}_{\text{R}_i|\text{O}_j}(\bm{v}_{\text{O}_j})$ are identical to those of $\text{CC}_{\text{R}_i|\text{O}_j}$, since $\text{VO}_{\text{R}_i|\text{O}_j}(\bm{v}_{\text{O}_j})$ is translated from $\text{CC}_{\text{R}_i|\text{O}_j}$.
Given that collision avoidance between $\text{R}_i$ and $\text{O}_j$ is ensured if $\bm{v}_{\text{R}_i} \notin \text{VO}_{\text{R}_i|\text{O}_j}(\bm{v}_{\text{O}_j})$, i.e., $\bm{v}_{\text{R}_i} \in \text{VO}_{\text{R}_i|\text{O}_j}^{\text{C}}(\bm{v}_{\text{O}_j})$.
This condition can be represented by the disjunction of two linear constraints as:
\begin{equation}
    \bm{v}_{\text{R}_i} \in \text{VO}_{\text{R}_i|\text{O}_j}^{\text{C}}(\bm{v}_{\text{O}_j}) \iff \bigcup_{k \in \{1, 2\}} (\bm{N}_{\text{R}_i,k}^{\text{O}_j})^\top \bm{v}_{\text{R}_i} \geq c_{\text{R}_i, k}^{\text{O}_j},
    \label{eq:vo_cons1}
\end{equation}
where $c_{\text{R}_i, k}^{\text{O}_j}, k \in \{1, 2\}$ is the scalar corresponding to the $k^{\text{th}}$ linear constraint of $\text{VO}_{\text{R}_i|\text{O}_j}^{\text{C}}(\bm{v}_{\text{O}_j})$, given by
\begin{equation*}
    c_{\text{R}_i,1}^{\text{O}_j} = (\bm{N}_{\text{R}_i, 1}^{\text{O}_j})^\top \bm{v}_{\text{O}_j}, \;
    c_{\text{R}_i,2}^{\text{O}_j} = (\bm{N}_{\text{R}_i, 2}^{\text{O}_j})^\top \bm{v}_{\text{O}_j}.
\end{equation*}
As indicated by~\eqref{eq:vo_cons1}, $\text{VO}_{\text{R}_i|\text{O}_j}^{\text{C}}(\bm{v}_{\text{O}_j})$ is represented by the combination of two half-spaces.
In fact, each half-space corresponds to a direction in which the robot can navigate around the obstacle. 
If $\bm{v}_{\text{R}_i}$ only satisfies $(\bm{N}_{\text{R}_i, 1}^{\text{O}_j})^\top \bm{v}_{\text{R}_i} \geq c_{\text{R}_i, 1}^{\text{O}_j}$, then the robot will navigate around the obstacle from its left side.
Here, `left' specifically refers to the left side relative to the robot's local coordinate frame, as shown in Fig.~\ref{fig:vocbf_explain}.
Conversely, if $\bm{v}_{\text{R}_i}$ only satisfies $(\bm{N}_{\text{R}_i, 2}^{\text{O}_j})^\top \bm{v}_{\text{R}_i} \geq c_{\text{R}_i, 2}^{\text{O}_j}$, then the robot will pass the obstacle from the right side. 
If $\bm{v}_{\text{R}_i}$ satisfies both $(\bm{N}_{\text{R}_i, 1}^{\text{O}_j})^\top \bm{v}_{\text{R}_i} \geq c_{\text{R}_i, 1}^{\text{O}_j}$ and $(\bm{N}_{\text{R}_i, 2}^{\text{O}_j})^\top \bm{v}_{\text{R}_i} \geq c_{\text{R}_i, 2}^{\text{O}_j}$, then the robot will move backward to avoid the obstacle.
Furthermore, \eqref{eq:vo_cons1} can be reformulated using the relative velocity $\bm{v}_{\text{R}_i}^{\text{O}_j}$ as
\begin{equation}
    \bigcup_{k \in \{1, 2\}} (\bm{N}_{\text{R}_i,k}^{\text{O}_j})^\top \bm{v}_{\text{R}_i}^{\text{O}_j} \geq 0.
    \label{eq:vo_cons2}
\end{equation}
If the robot's velocity satisfies~\eqref{eq:vo_cons2}, then collision avoidance is guaranteed between $\text{R}_i$ and $\text{O}_j$.
We construct two VOCBFs based on these two linear constraints~\eqref{eq:vo_cons2} as
\begin{gather}
    h_{\text{R}_i, 1}^{\text{O}_j}(\bm{x}) = (\bm{N}_{\text{R}_i, 1}^{\text{O}_j})^\top \bm{v}_{\text{R}_i}^{\text{O}_j} = (\bm{v}_{\text{R}_i}^{\text{O}_j})^\top \bm{N}_{\text{R}_i, 1}^{\text{O}_j},  
    \label{eq:vocbf1} \\ 
    h_{\text{R}_i, 2}^{\text{O}_j}(\bm{x}) = (\bm{N}_{\text{R}_i, 2}^{\text{O}_j})^\top \bm{v}_{\text{R}_i}^{\text{O}_j} = (\bm{v}_{\text{R}_i}^{\text{O}_j})^\top \bm{N}_{\text{R}_i, 2}^{\text{O}_j}.
    \label{eq:vocbf2}
\end{gather}
The relative position and velocity between $\text{R}_i$ and $\text{O}_j$ are computed as
\begin{equation*}
\begin{aligned}
    \bm{p}_{\text{R}_i}^{\text{O}_j} &=  
    \left[\begin{array}{c}
        x_{\text{p},i} + l_i\cos\theta_i - x_{\text{o}, j} \\
        y_{\text{p},i} + l_i\sin\theta_i - y_{\text{o}, j}
    \end{array}\right] = 
    \left[\begin{array}{c}
        \Delta x_i^j \\
        \Delta y_i^j
    \end{array}\right], \\
    \bm{v}_{\text{R}_i}^{\text{O}_j} &= 
    \frac{d \bm{p}_{\text{R}_i}^{\text{O}_j}}{dt} = 
    \left[\begin{array}{c}
        v_i\cos\theta_i - l\sin\theta_i \omega_i - v_{\text{ox}, j} \\
        v_i\sin\theta_i + l\cos\theta_i \omega_i - v_{\text{oy}, j}
    \end{array}\right],
\end{aligned}
\end{equation*}
where center positions of $\text{R}_i$ and $\text{O}_j$ are given by $\bm{p}_{\text{R}_i} = [x_{\text{p}, i} + l_i\cos\theta_i, y_{\text{p}, i} + l_i\sin\theta_i]^\top$ and $\bm{p}_{\text{O}_j} = [x_{\text{o}, j}, y_{\text{o}, j}]^\top$, respectively.
Using~\eqref{eq:tangent_vecs} and~\eqref{eq:normal_vecs}, the VOCBFs~\eqref{eq:vocbf1} and~\eqref{eq:vocbf2} can be expressed as
\begin{align}
    h_{\text{R}_i, 1}^{\text{O}_j}(\bm{x}) &= (\bm{v}_{\text{R}_i}^{\text{O}_j})^\top 
    \bm{R}(\alpha_{\text{R}_i}^{\text{O}_j} - \frac{\pi}{2}) \bm{p}_{\text{R}_i}^{\text{O}_j},
    \label{eq:cbf_form1} \\
    h_{\text{R}_i, 2}^{\text{O}_j}(\bm{x}) &= (\bm{v}_{\text{R}_i}^{\text{O}_j})^\top
    \bm{R}(\frac{\pi}{2} - \alpha_{\text{R}_i}^{\text{O}_j}) \bm{p}_{\text{R}_i}^{\text{O}_j}.
    \label{eq:cbf_form2}
\end{align}
It is important to note that for collision avoidance, at least one of these two VOCBFs~\eqref{eq:vocbf1} and~\eqref{eq:vocbf2} must be non-negative.
Therefore, at least one of the two corresponding CBF constraints~\eqref{eq:cons_cbf1} must be satisfied.
In fact, satisfying only one of these two constraints is sufficient in most case, meaning the robot can simply choose a direction to navigate around the obstacle.

\begin{assumption}
\label{ass1}
$\alpha_{\text{R}_i}^{\text{O}_j} \neq \frac{\pi}{2}$, i.e., $\cos \alpha_{\text{R}_i}^{\text{O}_j} \neq 0$.
\end{assumption}
\begin{assumption}
\label{ass2}
$\cos\theta_i \neq 0$.
\end{assumption}
\begin{remark}
Assumption~\ref{ass1} may become invalid when the distance $\Vert \bm{p}_{\text{R}_i}^{\text{O}_j} \Vert$ between the robot and the obstacle equals $r_{\text{R}_i} + r_{\text{O}_j}$.
Similarly, Assumption~\ref{ass2} could be invalid when the robot's orientation equals $\frac{\pi}{2}$ or $-\frac{\pi}{2}$.
Without loss of generality, we assume both Assumptions 1 and 2 hold true.
\end{remark}
\begin{theorem}
Given the acceleration-controlled model~\eqref{eq:affine_robot_model}, the proposed VOCBF candidates are valid CBFs under Assumption~\ref{ass1} and~\ref{ass2}, i.e., $\forall \bm{x} \in \mathcal{D}, L_g h_{\text{R}_i, 1}^{\text{O}_j}(\bm{x}) \neq \mathbf{0}, L_g h_{\text{R}_i, 2}^{\text{O}_j}(\bm{x}) \neq \mathbf{0}$.
\end{theorem}

\begin{proof}
We demonstrate the validity of the VOCBF constraints using~\eqref{eq:cbf_form1} as an example, as the proof of~\eqref{eq:cbf_form2} follows a similar approach.
The time derivative of $h_{\text{R}_i, 1}^{\text{O}_j}(\bm{x})$ is given by
\begin{align}
    \frac{d h_{\text{R}_i, 1}^{\text{O}_j}(\bm{x})}{dt} &= (\frac{d \bm{v}_{\text{R}_i}^{\text{O}_j}}{dt})^\top 
    \bm{R}(\alpha_{\text{R}_i}^{\text{O}_j} - \frac{\pi}{2}) \bm{p}_{\text{R}_i}^{\text{O}_j} \label{eq:hdot1} \\
    &+ (\bm{v}_{\text{R}_i}^{\text{O}_j})^\top 
    \frac{d \bm{R}(\alpha_{\text{R}_i}^{\text{O}_j} - \frac{\pi}{2})}{dt} \bm{p}_{\text{R}_i}^{\text{O}_j} \label{eq:hdot2} \\
    &+ (\bm{v}_{\text{R}_i}^{\text{O}_j})^\top 
    \bm{R}(\alpha_{\text{R}_i}^{\text{O}_j} - \frac{\pi}{2}) \frac{d \bm{p}_{\text{R}_i}^{\text{O}_j}}{dt}, \label{eq:hdot3}
\end{align}
where the derivative can also be expressed as:
\begin{equation}
    \begin{aligned}
    \frac{d h_{\text{R}_i, 1}^{\text{O}_j}(\bm{x})}{dt} &= L_f h_{\text{R}_i, 1}^{\text{O}_j}(\bm{x}) + L_g h_{\text{R}_i, 1}^{\text{O}_j}(\bm{x}) \bm{u} \\
    &+ \frac{d h_{\text{R}_i, 1}^{\text{O}_j}(\bm{x})}{d \bm{x}_{\text{O}_j}} \frac{d \bm{x}_{\text{O}_j}}{dt},
\end{aligned}
\end{equation}
where we consider time-varying CBFs~\cite{huang2023obs} to handle dynamic obstacles.
Only the term~\eqref{eq:hdot1} is related to the control inputs $a$ and $\alpha$, while other terms~\eqref{eq:hdot2} and~\eqref{eq:hdot3} do not contain any terms related to $a$ and $\alpha$.
Extracting the term related to $\bm{u}=[a, \alpha]^{\top}$ from~\eqref{eq:hdot1}, we derive
\begin{equation*}
    L_g h_{\text{R}_i, 1}^{\text{O}_j}(\bm{x}) = \left[\begin{array}{c}
    \langle
    \bm{R}(\alpha_{\text{R}_i}^{\text{O}_j} -\frac{\pi}{2}) 
    \bm{p}_{\text{R}_i}^{\text{O}_j},
    \left[\begin{array}{c}
        \cos\theta_i \\
        \sin\theta_i
    \end{array}\right]
    \rangle \\
    \langle
    \bm{R}(\alpha_{\text{R}_i}^{\text{O}_j} - \frac{\pi}{2})
    \bm{p}_{\text{R}_i}^{\text{O}_j},
    \left[\begin{array}{c}
        -l\sin\theta_i \\
        l\cos\theta_i
    \end{array}\right]
    \rangle
    \end{array}\right]^\top,
\end{equation*}
where $\langle \bm{a}, \bm{b} \rangle$ represents the inner product between $\bm{a}$ and $\bm{b}$.

To prove that $L_g h_{\text{R}_i, 1}^{\text{O}_j}(\bm{x}) \neq \mathbf{0}, \forall \bm{x} \in \mathcal{D}$, we proceed by contradiction.
Suppose $L_g h_{\text{R}_i, 1}^{\text{O}_j}(\bm{x}) = \mathbf{0}$, this leads to the following possible scenarios:
\begin{itemize}
    \item \textbf{Scenario 1:} $\bm{R}(\alpha_{\text{R}_i}^{\text{O}_j} -\frac{\pi}{2}) \bm{p}_{\text{R}_i}^{\text{O}_j} = \bm{0}$, which implies $L_g h_{\text{R}_i, 1}^{\text{O}_j}(\bm{x}) \equiv \bm{0}$.
    This further implies $\sin\alpha_{\text{R}_i}^{\text{O}_j} \Delta x_i^j + \cos\alpha_{\text{R}_i}^{\text{O}_j} \Delta y_i^j = 0$ and $-\cos\alpha_{\text{R}_i}^{\text{O}_j} \Delta x_i^j + \sin\alpha_{\text{R}_i}^{\text{O}_j} \Delta y_i^j = 0$.
    Under Assumption~\ref{ass1}, this results in $\Delta x_i^j = \Delta y_i^j = 0$, implying that the centers of $\text{R}_i$ and $\text{O}_j$ coincide.
    This is impossible, as it wound mean the robot is already inside the obstacle.
    \item \textbf{Scenario 2:} $\bm{R}(\alpha_{\text{R}_i}^{\text{O}_j} -\frac{\pi}{2})
        \bm{p}_{\text{R}_i}^{\text{O}_j} \bot
        \left[\begin{array}{c}
            \cos\theta_i \\
            \sin\theta_i
        \end{array}\right]$ and $\bm{R}(\alpha_{\text{R}_i}^{\text{O}_j}-\frac{\pi}{2})
        \bm{p}_{\text{R}_i}^{\text{O}_j} \bot
        \left[\begin{array}{c}
            -l\sin\theta_i \\
            l\cos\theta_i
        \end{array}\right]$.
        Let $\bm{R}(\alpha_{\text{R}_i}^{\text{O}_j} -\frac{\pi}{2}) \bm{p}_{\text{R}_i}^{\text{O}_j} = [s_1, s_2]^\top$.
        Then $s_1 \cos\theta_i + s_2 \sin\theta_i = 0$ and $s_1 \sin\theta_i = -s_2 \cos\theta_i$.
        Under Assumption~\ref{ass2}, we derive $s_1 = s_2 = 0$, returning to Scenario 1, which is also impossible.
\end{itemize}
Thus, $L_g h_{\text{R}_i, 1}^{\text{O}_j}(\bm{x}) \neq \mathbf{0}$, indicating that the constraint of $h_{\text{R}_i, 1}^{\text{O}_j}(\bm{x})$ is always valid.
Similarly, for~\eqref{eq:cbf_form2}, we derive
\begin{equation*}
    L_g h_{\text{R}_i, 2}^{\text{O}_j}(\bm{x}) = \left[\begin{array}{c}
    \langle
    \bm{R}(\frac{\pi}{2} - \alpha_{\text{R}_i}^{\text{O}_j})
    \bm{p}_{\text{R}_i}^{\text{O}_j},
    \left[\begin{array}{c}
        \cos\theta_i \\
        \sin\theta_i
    \end{array}\right]
    \rangle \\
    \langle
    \bm{R}(\frac{\pi}{2} - \alpha_{\text{R}_i}^{\text{O}_j})
    \bm{p}_{\text{R}_i}^{\text{O}_j},
    \left[\begin{array}{c}
        -l\sin\theta_i \\
        l\cos\theta_i
    \end{array}\right]
    \rangle
    \end{array}\right]^\top.
\end{equation*}
The proof that $L_g h_{\text{R}_i, 2}^{\text{O}_j}(\bm{x}) \neq \mathbf{0}$ is analogous and is omitted for brevity.
\end{proof}

We have introduced how to construct VOCBFs based on $\text{VO}_{\text{R}_i|\text{O}_j}(\bm{v}_{\text{O}_j})$ to achieve collision avoidance between $\text{R}_i$ and $\text{O}_j$.
Similarly, to avoid collisions between robots $\text{R}_i$ and $\text{R}_j$, VOCBFs can also be constructed based on $\text{VO}_{\text{R}_i|\text{R}_j}(\bm{v}_{\text{R}_j})$.
However, since VO ignores the reactive nature between robots, which enables each robot to independently adjust its velocity to avoid collision with other robots and obstacles, using VO for navigation and collision avoidance in distributed multi-robot systems may bring unnecessary oscillations.
Some variants of VO like RVO~\cite{van2008reciprocal} and HRVO~\cite{snape2011hybrid} explicitly consider the reactive nature of robots, leading to improved performance without unnecessary oscillations.
We also consider integrating these variations into our methodology and find that the constraints of RVO-based~\cite{van2008reciprocal} CBFs are consistent with those of VOCBFs.
This consistency arises because velocity is treated as a state rather than a control input in the acceleration-controlled unicycle model~\eqref{eq:affine_robot_model}.
The proof is as follows:
The outer normal vectors $\bm{N}_{\text{R}_i, k}^{\text{R}_j}, k \in \{1, 2\}$ and the tangent directional vectors $\bm{T}_{\text{R}_i, k}^{\text{R}_j}, k \in \{1, 2\}$ of RVO are identical to those of VO, with $\text{O}_j$ replaced by $\text{R}_j$ because RVO is utilized among robots.
The primary difference between RVO and VO lies in the position of the cone apex.
For $\text{RVO}_{\text{R}_i|\text{R}_j}(\bm{v}_{\text{R}_i}, \bm{v}_{\text{R}_j})$, the apex is at $\frac{\bm{v}_{\text{R}_i} + \bm{v}_{\text{R}_j}}{2}$, whereas for $\text{VO}_{\text{R}_i|\text{R}_j}(\bm{v}_{\text{R}_j})$, the apex is at $\bm{v}_{\text{R}_j}$.
The RVO-based collision avoidance constraint is similar to~\eqref{eq:vo_cons1} and is expressed as:
\begin{equation}
    \bigcup_{k \in \{1, 2\}} (\bm{N}_{\text{R}_i,k}^{\text{R}_j})^\top \bm{v}_{\text{R}_i} \geq c_{\text{R}_i,k}^{\text{R}_j},
    \label{eq:rvo_cons_1}
\end{equation}
where the scalar corresponding to the $k^{\text{th}}$ linear constraint of $\text{RVO}_{\text{R}_i|\text{R}_j}(\bm{v}_{\text{R}_i}, \bm{v}_{\text{R}_j})$ is given by
\begin{equation*}
    c_{\text{R}_i,k}^{\text{R}_j} = (\bm{N}_{\text{R}_i, k}^{\text{R}_j})^\top \frac{\bm{v}_{\text{R}_i} + \bm{v}_{\text{R}_j}}{2}, \; k \in \{1, 2\}.
\end{equation*}
We can construct RVO-based CBFs as:
\begin{gather}
    h_{\text{R}_i, 1}^{\text{R}_j}(\bm{x}) = (\bm{N}_{\text{R}_i,1}^{\text{R}_j})^\top \bm{v}_{\text{R}_i} - c_{\text{R}_i,1}^{\text{R}_j} = \frac{1}{2} (\bm{N}_{\text{R}_i,1}^{\text{R}_j})^\top \bm{v}_{\text{R}_i}^{\text{R}_j},
    \label{eq:rvo_cbf1} \\ 
    h_{\text{R}_i, 2}^{\text{R}_j}(\bm{x}) = (\bm{N}_{\text{R}_i,2}^{\text{R}_j})^\top \bm{v}_{\text{R}_i} - c_{\text{R}_i,2}^{\text{R}_j} = \frac{1}{2} (\bm{N}_{\text{R}_i,2}^{\text{R}_j})^\top \bm{v}_{\text{R}_i}^{\text{R}_j}.
    \label{eq:rvo_cbf2}
\end{gather}
By substituting $\text{R}_j$ for $\text{O}_j$, the constraint of~\eqref{eq:rvo_cbf1} becomes equivalent to the constraint of~\eqref{eq:vocbf1}, and similarly for~\eqref{eq:rvo_cbf2}.
Thus, constructing CBFs based on either VO or RVO is consistent. 
The primary reason is that VO-based approaches select a new velocity outside VO to achieve collision avoidance, and in the acceleration-controlled unicycle model~\eqref{eq:affine_robot_model}, velocity is treated as a state rather than a control input.
The same consistency applies to HRVO, as its cone apex is determined by a combination of VO and RVO.

In conclusion, we have demonstrated how to construct VOCBFs in the velocity space to achieve collision avoidance between robots and obstacles. 
Additionally, we prove that the constraints of these VOCBFs are always valid.
By ensuring that at least one of the two VOCBF constraints is satisfied, safety is guaranteed. 
Finally, we prove that constructing CBFs based on either VO or RVO is consistent, as velocity in our robot dynamics~\eqref{eq:affine_robot_model} is treated as a state rather than a control input.
\subsection{Physical Constraints on States}
\label{sec:physical_cons}
The states and controls of the robot are constrained within specified ranges based on the robot's physical capabilities.
Constraints~\eqref{eq:cons_u1}, which address the physical limitations of control inputs, can be directly incorporated into the optimization problem~\eqref{eq:clf_cbf_qp_optimal_problem} because control inputs are the optimization variables.
However, in addition to these input constraints, the robot's velocities $v$ and $\omega$ must also be restricted to specified ranges, i.e., $v_\text{min} \leq v \leq v_\text{max}$ and $\omega_\text{min} \leq \omega \leq \omega_\text{max}$.
Since $v$ and $\omega$ are states, rather than control inputs of~\eqref{eq:affine_robot_model}, these state constraints cannot be directly included in the optimization problem.
To address this, we propose enforcing state constraints on $v$ and $\omega$ through CBFs.
We define four CBFs to enforce the bounds as follows:
\begin{equation*}
\begin{aligned}
    & h_{v_\text{min}}(\bm{x}) = v - v_\text{min}, \; h_{\omega_\text{min}}(\bm{x}) = \omega - \omega_\text{min}, \\
    & h_{v_\text{max}}(\bm{x}) = v_\text{max} - v, \; h_{\omega_\text{max}}(\bm{x}) = \omega_\text{max} - \omega.
\end{aligned}
\end{equation*}
The relative degree of these CBFs is 1. 
The corresponding constraints derived from these CBFs are:
\begin{align}
    L_f h_{v_\text{min}}(\bm{x}) + L_g h_{v_\text{min}}(\bm{x})\bm{u} 
    + \mu_{v_\text{min}}(h_{v_\text{min}}(\bm{x})) &\geq 0, \label{eq:cons_vmin} \\
    L_f h_{v_\text{max}}(\bm{x}) + L_g h_{v_\text{max}}(\bm{x})\bm{u} 
    + \mu_{v_\text{max}}(h_{v_\text{max}}(\bm{x})) &\geq 0, \label{eq:cons_vmax} \\
    L_f h_{\omega_\text{min}}(\bm{x}) + L_g h_{\omega_\text{min}}(\bm{x})\bm{u} 
    + \mu_{\omega_\text{min}}(h_{\omega_\text{min}}(\bm{x})) &\geq 0, \label{eq:cons_wmin} \\
    L_f h_{\omega_\text{max}}(\bm{x}) + L_g h_{\omega_\text{max}}(\bm{x})\bm{u}
    + \mu_{\omega_\text{max}}(h_{\omega_\text{max}}(\bm{x})) &\geq 0, \label{eq:cons_wmax} 
\end{align}
where $\mu_{v_\text{min}}$, $\mu_{v_\text{max}}$, $\mu_{\omega_\text{min}}$ and $\mu_{\omega_\text{max}}$ represent different class $\mathcal{K}$ functions.
By incorporating these constraints into the optimization problem, the states $v$ and $\omega$ are effectively restricted within their specified ranges, ensuring compliance with the robot's physical limitations.

\subsection{Controller Synthesis}
In this section, we mainly demonstrate how to combine the constraints of CLFs, VOCBFs, and physical limitations on states and control inputs with the objective function to formulate an optimization problem like~\eqref{eq:clf_cbf_qp_optimal_problem}.
Given that there are two VOCBF constraints associated with each obstacle, and at least one of these two constraints must be satisfied, we propose introducing integer variables to handle this requirement.
Consequently, the optimization problem~\eqref{eq:clf_cbf_qp_optimal_problem} is converted into a mixed-integer quadratic programming (MIQP) problem as:
\noindent\rule{\columnwidth}{0.8pt}
\textbf{CLF-VOCBF-MIQP:}
\begin{subequations}
\begin{align}
    \min_{(\bm{u}, \bm{\delta}) \in \mathbb{R}^{m + 4}} &\frac{1}{2}\bm{u}^\top H \bm{u} + \frac{1}{2}(\bm{u} - \bm{u}_\text{pre})^\top R (\bm{u} - \bm{u}_\text{pre}) + \bm{\delta}^\top P \bm{\delta} \label{eq:miqp_objective} \\
    \text{s.t.} ~& L_f V_\text{d}(\bm{x}) + L_g V_\text{d}(\bm{x})\bm{u} + \gamma_\text{d}(V_\text{d}(\bm{x})) \leq \delta_\text{d}, \label{eq:voqp_clf1} \\
    & L_f V_\theta(\bm{x}) + L_g V_\theta(\bm{x})\bm{u} + \gamma_\theta(V_\theta(\bm{x})) \leq \delta_\theta, \label{eq:voqp_clf2} \\
    & L_f V_v(\bm{x}) + L_g V_v(\bm{x})\bm{u} + \gamma_v(V_v(\bm{x})) \leq \delta_v, \label{eq:voqp_clf3} \\
    & L_f V_\omega(\bm{x}) + L_g V_\omega(\bm{x})\bm{u} + \gamma_\omega(V_\omega(\bm{x})) \leq \delta_\omega \label{eq:voqp_clf4} \\ \notag
    & \eqref{eq:cons_vmin}, \eqref{eq:cons_vmax}, \eqref{eq:cons_wmin}, \eqref{eq:cons_wmax}, \\
    & -L_f h_{\text{R}_i, k}^{\text{O}_j}(\bm{x}) - L_g h_{\text{R}_i, k}^{\text{O}_j}(\bm{x})\bm{u} - \mu(h_{\text{R}_i, k}^{\text{O}_j}(\bm{x})) - \notag \\
    & \frac{d h_{\text{R}_i, k}^{\text{O}_j}(\bm{x})}{d \bm{x}_{\text{O}_j}} \frac{d \bm{x}_{\text{O}_j}}{dt} \leq G(1-z_{j, k}), \notag \\
    & j \in \{0, \dots, M-1 \}, k \in \{1, 2\}, \label{eq:cons_cbf_interger} \\
    & \sum_{k=1}^{2} z_{j, k} \geq 1, z_{j, k} \in \{0, 1\}, \label{eq:integer_cons} \\
    & \bm{u} \in \mathcal{U}, \Delta \bm{u} \in \Delta \mathcal{U}, \label{eq:cons_u} \\
    & \bm{\delta} \in \mathbb{R}^4, \label{eq:cons_delta}
\end{align}
\label{eq:vomiqp_optimal_problem}
\end{subequations}
\noindent\rule{\columnwidth}{0.8pt}
\noindent
where $H$, $R$ and $P$ are predefined positive definite matrices, $\bm{\delta}=[\delta_\text{d}, \delta_\theta, \delta_v, \delta_\omega]^\top$ represents the relaxation variables of CLFs,
$\bm{u}_\text{pre}$ represents the control input from the previous time step and is considered here to ensure smoother control input transitions.
The integer variable $z_{j, k}$ corresponds to the $k^\text{th}$ constraint of the VOCBF for the $j^\text{th}$ obstacle, determining whether the constraint is active ($z_{j, k} = 1$) or inactive ($z_{j, k} = 0$).
G is a large positive constant. 
The term $\frac{d h_{\text{R}_i, k}^{\text{O}_j}(\bm{x})}{d \bm{x}_{\text{O}_j}} \frac{d \bm{x}_{\text{O}_j}}{dt}$ accounts for the effect of dynamic obstacles and equals zero for static obstacles.
Constraint~\eqref{eq:integer_cons} ensures that at least one of the two VOCBF constraints is active for each obstacle, guaranteeing safety.
If $z_{j, 1} = 1$ or $z_{j, 2} = 1$, then the robot navigates around the obstacle from its left or right side, respectively. 
If both $z_{j, 1} = 1$ and $z_{j, 2} = 1$, then the robot moves backward to avoid a collision.
Constraints~\eqref{eq:voqp_clf1}, \eqref{eq:voqp_clf2}, \eqref{eq:voqp_clf3} and~\eqref{eq:voqp_clf4} corresponding to the CLFs are used to guide the robot towards its destination and relaxed by relaxation variables $\bm{\delta}$ to guarantee safety when they conflict with VOCBF constraints.
Moreover, as VO-based methods are prone to deadlocks due to their local collision avoidance nature, combining VOCBFs with CLFs helps address this issue effectively.
Weight factors in $P$ prioritize $\delta_\text{d}$ and $\delta_\theta$ over $\delta_v$ and $\delta_\omega$, as constraints~\eqref{eq:voqp_clf1} and~\eqref{eq:voqp_clf2} are primarily responsible for navigation.
Constraints~\eqref{eq:cons_vmin} - \eqref{eq:cons_wmax} restrict the robot’s states $v$ and $\omega$ to their specified ranges.
Constraint~\eqref{eq:cons_u} ensures that control inputs $\bm{u}$ remain within permissible bounds and the change rate of control inputs is constrained within a specified range.

Introducing integer variables makes the optimization problem~\eqref{eq:vomiqp_optimal_problem} relatively concise.
However, solving the MIQP problem~\eqref{eq:vomiqp_optimal_problem} can be time-consuming.
To improve computational efficiency, we propose splitting the original MIQP problem into multiple sub-optimization problems in the form of CLF-VOCBF-QPs.
Specifically, for $M$ obstacles, there will be $3^M$ sub-optimization problems.
Consider the case with one obstacle. 
Since at least one of the two VOCBF constraints must be satisfied, the original MIQP problem is split into three sub-optimization problems:
\begin{enumerate}
    \item Only the constraint of $h_{\text{R}_i, 1}^{\text{O}_j}(\bm{x})$ is considered.
    \item Only the constraint of $h_{\text{R}_i, 2}^{\text{O}_j}(\bm{x})$ is considered.
    \item Both constraints of $h_{\text{R}_i, 1}^{\text{O}_j}(\bm{x})$, $h_{\text{R}_i, 2}^{\text{O}_j}(\bm{x})$ are considered.
\end{enumerate}
All other constraints remain consistent.
Since all constraints are affine in the control inputs $\bm{u}$, each sub-optimization problem takes the form of a CLF-VOCBF-QP, which can be efficiently solved in real-time.
The detailed formulation for the third scenario, where both VOCBF constraints are considered, is as follows:

\noindent\rule{\columnwidth}{0.8pt}
\textbf{CLF-VOCBF-QPs:}
\begin{subequations}
\begin{align}
    \min_{(\bm{u}, \bm{\delta}) \in \mathbb{R}^{m + 4}} & \frac{1}{2}\bm{u}^\top H \bm{u} + \frac{1}{2}(\bm{u} - \bm{u}_\text{pre})^\top R (\bm{u} - \bm{u}_\text{pre}) + \bm{\delta}^\top P \bm{\delta} \label{eq:voqp_prob} \\ 
    \text{s.t.} ~
    & \eqref{eq:voqp_clf1}, \eqref{eq:voqp_clf2}, \eqref{eq:voqp_clf3}, \eqref{eq:voqp_clf4}, \notag \\
    & \eqref{eq:cons_vmin}, \eqref{eq:cons_vmax}, \eqref{eq:cons_wmin}, \eqref{eq:cons_wmax}, \eqref{eq:cons_u}, \eqref{eq:cons_delta}, \notag \\
    & L_f h_{\text{R}_i, k}^{\text{O}_j}(\bm{x}) + L_g h_{\text{R}_i, k}^{\text{O}_j}(\bm{x}) + \frac{d h_{\text{R}_i, k}^{\text{O}_j}(\bm{x})}{d \bm{x}_{\text{O}_j}} \frac{d \bm{x}_{\text{O}_j}}{dt} \notag \\
    & + \mu(h_{\text{R}_i, k}^{\text{O}_j}(\bm{x})) \geq 0, j \in \{0, \dots, M-1 \}, k \in \{1, 2\}, \label{eq:cons_cbf} 
\end{align}
\label{eq:voqp_optimal_problem}
\end{subequations}
\noindent\rule{\columnwidth}{0.8pt}
\noindent
where constraints~\eqref{eq:cons_cbf} indicate that both constraints of $h_{\text{R}_i, 1}^{\text{O}_j}(\bm{x})$ and $h_{\text{R}_i, 2}^{\text{O}_j}(\bm{x})$ are considered in this sub-optimization problem.
The formulations for the other two scenarios, where only one of the VOCBF constraints is considered, are omitted for brevity.
Since all feasible integer solutions in the original MIQP problem are covered across these sub-optimization problems, the optimal solution for the original MIQP problem is obtained by solving each sub-problem and selecting the one with the minimum objective function.
For the case with one obstacle, after solving the three CLF-VOCBF-QPs, the solution with the smallest objective value is chosen as the final optimal solution of the original MIQP problem~\eqref{eq:vomiqp_optimal_problem}.

When applying CLF-VOCBF-QPs to multiple obstacles, the number of sub-optimization problems increases exponentially with the number of obstacles.
To further improve computational efficiency, we first determine whether the feasible region of each sub-optimization problem is empty or not before solving it. 
Specifically, we check whether
\begin{equation}
    \mathcal{A}:=\{\bm{u} \in \mathbb{R}^{m} \, \big| \, \text{subject to} \; \eqref{eq:cons_vmin} - \eqref{eq:cons_wmax}, \eqref{eq:cons_u}, \eqref{eq:cons_cbf} \}
    \label{eq:fea_detect}
\end{equation}
is empty or not.
Since all constraints in $\mathcal{A}$ are affine in $\bm{u}$, this can be determined using linear programming (LP).
Constraints of CLFs are excluded from $\mathcal{A}$ as they are relaxed through the relaxation variables and are always satisfied.
Thus, only hard constraints, such as those for collision avoidance and physical limitations, are considered.
If $\mathcal{A}$ is not empty, we proceed to solve the corresponding sub-optimization problem~\eqref{eq:voqp_optimal_problem}.

\begin{figure*} 
    \centering
    \includegraphics[width=1.0\linewidth]{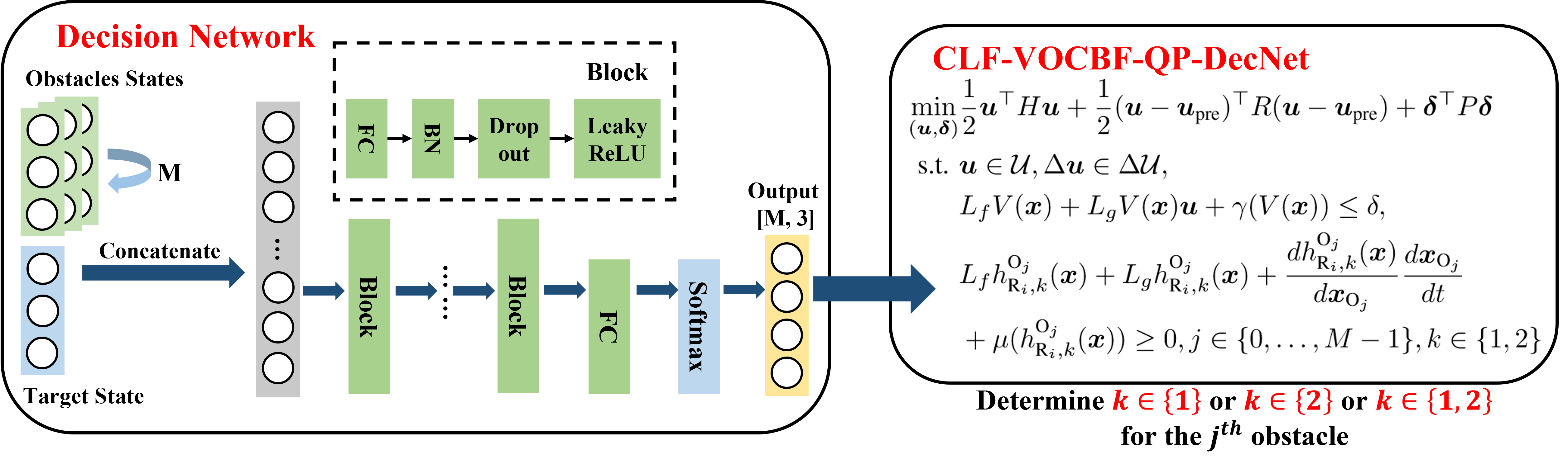}
    \caption{Structure of the decision network. The decision network takes as input the states of $M$ obstacles and the robot's target state, both represented in the robot's local frame. It outputs the probabilities for the robot to navigate around each obstacle in three possible directions (left, right, or backward). These probabilities are used to determine the active constraints in the formulation of the CLF-VOCBF-QP-DecNet optimization problem.}
    \label{fig:decision_net}
\end{figure*}
As the number of sub-optimization problems grows exponentially with the number of obstacles, it becomes challenging to apply the above methods in environments with a large number of obstacles.
To address this, we propose reducing computational costs further using a neural network.
Specifically, we design a decision network to assist the robot in determining which side of each obstacle to navigate around, and the structure of the decision network is shown in Fig.~\ref{fig:decision_net}.

The decision network takes as input the states of $M$ obstacles and the robot's target state, all represented in the robot's local frame.
The network first performs global feature aggregation by extracting information from all obstacles using average pooling.
The global obstacle information is then concatenated with the information of each individual obstacle and the robot's target state to ensure that decisions consider both local and global context.
The decision network outputs the probabilities for the robot to navigate around each obstacle in three directions, i.e., $\sum p_z = 1, z \in \{0, 1, 2\}$, where $p_0$ represents the probability of navigating around the obstacle from the left side, $p_1$ is the probability of navigating around the right side, and $p_2$ is the probability of moving backward (rear side).
For each obstacle, the robot chooses the direction corresponding to the highest probability as its decision. 
For example, if $p_0$ is the largest, then $k \in \{1\}$ in constraint~\eqref{eq:cons_cbf}, meaning the robot will navigate around from the left side; 
if $p_1$ is the largest, then $k \in \{2\}$ in constraint~\eqref{eq:cons_cbf}, meaning the robot will navigate around from the right side; 
if $p_2$ is the largest, then $k \in \{1, 2\}$ in constraint~\eqref{eq:cons_cbf}, meaning  both constraints are considered, and the robot moves backward.

With the help of the decision network, the direction for navigating around each obstacle is determined in advance. 
As a result, only a single sub-optimization problem needs to be solved instead of $3^M$ sub-optimization problems.
This significantly reduces the computational burden, enabling the method to handle environments with a large number of obstacles efficiently.
\subsection{Analysis of Runtime Complexity Bounds and Feasibility}
In this section, we provide a detailed analysis of the runtime complexity bounds for our proposed methods and examine the feasibility of the optimization problem.

\begin{table}
\caption{Runtime complexity bounds w.r.t. different methods.}
\label{tab:runtime}
\centering
\begin{threeparttable}
\renewcommand{\arraystretch}{1.0}
\begin{tabular}{>{\raggedright\arraybackslash}p{4cm}cc}
\hline
\multirow{2}{*}{Methods} & \multicolumn{2}{c}{Runtime complexity bounds\tnote{a}} \\ \cline{2-3} 
               & Lower   & Upper \\ \hline
CLF-VOCBF-MIQP & $O_1$  & $4^M O_1$   \\
CLF-VOCBF-QPs  & $3^M O_2 + O_1$ & $3^M (O_1 + O_2)$   \\
CLF-VOCBF-QP-DecNet & $O_1$ & $O_1$ \\ \hline
\end{tabular}
\begin{tablenotes}
    \item [] \textsuperscript{a} The runtime complexity $O_1$ of a QP is typically $O(n^{3.5} + n^{2.5}m)$, while the runtime complexity $O_2$ of an LP is $O(n^2m)$, where $n$ represents the number of optimization variables, $m$ is the number of constraints, and $M$ denotes the number of obstacles.
\end{tablenotes}
\end{threeparttable}
\end{table}

We propose three distinct methods to solve the original MIQP problem:
\textbf{(i) CLF-VOCBF-MIQP} directly solves the MIQP problem;
\textbf{(ii) CLF-VOCBF-QPs} splits the original MIQP into multiple sub-optimization problems and solves each to determine the optimal solution;
\textbf{(iii) CLF-VOCBF-QP-DecNet} utilizes a decision network to guide the process and only needs to solve one sub-optimization problem.
To compare the computational efficiency of these methods, we provide detailed runtime complexity bounds, as summarized in Table~\ref{tab:runtime}, and we assume the original MIQP problem is feasible.
For simplicity, we denote the runtime complexity of a QP as $O_1$, and that of an LP as $O_2$.
For an optimization problem with $n$ optimization variables and $m$ constraints, $O_1$ is typically $O(n^{3.5} + n^{2.5}m)$, while $O_2$ is typically $O(n^2m)$.
MIQP is generally solved using the branch-and-bound method. 
The lower bound of its runtime occurs when the solution of the initial relaxation problem satisfies all integer variable constraints exactly, requiring only a single QP to be solved. 
The upper bound, however, is determined by the total number of possible combinations of the $2M$ integer variables in ~\eqref{eq:vomiqp_optimal_problem}, where each variable $z \in \{0, 1\}$.\
As a result, the upper bound runtime is $2^{2M} * O_1$.
For CLF-VOCBF-QPs, the process involves checking feasibility as described in~\eqref{eq:fea_detect} before solving each sub-QP.
The lower bound of the runtime occurs when only one sub-optimization problem is feasible, and it also corresponds to the optimal solution, while the upper bound occurs when all sub-problems are feasible and need to be solved.
In contrast, for CLF-VOCBF-QP-DecNet, which is guided by the decision network, only one QP is solved. 
Thus, both the lower and upper bounds of its runtime are the same, significantly reducing computational cost.

We also provide a feasibility analysis for our approach. 
Since CLF constraints are relaxed by relaxation variables, the optimization problem becomes infeasible only when the collision avoidance constraints~\eqref{eq:cons_cbf} conflict with the physical limitation constraints~\eqref{eq:cons_vmin}, \eqref{eq:cons_vmax}, \eqref{eq:cons_wmin}, \eqref{eq:cons_wmax} and \eqref{eq:cons_u}.
The feasibility of the optimization problem depends on the choice of class $\mathcal{K}$ functions and the robot's locomotion capabilities.
As proposed in~\cite{zeng2021feasibility}, dynamically adjusting the decay rate of CBFs can enhance feasibility, and this method can also be incorporated into our approach.

Moreover, VO-based constraints typically require $\bm{v}_{\text{R}_i} \notin \text{VO}_{\text{R}_i|\text{O}_j}(\bm{v}_{\text{O}_j})$ to ensure safety.
However, this condition is often overly conservative, as safety is generally ensured based on the current distance between the robot and the obstacle, even though it may be deemed unsafe under VO-based constraints.
To address this, \cite{roncero2024multi} suggests relaxing VO-based constraints and imposing simpler hard constraints to ensure safety, like designing a following CBF:
\begin{equation*}
    h(\bm{x}) = \|\bm{p}_{\text{R}_i}^{\text{O}_j} \| - d_s - \frac{\bm{v}_{\text{R}_i}^{\text{O}_j}}{2a_\text{max}},
\end{equation*}
where $d_s$ is the safe margin.
This strategy can also be applied to our approach. 
However, since the primary focus of this work is to introduce the construction of VOCBFs and the efficient solving of optimization problems, we refer readers to the original works for more details on these interesting methods.

In conclusion, we have demonstrated how to integrate all constraints, including CLFs, VOCBFs, and physical limitations, to formulate and efficiently solve the optimization problem.
With the proposed decision network, the computational cost is significantly reduced.
Furthermore, when applying our approach to distributed multi-robot systems, the constraints of RVO-based CBFs are equivalent to those of VOCBFs, allowing us to incorporate the corresponding VOCBF constraints into the optimization problem.

\section{Numerical Simulations}
\label{sec:num}
\subsection{Implementation Details}
\label{5-1}
\begin{table}
    \caption{Setup of Simulation Parameters.}
    \label{tab:simulation_params}
    \centering
    \begin{tabular}{l|l|l}
    \hline
    Notation & Meaning & Value     \\ \hline
    $\Delta t$ & Time step of simulation & $0.05 \, \si[per-mode=symbol]{\second}$ \\
    $r$ & Radius of robot & $0.3 \, \si[per-mode=symbol]{\metre}$ \\
    $l$ & Distance between center and rear axle axis & $0.15 \, \si[per-mode=symbol]{\metre}$ \\
    $d_\text{s}$ & Safe margin for collision avoidance & $0.15 \, \si[per-mode=symbol]{\metre}$ \\
    $v_\text{min}$ & Robot's minimum linear velocity & $0.0 \, \si[per-mode=symbol]{\metre\per\second}$ \\
    $v_\text{max}$ & Robot's maximum linear velocity & $4.0 \, \si[per-mode=symbol]{\metre\per\second}$ \\
    $\omega_\text{max}$ & Robot's maximum angular velocity & $0.5 \, \si[per-mode=symbol]{\radian\per\second}$ \\
    $a_\text{max}$ & Robot's maximum linear acceleration & $1.0 \, \si[per-mode=symbol]{\metre\per\second^2}$ \\
    $\alpha_\text{max}$ & Robot's maximum angular acceleration & $0.6 \, \si[per-mode=symbol]{\radian\per\second^2}$ \\
    $\Delta a_\text{max}$ & Maximum change rate of linear acceleration & $6.0 \, \si[per-mode=symbol]{\metre\per\second^3}$ \\
    $\Delta \alpha_\text{max}$ & Maximum change rate of angular acceleration & $3.0 \, \si[per-mode=symbol]{\radian\per\second^3}$ \\
    $\gamma(\cdot)$ & Class $\mathcal{K}$ functions for all CLFs & 1.0 \\
    $\mu(\cdot)$ & Class $\mathcal{K}$ functions for all CBFs & 1.0 \\
    $\mu_\text{hocbf}(\cdot)$ & Class $\mathcal{K}$ functions for HOCBFs & 0.75, 0.65 \\
    \hline     
    \end{tabular}%
\end{table}
We conducted various numerical simulations to validate the effectiveness and performance of our proposed approach.
All numerical simulations are conducted using Python on an Ubuntu Laptop with an Intel Core i9-13900HX CPU.
The decision network, consisting of five blocks (as shown in Fig.~\ref{fig:decision_net}), is trained using PyTorch on a GeForce RTX 4060 GPU.
To train the decision network, we generated a dataset with 2,000 random scenarios. 
These scenarios included randomly generated start and target positions for the robot, initial positions and velocities for two circular obstacles, and random radii for both the robot and obstacles.
Moreover, we totally propose three methods to solve the original MIQP problem:
\textbf{(i) CLF-VOCBF-MIQP} (original form), \textbf{(ii) CLF-VOCBF-QPs} (splitting the MIQP into multiple sub-optimization problems) and \textbf{(iii) CLF-VOCBF-QP-DecNet} (guided by the decision network, requiring only the solution of a single sub-optimization problem).
The parameters for the above optimization problems are identical, with class $\mathcal{K}$ functions chosen as linear scalars for simplicity.
Physical constraints on the robot's kinematics are incorporated, including $v_\text{min} \leq v \leq v_\text{max}$, $-\omega_\text{max} \leq \omega \leq \omega_\text{max}$, $-a_\text{max} \leq a \leq a_\text{max}$, $-\alpha_\text{max} \leq \alpha \leq \alpha_\text{max}$, $-\Delta a_\text{max} \leq \Delta a \leq \Delta a_\text{max}$ and $-\Delta \alpha_\text{max} \leq \Delta \alpha \leq \Delta \alpha_\text{max}$.
A safety margin $d_s$ is added to the robot's radius to meet safety requirements.
The simulation time step $\Delta t$ is set to $0.05 \, \si[per-mode=symbol]{\second}$.
Parameters for the robot’s kinematics and the optimization controller are listed in Table~\ref{tab:simulation_params}.
We used BONMIN~\cite{sutradhar2016minlp} to solve CLF-VOCBF-MIQP and qpOASES~\cite{ferreau2014qpoases} for CLF-VOCBF-QPs and CLF-VOCBF-QP-DecNet.
In Section \ref{sec:5-2}, \ref{sec:5-3}, and \ref{sec:5-4}, we present the simulation results and comparisons with state-of-the-art (SOTA) methods.

\subsection{Collision Avoidance with Obstacles}
\label{sec:5-2}
\label{subsec:co}
\begin{figure}
    \centering
    \begin{subfigure}{0.45\linewidth}
        \centering
        \includegraphics[width=0.98\linewidth]{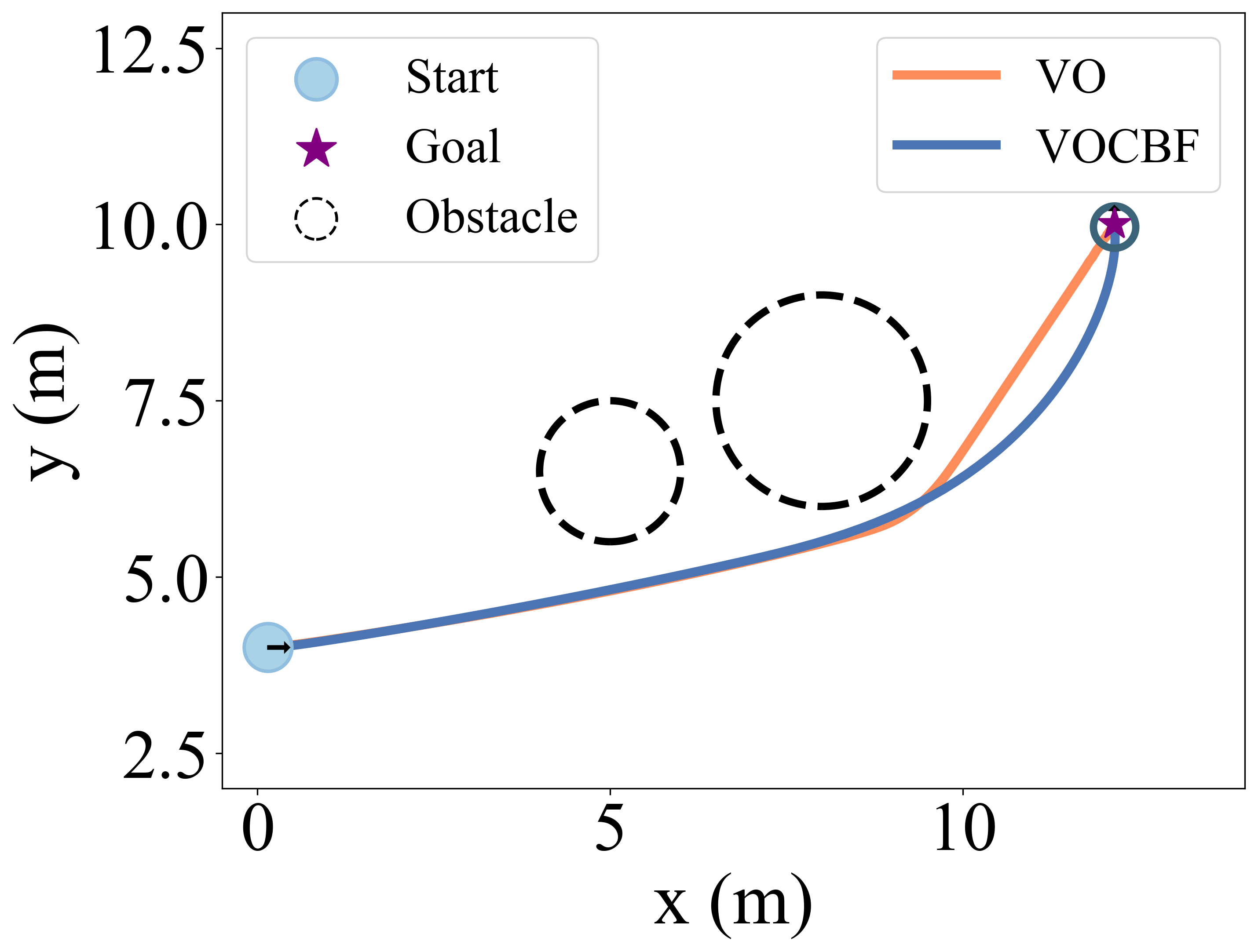}
        \caption{With static obstacles}
        \label{subfig:static_obs}
    \end{subfigure}
    \centering
    \begin{subfigure}{0.45\linewidth}
        \centering
        \includegraphics[width=0.98\linewidth]{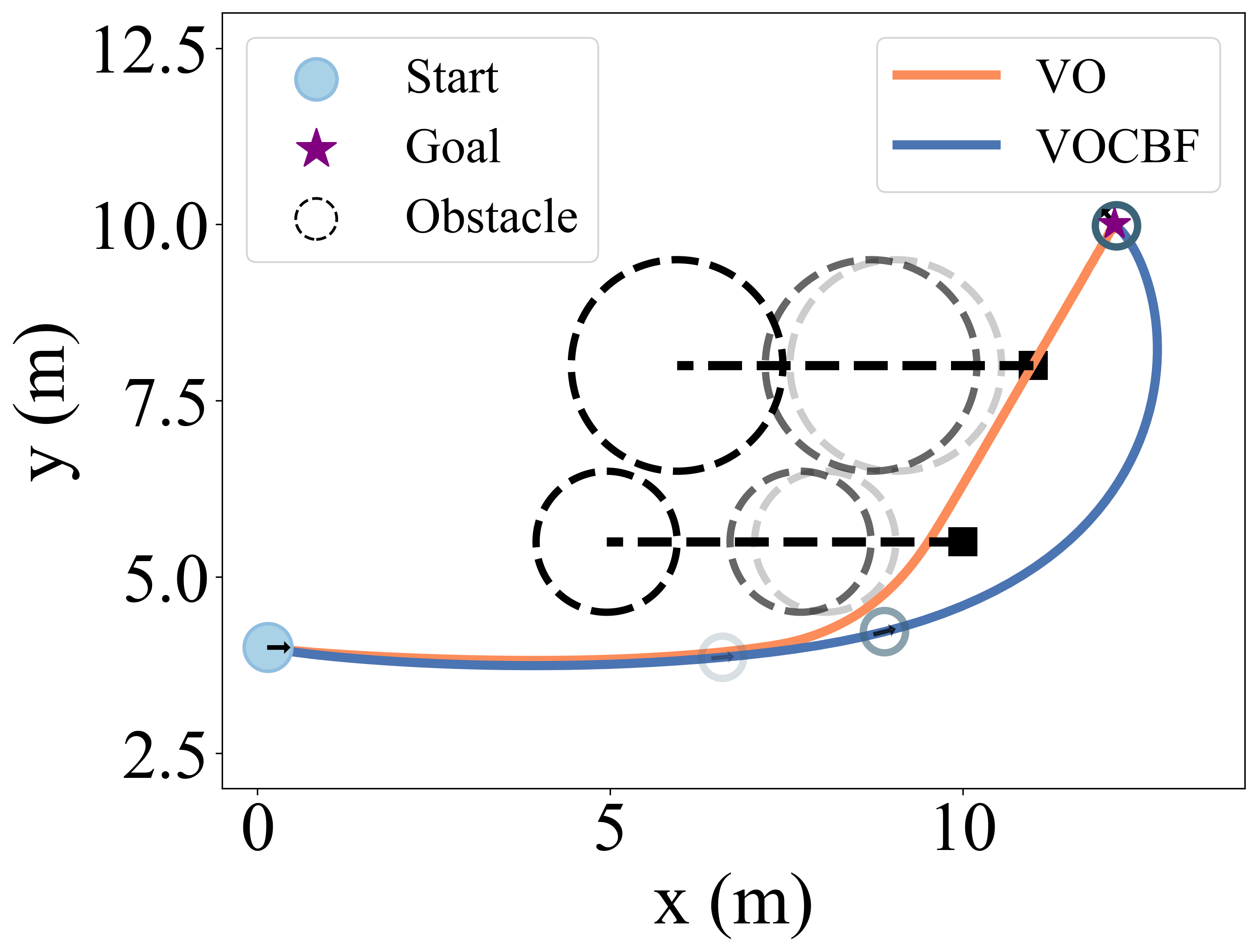}
        \caption{With dynamic obstacles}
        \label{subfig:dynamic_obs}
    \end{subfigure}
    \caption{Simulation results of guiding the robot to its destination while avoiding collisions with static and dynamic obstacles. The robot's start and target positions are represented by the light blue circle and purple star, respectively, while the black squares indicate the start positions of the dynamic obstacles. All obstacles are represented by black dashed circles. In (b), the robot’ positions over time, along with those of the dynamic obstacles, are illustrated using color gradients.}
    \label{fig:navigation}
\end{figure}
\begin{figure}
    \centering
    \begin{subfigure}{0.45\linewidth}
        \centering
        \includegraphics[width=0.98\linewidth]{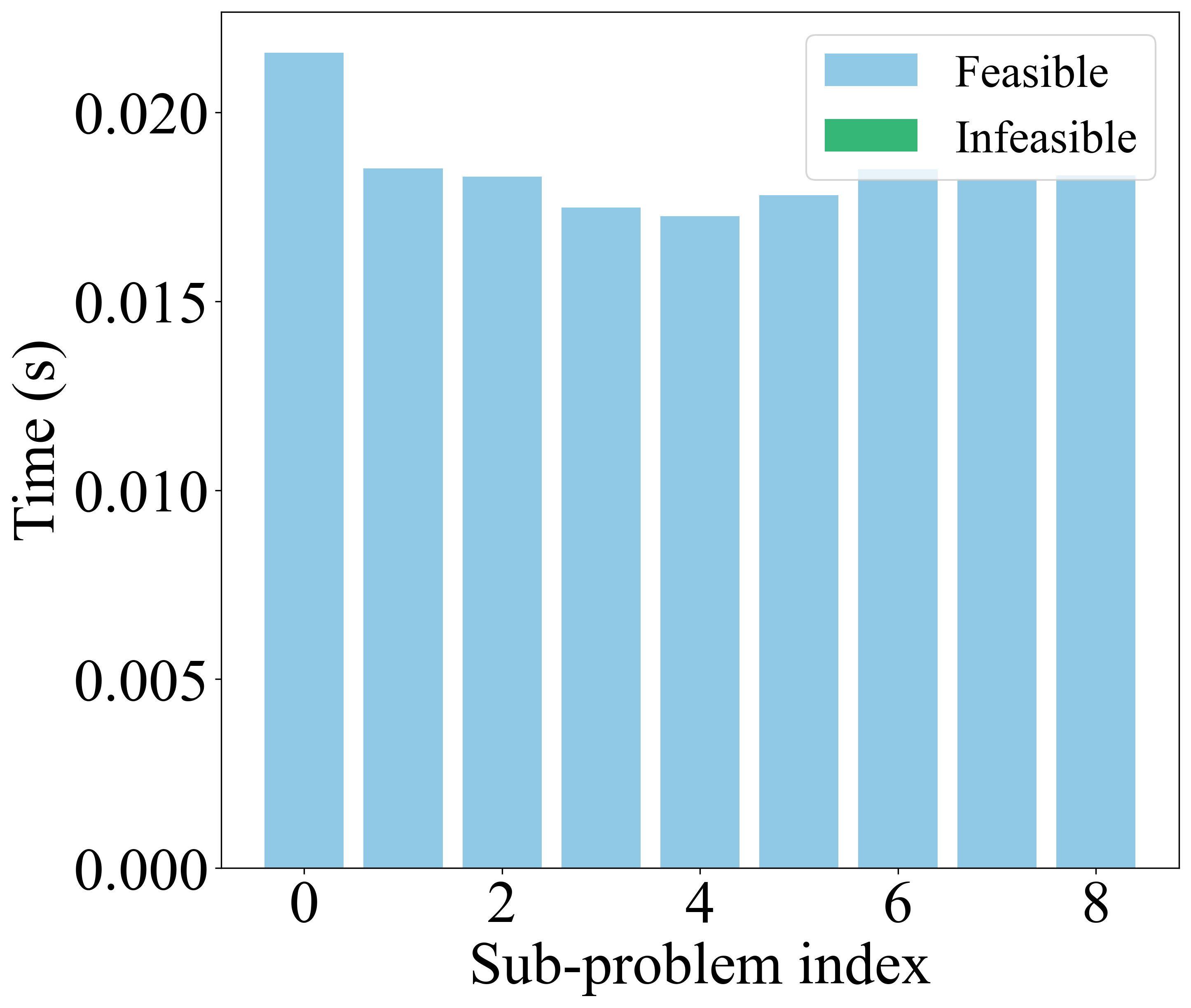}
        \caption{Case 1}
        \label{subfig:case1}
    \end{subfigure}
    \centering
    \begin{subfigure}{0.45\linewidth}
        \centering
        \includegraphics[width=0.98\linewidth]{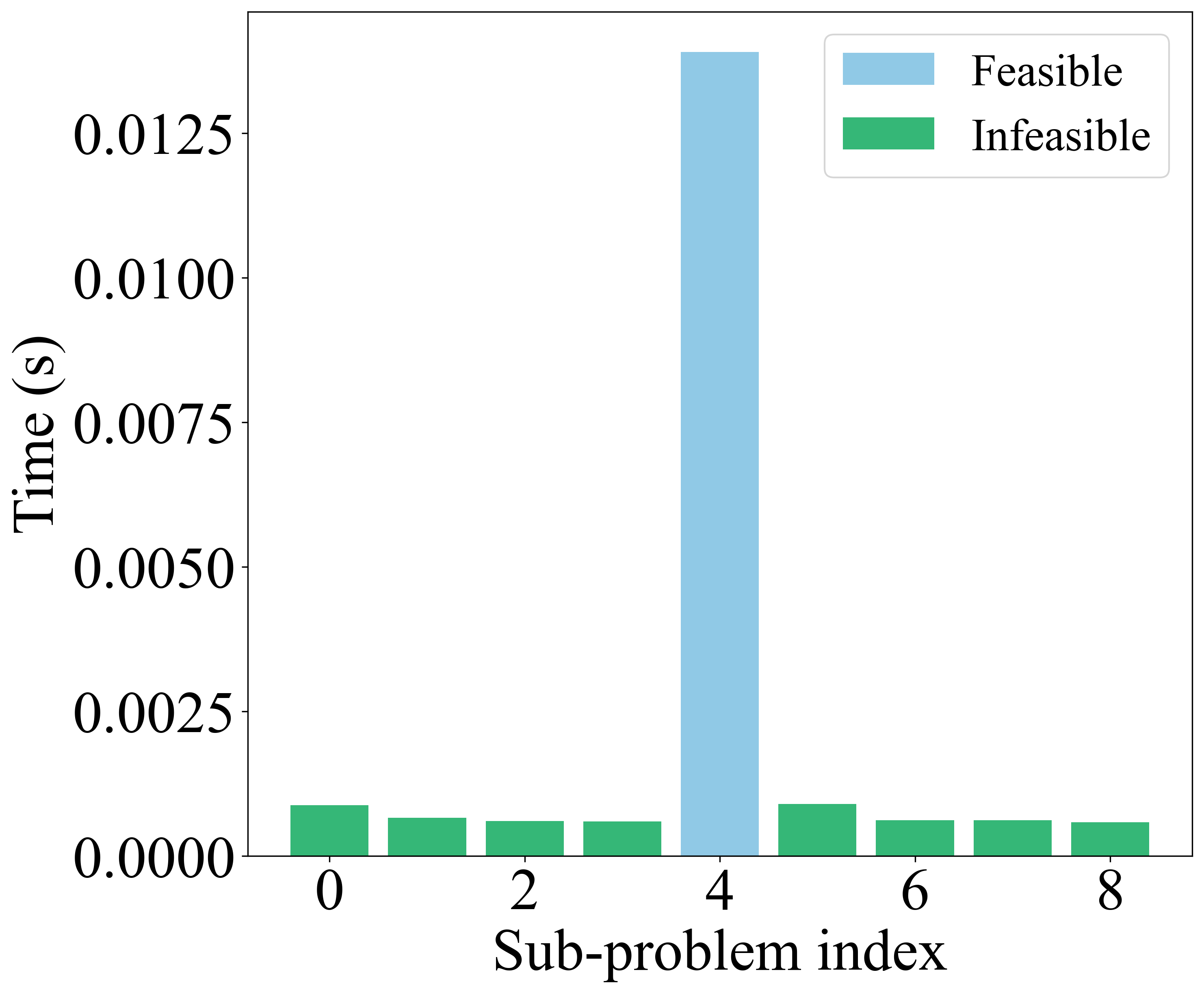}
        \caption{Case 2}
        \label{subfig:case2}
    \end{subfigure}
    \caption{Explanation of CLF-VOCBF-QPs which splits the original MIQP into multiple sub-optimization problems. For two obstacles, nine sub-problems are generated. To enhance efficiency, an LP is used to check if the feasible region $\mathcal{A}$ is empty before solving each problem.} 
    \label{fig:explain_qp}
\end{figure}
In this section, we first demonstrate the effectiveness of our approach in guiding a robot to its destination while avoiding collisions with static and dynamic obstacles. 
We design two scenarios: one with two static obstacles and the other with two dynamic obstacles moving at constant velocities.
The robot's start and target positions are the same in both scenarios, located at $(0 \, \si[per-mode=symbol]{\metre}, 4 \, \si[per-mode=symbol]{\metre})$ and $(12 \, \si[per-mode=symbol]{\metre}, 10 \, \si[per-mode=symbol]{\metre})$, respectively.
The velocities of both dynamic obstacles are $(-0.5 \, \si[per-mode=symbol]{\metre\per\second}, 0 \, \si[per-mode=symbol]{\metre\per\second})$.
Notably, the final solutions from CLF-VOCBF-MIQP and CLF-VOCBF-QPs are identical, differing only in computational efficiency. 
CLF-VOCBF-QP-DecNet, guided by the decision network, may occasionally yield slightly different results due to its reliance on the network.
Simulation results confirm that our approach successfully navigates the robot to its destination while avoiding collisions, as shown in Fig.~\ref{fig:navigation}, with all three methods yielding the same results in this scenario.
The robot’s positions over time, along with those of the dynamic obstacles, are illustrated using color gradients. 

\begin{table}
\caption{Computation time of our methods w.r.t. different numbers of obstacles.}
\label{tab:computation_time}
\centering
\resizebox{\columnwidth}{!}{%
\begin{tabular}{cccccc}
\hline
\multirow{2}{*}{Method} & \multirow{2}{*}{Num of obs} & \multicolumn{4}{c}{Time (ms)} \\ \cline{3-6} 
 &   & Min    & Max    & Median & Avg    \\ \hline
 \multirow{3}{*}{\begin{tabular}[c]{@{}c@{}}CLF-VOCBF-MIQP\end{tabular}} 
 & 1 & 25.3 & 165.9 & 62.9 & 63.1 \\
 & 2 & 32.1 & 534.1 & 155.6 & 173.3 \\
 & 3 & 36.2 & 2092.1 & 324.4 & 477.8 \\ \hline
\multirow{3}{*}{\begin{tabular}[c]{@{}c@{}}CLF-VOCBF-QPs\end{tabular}}  
 & 1 & 9.6 & 44.9 & 29.3 & 25.7 \\
 & 2 & 13.7 & 178.9 & 46.2 & 72.1 \\
 & 3 & 27.7 & 1223.3 & 193.9 & 309.7 \\ \hline
\multirow{3}{*}{\begin{tabular}[c]{@{}c@{}}CLF-VOCBF-QP-DecNet\end{tabular}}  
 & 1 & 8.4 & 12.8 & 9.5  & \textbf{9.6} \\
 & 2 & 8.8 & 13.6 & 9.8  & \textbf{10.4} \\
 & 3 & 9.2 & 16.8 & 10.6 & \textbf{10.8} \\ \hline
\end{tabular}
}
\end{table}
\begin{figure}
    \centering
    \begin{subfigure}{0.45\linewidth}
        \centering
        \includegraphics[width=0.98\linewidth]{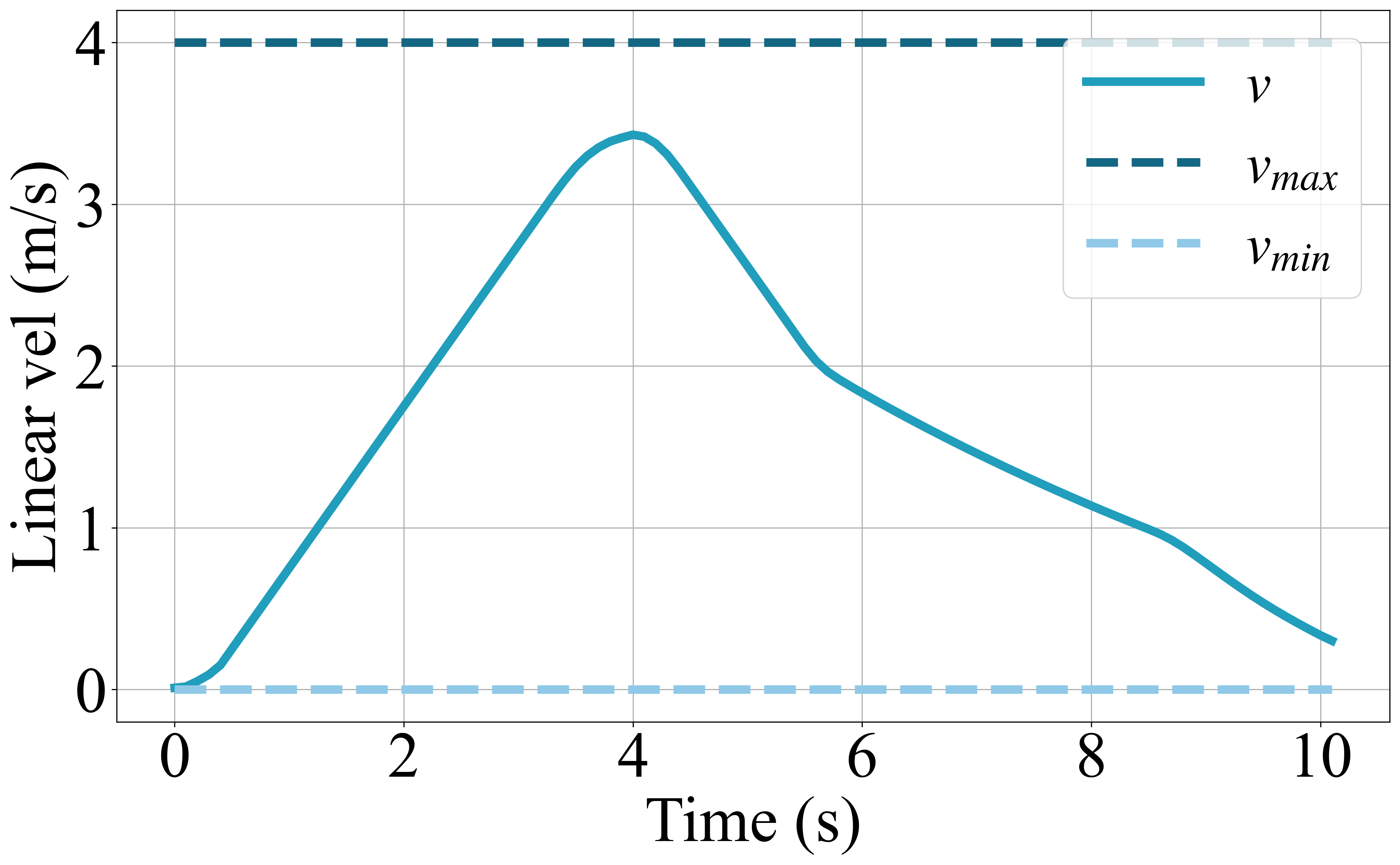}
        \caption{VOCBF}
        \label{subfig:state_v}
    \end{subfigure}
    \centering
    \begin{subfigure}{0.45\linewidth}
        \centering
        \includegraphics[width=0.98\linewidth]{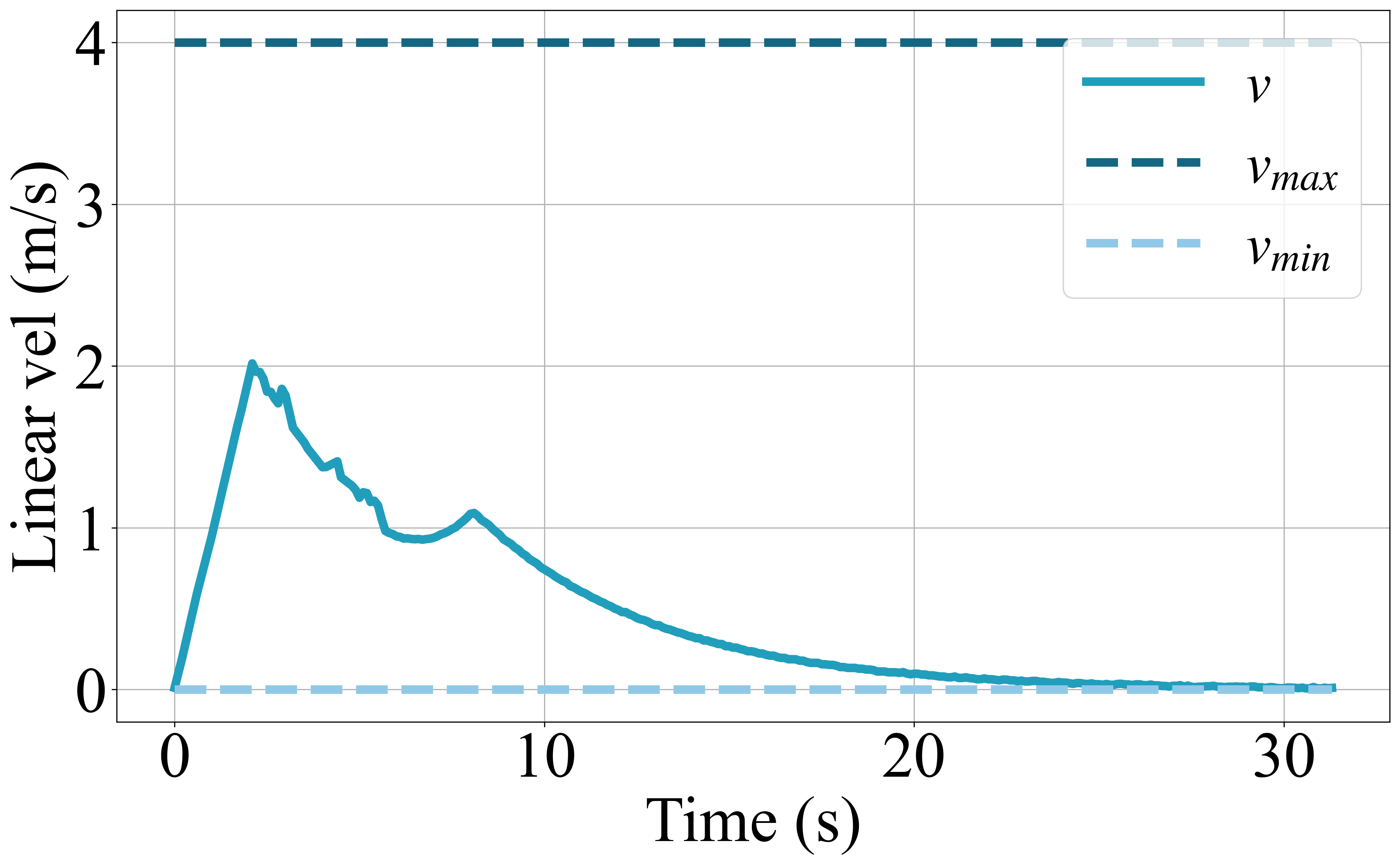}
        \caption{VO}
        \label{subfig:state_v_vo}
    \end{subfigure}

    \begin{subfigure}{0.45\linewidth}
        \centering
        \includegraphics[width=0.98\linewidth]{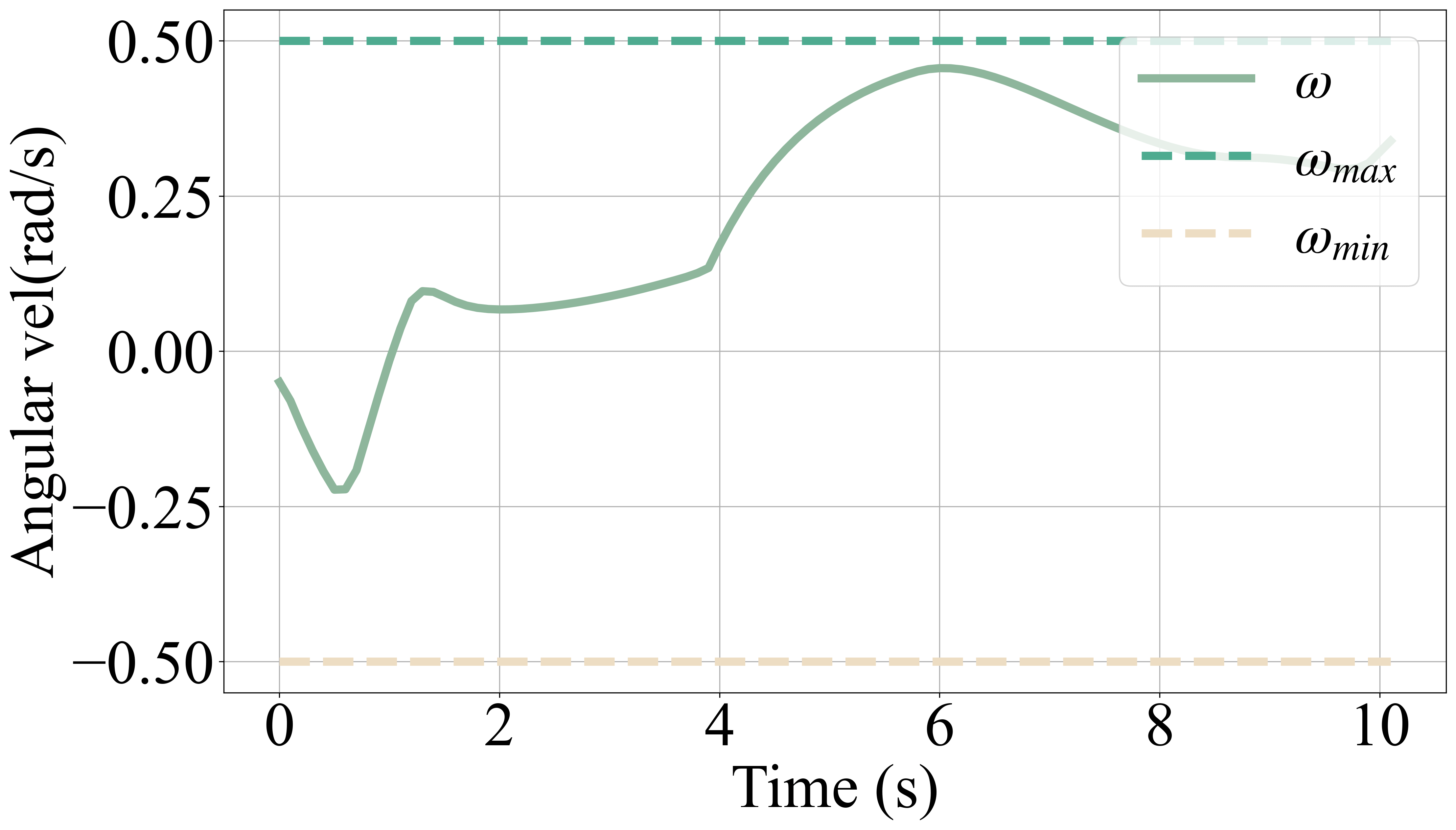}
        \caption{VOCBF}
        \label{subfig:state_w}
    \end{subfigure}
    \centering
    \begin{subfigure}{0.45\linewidth}
        \centering
        \includegraphics[width=0.98\linewidth]{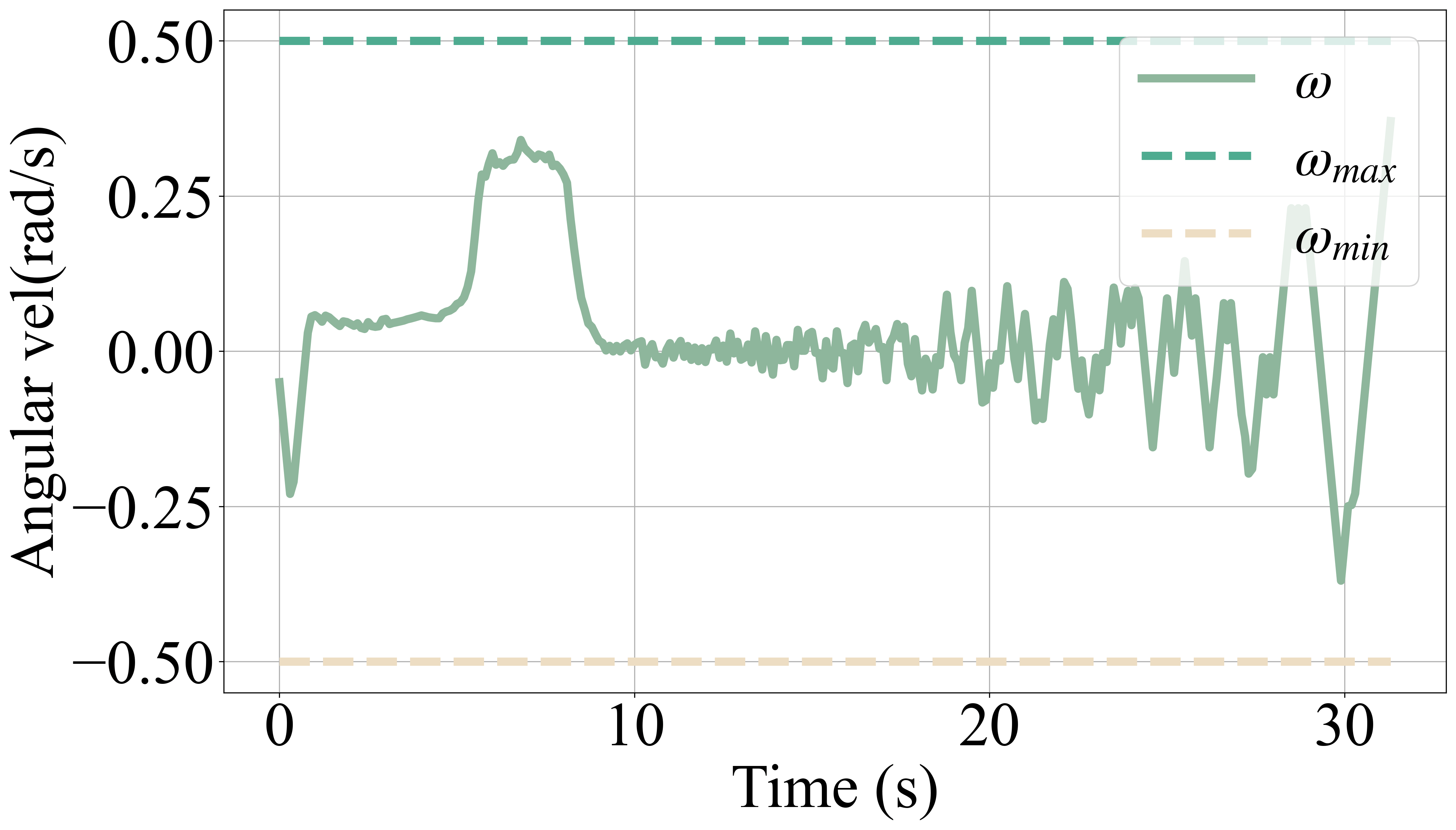}
        \caption{VO}
        \label{subfig:state_w_vo}
    \end{subfigure}

    \begin{subfigure}{0.45\linewidth}
        \centering
        \includegraphics[width=0.98\linewidth]{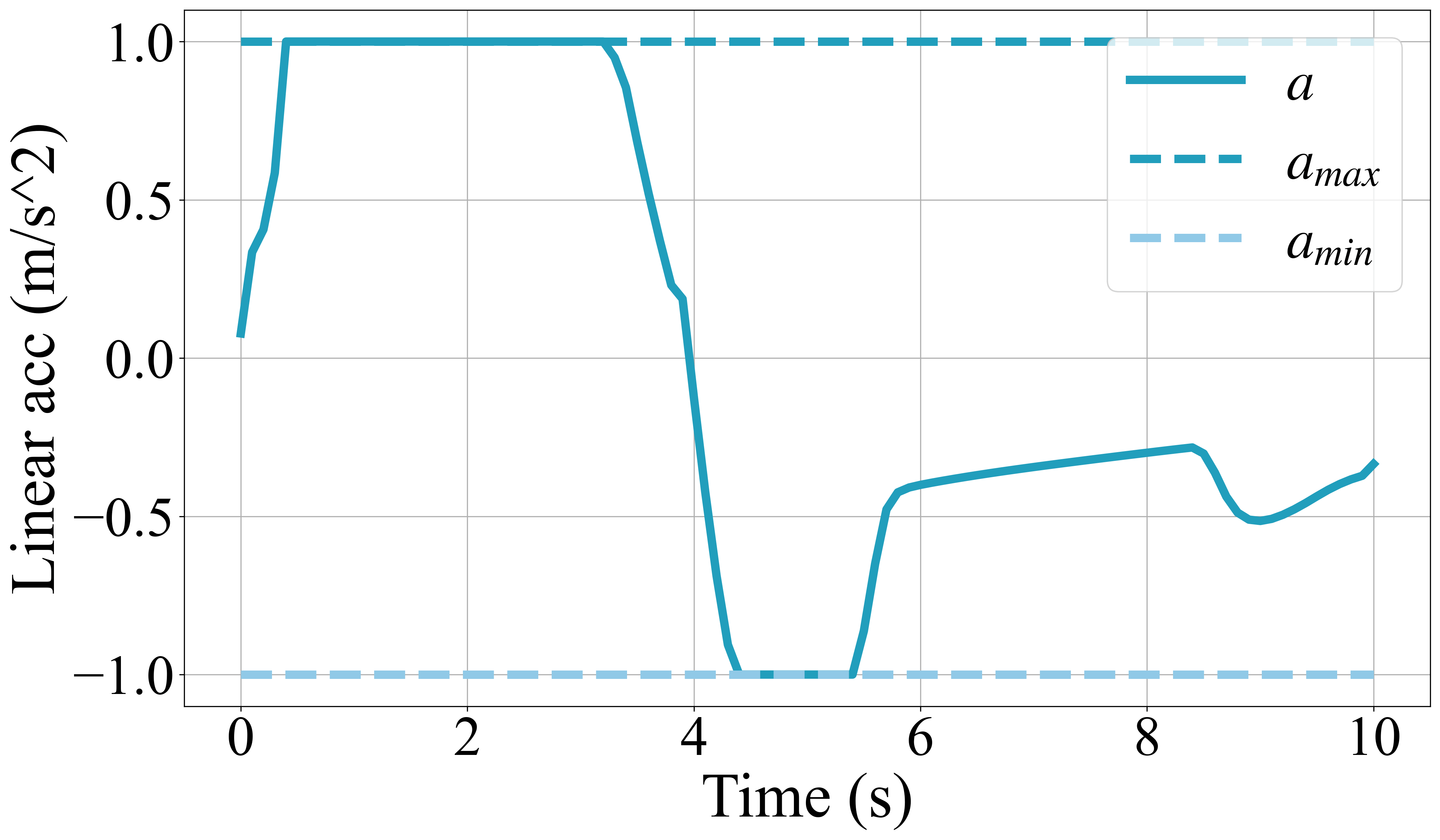}
        \caption{VOCBF}
        \label{subfig:control_v}
    \end{subfigure}
    \centering
    \begin{subfigure}{0.45\linewidth}
        \centering
        \includegraphics[width=0.98\linewidth]{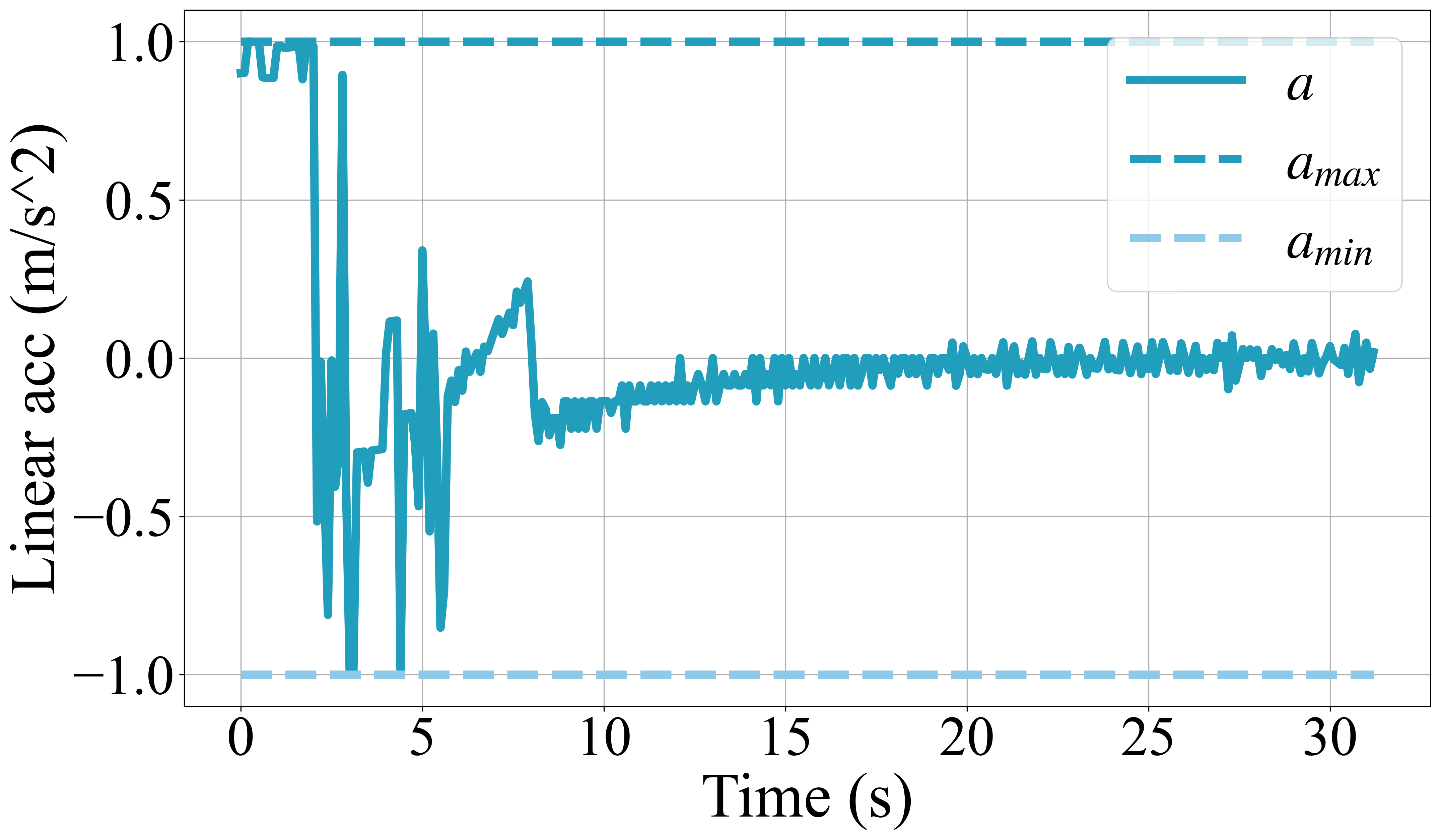}
        \caption{VO}
        \label{subfig:control_v_vo}
    \end{subfigure}

    \begin{subfigure}{0.45\linewidth}
        \centering
        \includegraphics[width=0.98\linewidth]{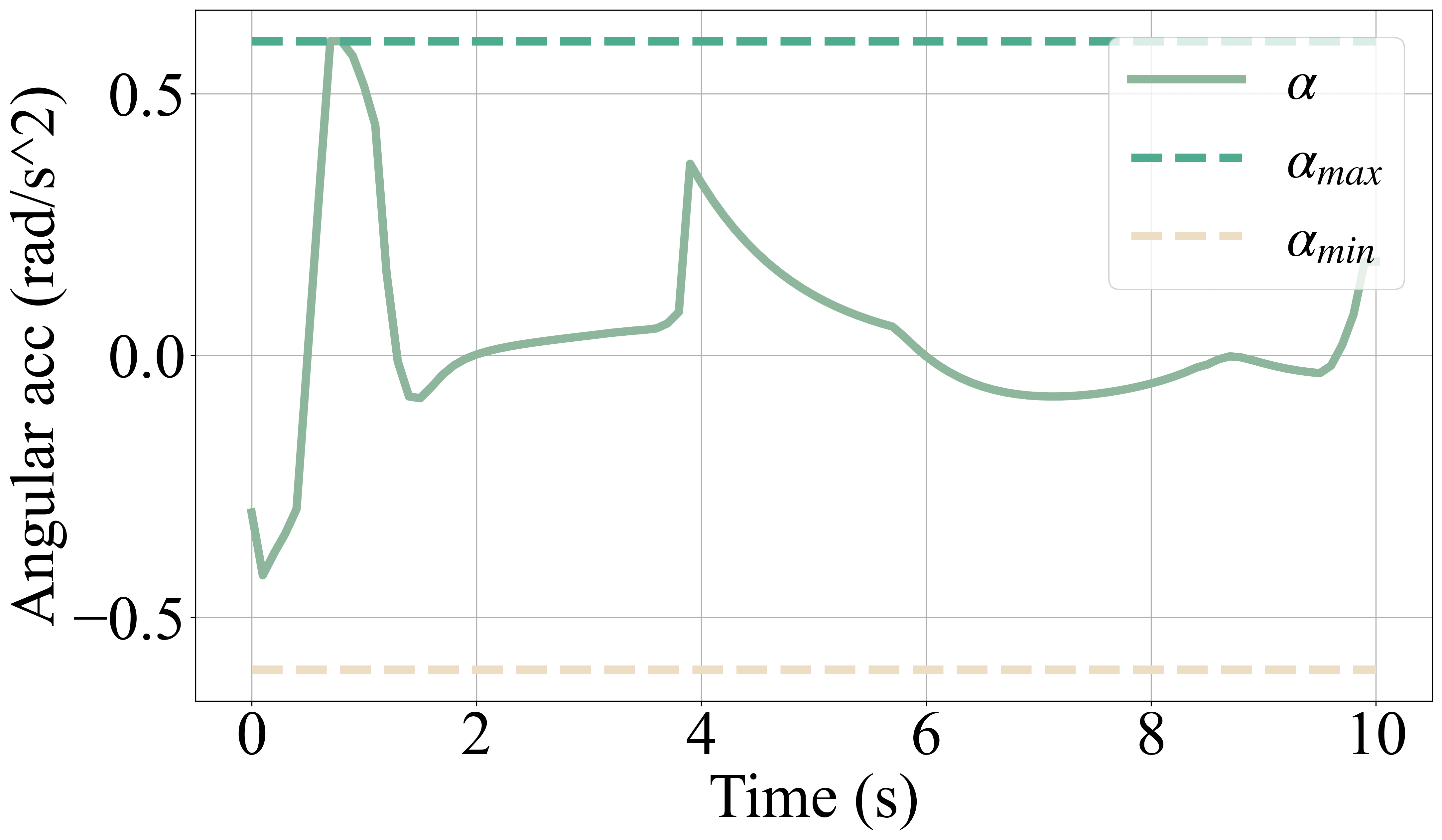}
        \caption{VOCBF}
        \label{subfig:control_w}
    \end{subfigure}
    \centering
    \begin{subfigure}{0.45\linewidth}
        \centering
        \includegraphics[width=0.98\linewidth]{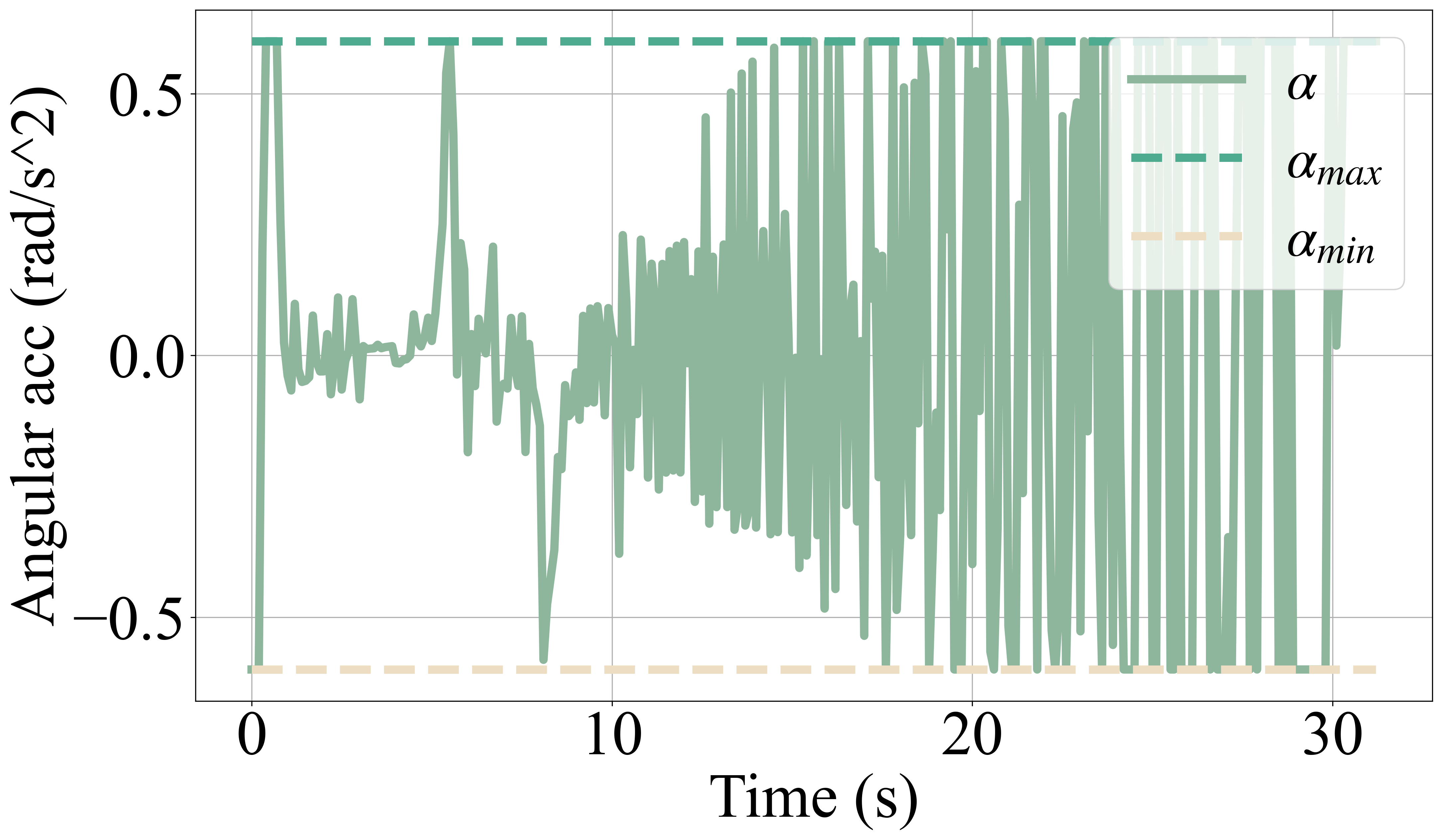}
        \caption{VO}
        \label{subfig:control_w_vo}
    \end{subfigure}
    \caption{Velocity and acceleration changes of VOCBF and VO.} 
    \label{fig:states_controls_changes}
\end{figure}
Furthermore, to evaluate the computational efficiency of CLF-VOCBF-MIQP, CLF-VOCBF-QPs, and CLF-VOCBF-QP-DecNet, we compared their performance across varying numbers of obstacles.
The runtime complexity bounds for these methods are provided in Table~\ref{tab:runtime}, while Table~\ref{tab:computation_time} shows the computation times for each method with varying numbers of obstacles.
CLF-VOCBF-QP-DecNet achieves the lowest computational cost since it is guided by the decision network and only solves a single QP problem. 
CLF-VOCBF-MIQP has the highest computational cost due to its reliance on solving a full MIQP problem with integer variables. 
CLF-VOCBF-QPs significantly reduce computational cost compared to CLF-VOCBF-MIQP by checking whether the feasible region $\mathcal{A}$ of each sub-optimization problem is empty before solving it.
For example, with two obstacles, CLF-VOCBF-QPs typically solves nine sub-optimization problems and selects the one with the minimum objective function, as shown in Fig.~\ref{fig:explain_qp}(\subref{subfig:case1}).
In cases like Fig.~\ref{fig:explain_qp}(\subref{subfig:case2}), after verifying the feasible regions, it may need to solve only one optimization problem, greatly reducing computational cost.
However, CLF-VOCBF-QPs still incur higher computational cost than CLF-VOCBF-QP-DecNet due to the need for feasibility checks and solving multiple sub-optimization problems.
Moreover, the robot’s velocity and acceleration changes in the presence of dynamic obstacles are shown in Fig.~\ref{fig:states_controls_changes}.

\subsection{Compared with Benchmarks} 
\label{sec:5-3}
In this section, we first compare our proposed CLF-VOCBF-QPs approach with the classical VO method~\cite{fiorini1998motion} to demonstrate that VOCBF provides stronger safety assurances.
CLF-VOCBF-QPs is used as it accurately represents the solution to the original MIQP problem, whereas CLF-VOCBF-QP-DecNet may occasionally deviate due to potential inaccuracies in the decision network.
The test scenario involves two dynamic obstacles with constant velocities of $(-0.4 \, \si[per-mode=symbol]{\metre\per\second}, -0.15 \, \si[per-mode=symbol]{\metre\per\second})$ and $(0.0 \, \si[per-mode=symbol]{\metre\per\second}, 0.2 \, \si[per-mode=symbol]{\metre\per\second})$.
The robot also aims to move from $(0 \, \si[per-mode=symbol]{\metre}, 4 \, \si[per-mode=symbol]{\metre})$ to $(12 \, \si[per-mode=symbol]{\metre}, 10 \, \si[per-mode=symbol]{\metre})$.
The navigation processes for both methods are shown in Fig.~\ref{fig:cmp_vo}, where the initial positions of the obstacles are the same for both methods, as indicated by the black squares.
Since the robot takes different amounts of time to reach the target position in each method, the obstacles travel different distances, leading to variations in their trajectories, as depicted in Fig.~\ref{fig:cmp_vo}(\subref{subfig:vocbf_cmp}) and Fig.~\ref{fig:cmp_vo}(\subref{subfig:vo_cmp}).
As a result, the final positions of the obstacles are determined by the respective completion times of the robot's navigation.
CLF-VOCBF-QPs, guided by CLFs, moves the robot efficiently toward its target while maintaining a larger distance from obstacles than the prescribed safe margin.
In contrast, VO results in a minimum robot-obstacle distance smaller than the safe margin, indicating a potential collision, as illustrated in Fig.~\ref{fig:cmp_vo}(\subref{subfig:vo_cmp}).
Additionally, VO is designed for robots controlled by horizontal and vertical velocities, which leads to oscillatory changes in states and control inputs when applied to the acceleration-controlled model~\eqref{eq:affine_robot_model}, as shown in Fig.~\ref{fig:states_controls_changes}.
In contrast, CLF-VOCBF-QPs ensures smoother variations in both states and control inputs, further demonstrating its effectiveness.
\begin{figure}
    \centering
    \begin{subfigure}{0.45\linewidth}
        \centering
        \includegraphics[width=0.98\linewidth]{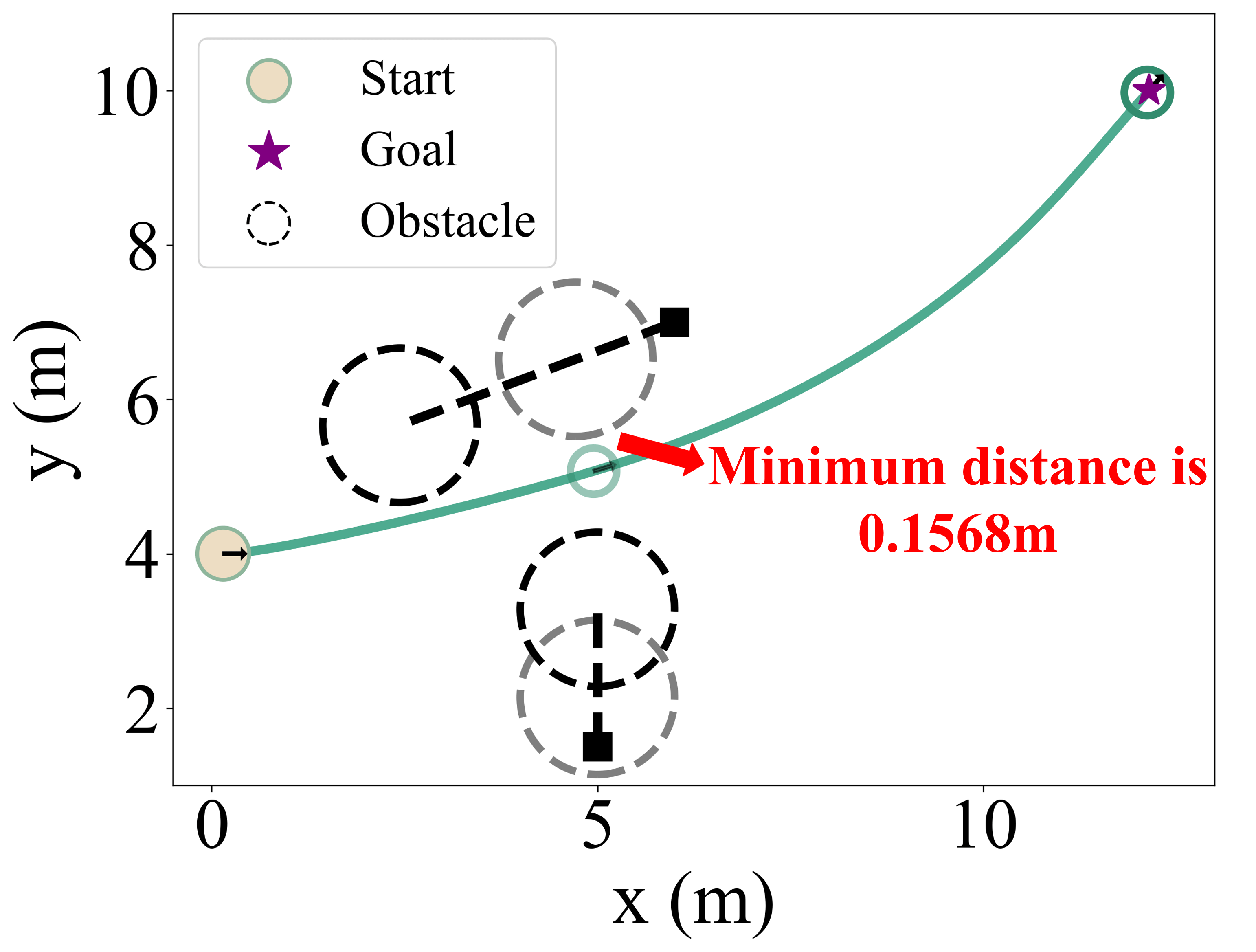}
        \caption{VOCBF}
        \label{subfig:vocbf_cmp}
    \end{subfigure}
    \centering
    \begin{subfigure}{0.45\linewidth}
        \centering
        \includegraphics[width=0.98\linewidth]{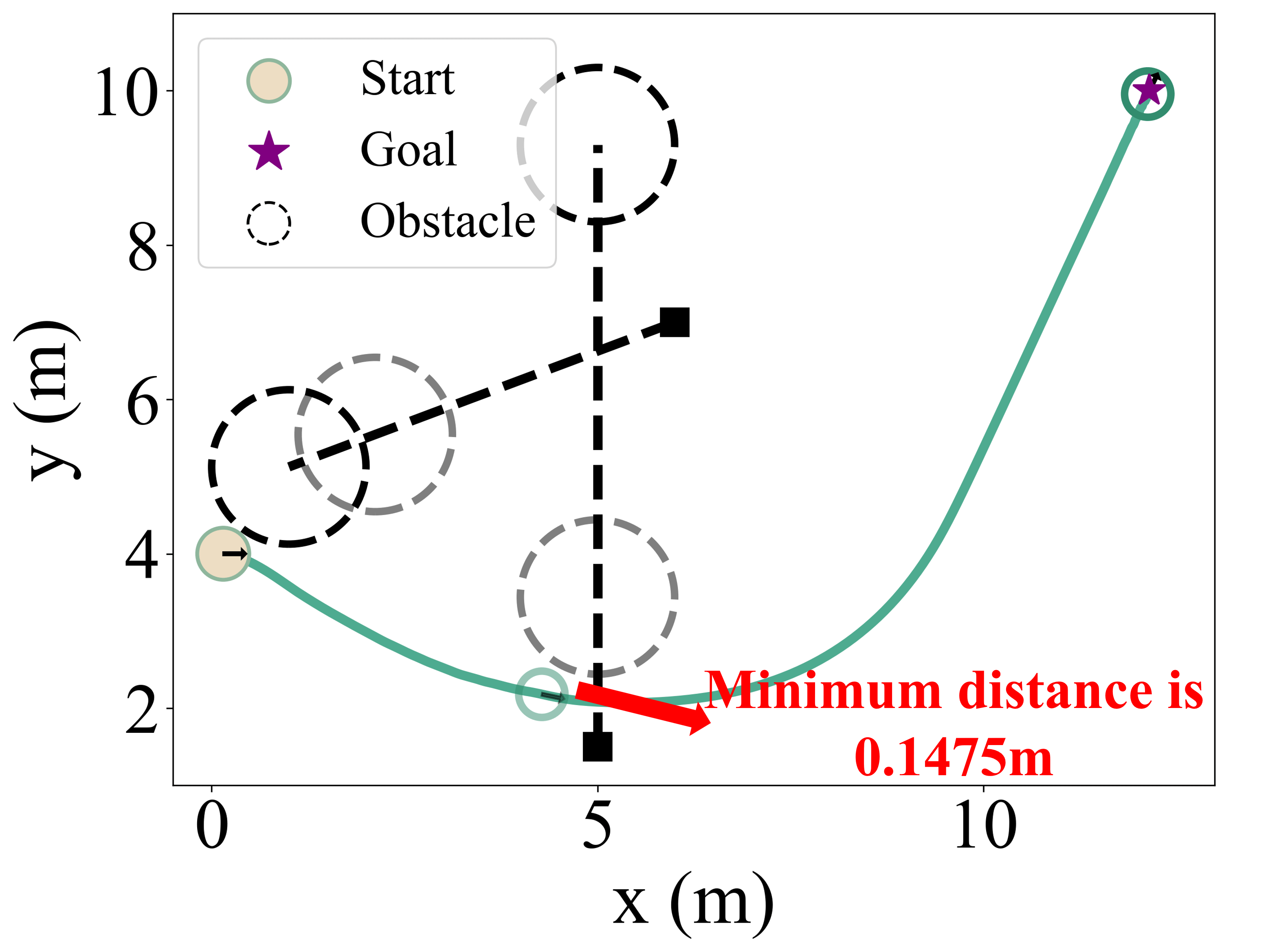}
        \caption{VO}
        \label{subfig:vo_cmp}
    \end{subfigure}
    \caption{Comparison of VOCBF and VO, and VOCBF provides more robust safety guarantees than VO.} 
    \label{fig:cmp_vo}
\end{figure}

We also compare our proposed approach with HOCBF~\cite{xiao2022high}, which is constructed based on the Euclidean distance as~\eqref{eq:hocbf}, to demonstrate its better performance in dynamic obstacle avoidance.
To specifically evaluate the differences in obstacle avoidance performance between VOCBF and HOCBF, all other constraints in the optimization problem remain consistent, with only the substitution of the VOCBF constraint~\eqref{eq:cons_cbf} by the corresponding HOCBF constraint. 
The parameters for HOCBF are chosen according to~\cite{xiao2022high} and are listed in Table~\ref{tab:simulation_params}.

A total of 600 random tests are conducted to evaluate the performance of CLF-VOCBF-QPs, CLF-VOCBF-QP-DecNet, and the SOTA HOCBF.
The evaluation is based on three key performance metrics:
\begin{enumerate}
    \item \textbf{Deadlock Rate:} The percentage of cases where the robot gets stuck during navigation without collisions, indicating the method's navigation efficiency.
    \item \textbf{Completion Rate:} The percentage of cases where the robot successfully reaches its target without collisions or deadlock, reflecting the method's overall effectiveness in collision avoidance and navigation.
    \item \textbf{Infeasible Rate:} The percentage of cases where the optimization problem is infeasible, showing how often the method struggles with feasibility.
\end{enumerate}

\begin{table}
\caption{Performance evaluation w.r.t. different methods.}
\label{tab:eva}
\centering
\resizebox{\columnwidth}{!}{%
\begin{tabular}{cccc}
\hline
\multirow{2}{*}{Methods} & \multicolumn{3}{c}{Evaluation Metrics ($\%$)}             \\ \cline{2-4} 
                         & Deadlock Rate & Completion Rate & Infeasible Rate \\ \hline
CLF-VOCBF-QPs            & 3             & \textbf{91}     & \textbf{6}   \\
CLF-VOCBF-QP-DecNet      & 2             & 85              & 13           \\
HOCBF                    & 3             & 86              & 11           \\ \hline
\end{tabular}
}
\end{table}
In the random scenario setup, the robot and two dynamic obstacles are placed in an environment bounded by $(0.0 \, \si[per-mode=symbol]{\metre}, 0.0 \, \si[per-mode=symbol]{\metre})$ and $(15.0 \, \si[per-mode=symbol]{\metre}, 15.0 \, \si[per-mode=symbol]{\metre})$.
The robot's radius ranges from $0.2 \, \si[per-mode=symbol]{\metre}$ to $0.7 \, \si[per-mode=symbol]{\metre}$, while obstacles have velocities randomly selected between $(-1.0 \, \si[per-mode=symbol]{\metre\per\second}, 1.0 \, \si[per-mode=symbol]{\metre\per\second})$ and radii ranging from $0.1 \, \si[per-mode=symbol]{\metre}$ to $1.5 \, \si[per-mode=symbol]{\metre}$. 
The performance evaluation results are presented in Table ~\ref{tab:eva}.
CLF-VOCBF-QPs achieves the highest completion rate and the lowest infeasible rate among all methods. 
By splitting the MIQP problem into multiple sub-optimization problems and checking feasibility before solving them, this method effectively ensures successful navigation and collision avoidance while maintaining feasibility.
CLF-VOCBF-QP-DecNet has a slightly lower completion rate due to its reliance on the decision network for direction selection. 
Incorrect guidance from the decision network can lead to infeasibility in some scenarios, resulting in suboptimal performance.
HOCBF performs worse than CLF-VOCBF-QPs, with a lower completion rate and a higher infeasible rate. 
This is because HOCBF, while incorporating time-varying CBFs, is fundamentally distance-based and less effective for dynamic obstacle avoidance.

\begin{figure}
    \centering
    \begin{subfigure}{0.45\linewidth}
        \centering
        \includegraphics[width=0.98\linewidth]{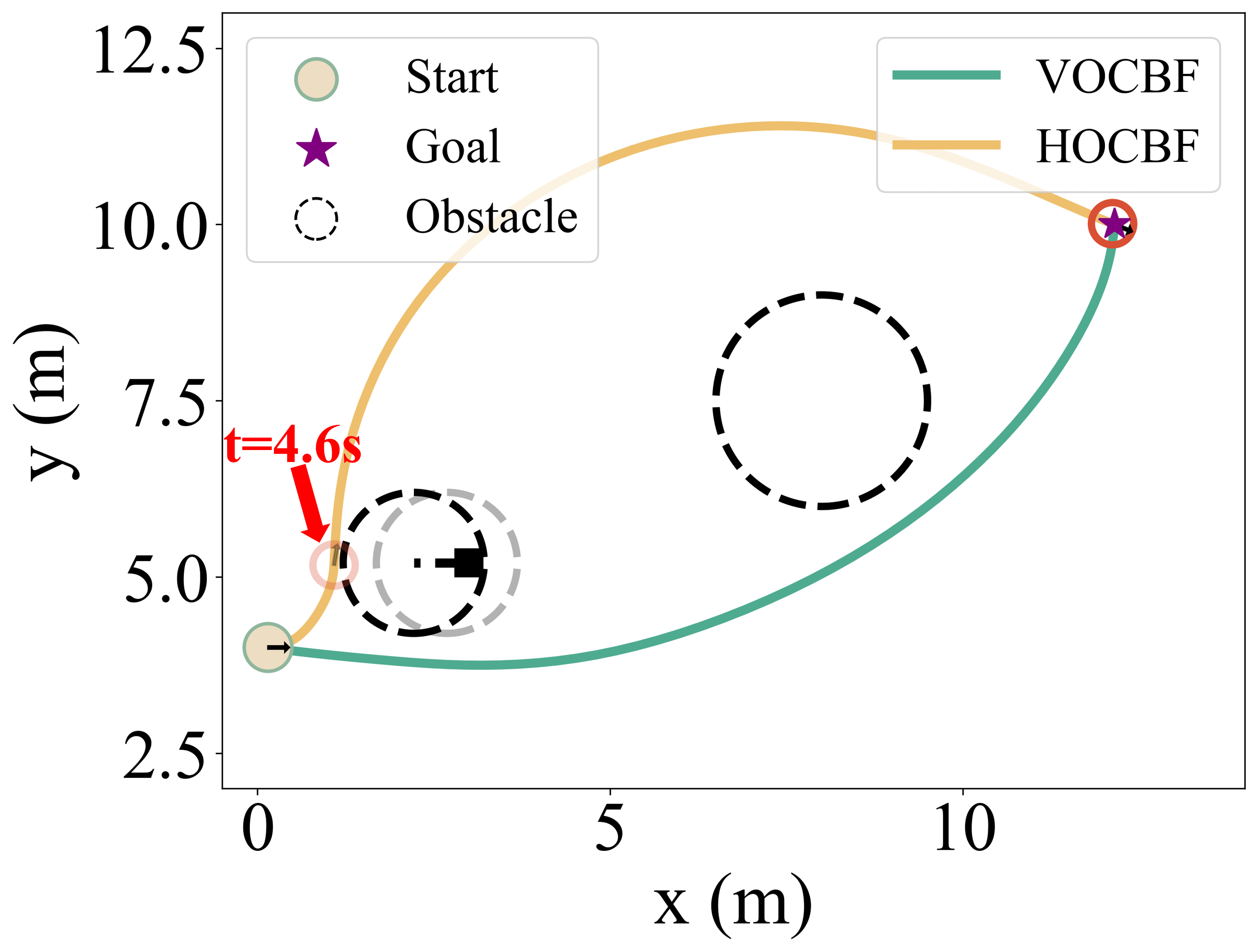}
        \caption{Low velocity of obstacle}
        \label{subfig:cmp_case1}
    \end{subfigure}
    \centering
    \begin{subfigure}{0.45\linewidth}
        \centering
        \includegraphics[width=0.98\linewidth]{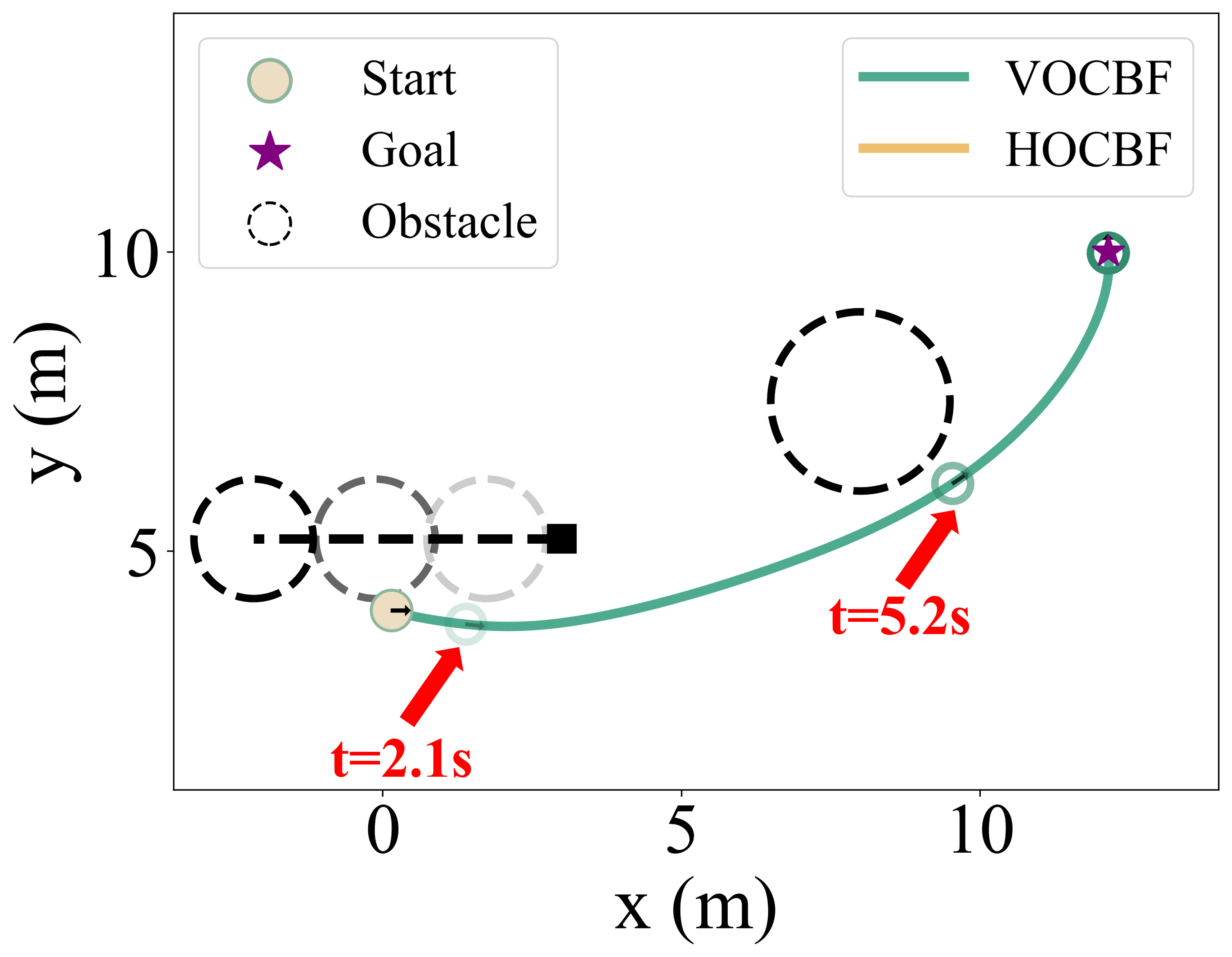}
        \caption{Fast velocity of obstacle}
        \label{subfig:cmp_case2}
    \end{subfigure}
    \caption{Comparison between VOCBF and HOCBF. In (a), both methods can navigate the robot to its destination while avoiding collisions with obstacles. In (b), HOCBF fails to avoid collisions with the dynamic obstacle due to the optimization problem being infeasible.} 
    \label{fig:compare}
\end{figure}
Additionally, we design two scenarios to thoroughly evaluate the performance of HOCBF and VOCBF in handling dynamic obstacles. Both scenarios share the same start and target positions and include one static obstacle and one dynamic obstacle with a constant velocity. 
The difference lies in the velocity of the dynamic obstacle: in the first scenario, it moves slowly at $(-0.05 \, \si[per-mode=symbol]{\metre\per\second}, 0.0 \, \si[per-mode=symbol]{\metre\per\second})$, whereas in the second, it moves faster at $(-0.6 \, \si[per-mode=symbol]{\metre\per\second}, 0.0 \, \si[per-mode=symbol]{\metre\per\second})$. 
Both approaches employ time-varying CBFs to avoid collisions with dynamic obstacles. 
In this analysis, we use CLF-VOCBF-QP-DecNet, as it achieves the same results as CLF-VOCBF-QPs while requiring lower computational cost.
\begin{table}
\caption{Comparison between VOCBF and HOCBF.}
\label{tab:cmp_hocbf_vocbf}
\centering
\begin{tabular}{cccc}
\hline
Method                 & Scenario & Reach Time (s) & Average Solving Time (ms) \\ \hline
\multirow{2}{*}{VOCBF} & Slow vel.    & 8.1    & 10.5          \\
                       & Fast vel.    & 7.9    & 10.6          \\ \hline
\multirow{2}{*}{HOCBF} & Slow vel.    & 15.7   & 10.2           \\
                       & Fast vel.    & N/A    & N/A           \\ \hline
\end{tabular}
\end{table}

When the dynamic obstacle moves slowly, both methods successfully guide the robot to its destination without collisions, as shown in Fig.~\ref{fig:compare}(\subref{subfig:cmp_case1}).
Additionally, our approach demonstrates faster convergence compared to HOCBF, enabling the robot to reach its target more quickly, as demonstrated in Table~\ref{tab:cmp_hocbf_vocbf}.
The HOCBF-based approach requires substantial relaxation of CLFs to ensure safety over a subset of the original safe set, resulting in larger relaxation variables and a longer reach time.
In contrast, CLF-VOCBF-QP-DecNet avoids excessive relaxation of CLFs, leading to a faster time.
When the dynamic obstacle moves at a faster velocity, only CLF-VOCBF-QP-DecNet successfully guides the robot to its target while ensuring safety, while HOCBF fails to avoid collisions with the dynamic obstacle due to the infeasibility of the optimization problem, as shown in Fig.~\ref{fig:compare}(\subref{subfig:cmp_case2}).
This indicates that our method, which constructs VOCBF in the velocity space, is better suited for dynamic obstacle avoidance.
\begin{remark}
\label{rem:vo}
VOCBF requires the robot's velocity to lie outside VO to ensure safety.
However, if the current distance between the robot and obstacle exceeds the sum of their radii while the robot's velocity falls within the VO, the situation is safe in terms of Euclidean distance but unsafe in terms of VO.
This highlights the conservative nature of VO-based approaches, as VOCBF also guarantees safety over a subset of the original safe set.
\end{remark}

\subsection{Navigation of Distributed Multi-Robot Systems}
\label{sec:5-4}
In this section, we demonstrate that our proposed approach can be applied to navigation and collision avoidance in distributed multi-robot systems. 
Each robot in the distributed multi-robot systems can independently make decisions to reach its destination while avoiding collisions with others.
As discussed in Section~\ref{sec:design_cbf}, the constraints of VO-based CBFs and their variants are equivalent, enabling the application of CLF-VOCBF-QP-DecNet.

We consider the circle scenario to evaluate our method: robots are uniformly distributed on a circle of radius $5 \, \si[per-mode=symbol]{\metre}$ centered at $(7\, \si[per-mode=symbol]{\metre}, 7 \, \si[per-mode=symbol]{\metre})$, and the initial and target positions of robots are symmetric along the center of the circle, as shown in Fig.~\ref{fig:multi_robot}(\subref{subfig:multi1}).
With CLF-VOCBF-QP-DecNet, each robot navigates independently toward its target while avoiding collisions, with robot positions at different times indicated by color gradients, as shown in Fig.~\ref{fig:multi_robot}.

In summary, while our proposed approach demonstrates promising results, challenges remain when scaling to very large multi-robot systems. As the number of robots increases, the prediction accuracy of the decision network tends to decline, and the computational cost of CLF-VOCBF-QPs grows exponentially. These limitations highlight that our method requires further refinement before it can be effectively applied to large-scale robotic systems.
\begin{figure}
    \centering
    \begin{subfigure}{0.45\linewidth}
        \centering
        \includegraphics[width=0.98\linewidth]{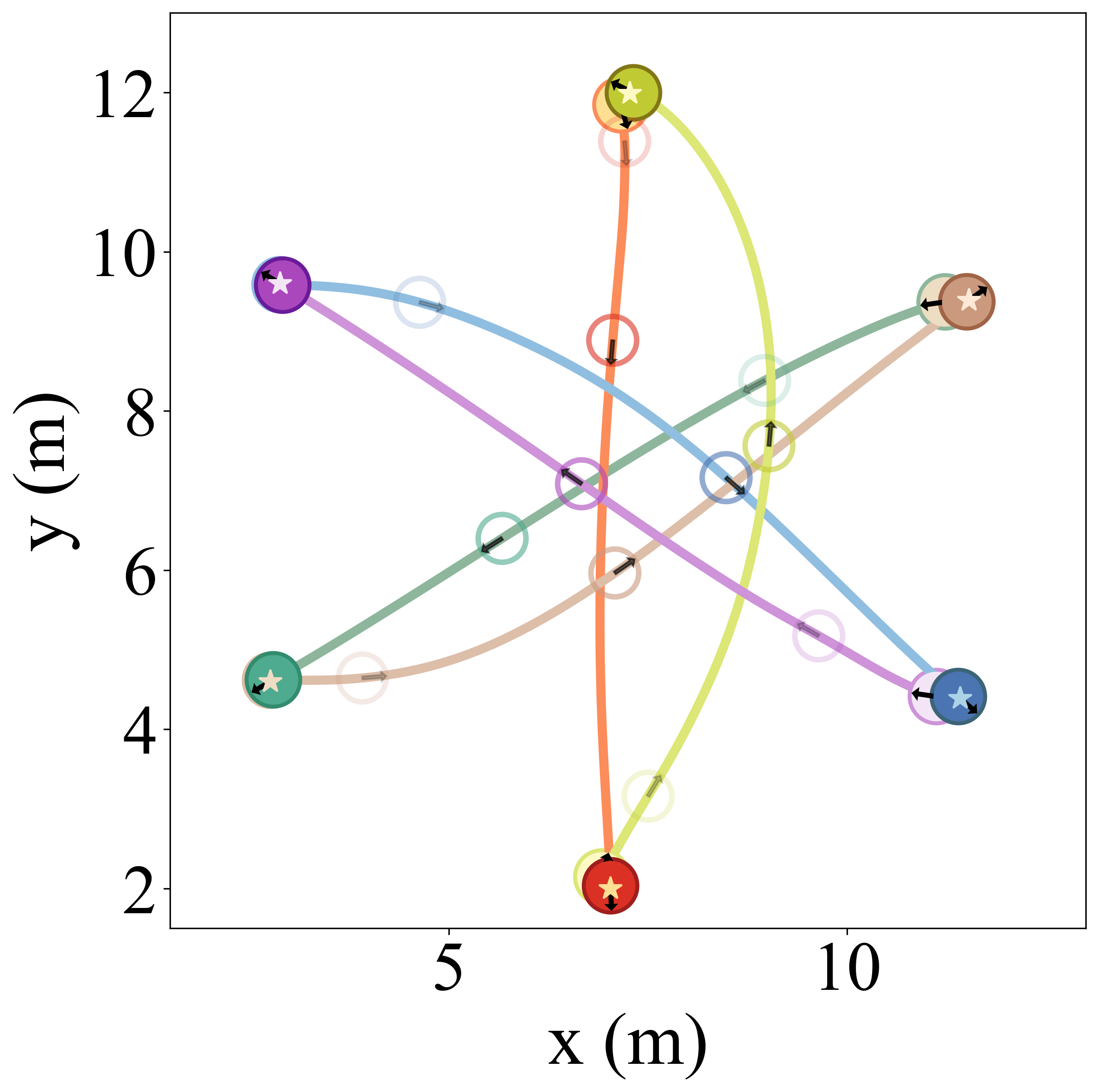}
        \caption{Six robots}
        \label{subfig:multi1}
    \end{subfigure}
    \centering
    \begin{subfigure}{0.45\linewidth}
        \centering
        \includegraphics[width=0.98\linewidth]{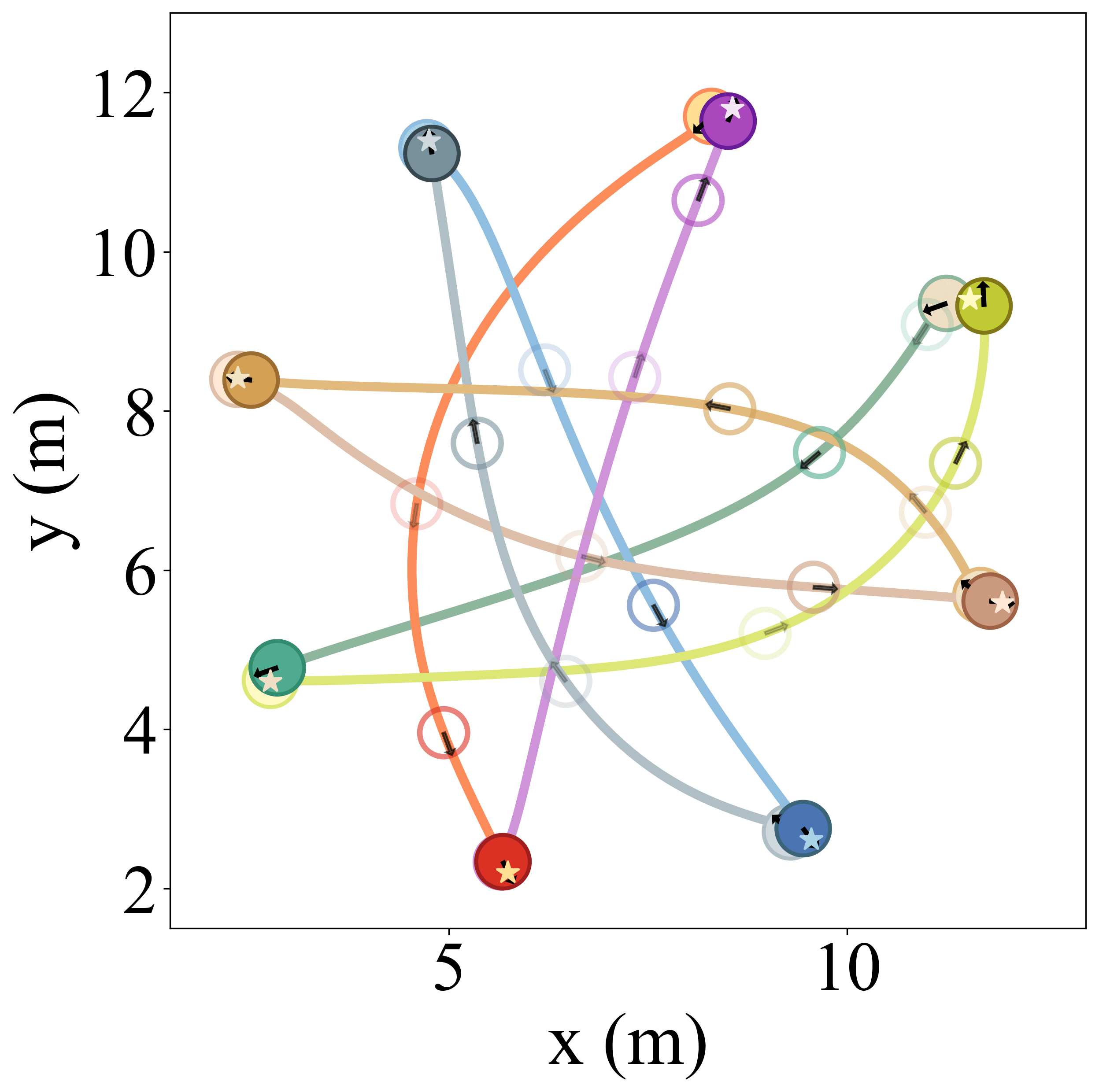}
        \caption{Eight robots}
        \label{subfig:multi2}
    \end{subfigure}
    \caption{Navigation of distributed multi-robot systems. The positions of different robots are shown as circles with different colors. Heading angle of the robot is illustrated by an arrow within the circle. Positions over time are illustrated using color gradients.}
    \label{fig:multi_robot}
\end{figure}
\section{Conclusions}
\label{sec:conclu}
In this paper, we propose a safety-critical controller CLF-VOCBF-QP for the acceleration-controlled unicycle model.
Designing CLFs and CBFs for this model often faces the challenge that not all control inputs explicitly appear in the constraints of CLFs and CBFs.
To address this issue, we propose designing state-feedback-based CLFs and constructing VOCBFs to ensure control inputs explicitly appear in the constraints.
Additionally, we formulate the constraint of CLFs and VOCBFs, along with other constraints of the robot's kinematics in the form of CLF-VOCBF-MIQP.
To efficiently solve the original MIQP problem, we split it into multiple sub-optimization problems and employ a decision network to guide the selection and resolution of a single sub-optimization problem.
Numerical simulations are conducted to validate that our approach can successfully achieve navigation and collision avoidance with both static and dynamic obstacles.
Furthermore, we also extend our approach to distributed multi-robot systems, enabling each robot to reach its destination while avoiding collisions with other robots.
Future work will focus on enhancing the accuracy of the decision network and extending our approach to large-scale multi-robot systems.

\bibliographystyle{IEEEtran}
\bibliography{reference}

\vspace{-3.0em}
\begin{IEEEbiography}
[{\includegraphics[width=1in,height=1.25in,clip,keepaspectratio]{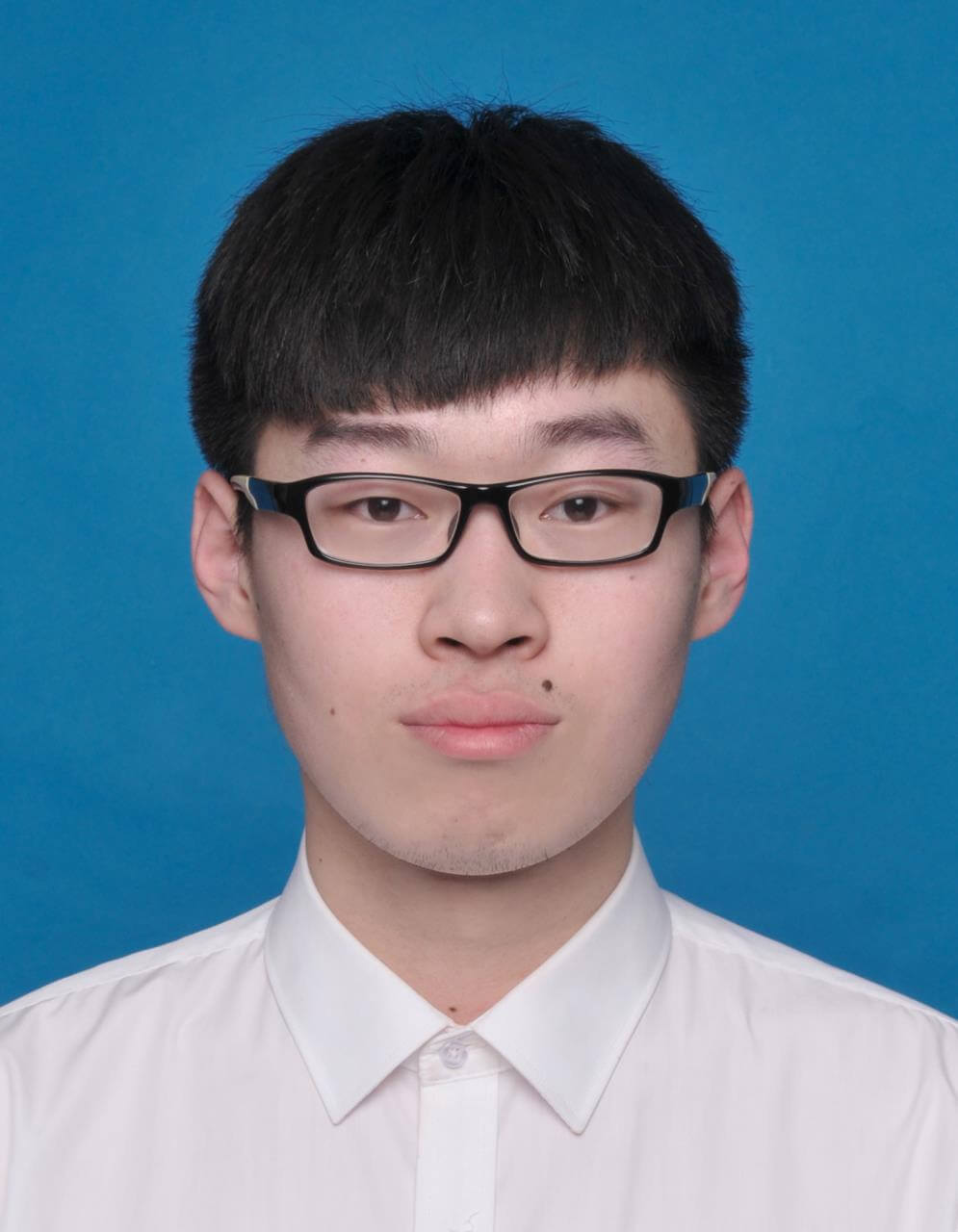}}]{Jihao Huang}
 received the B.Eng. degree in automation in 2020 from Hangzhou Dianzi University, Hangzhou, China. 
 
 He is currently pursuing the Ph.D. degree in control science and engineering at Zhejiang University, Hangzhou, China. His current research interests include control theory, distributed multi-robot systems and motion planning of robots.
\end{IEEEbiography}

\vspace{-3.0em}
\begin{IEEEbiography}
[{\includegraphics[width=1in,height=1.25in,clip,keepaspectratio]{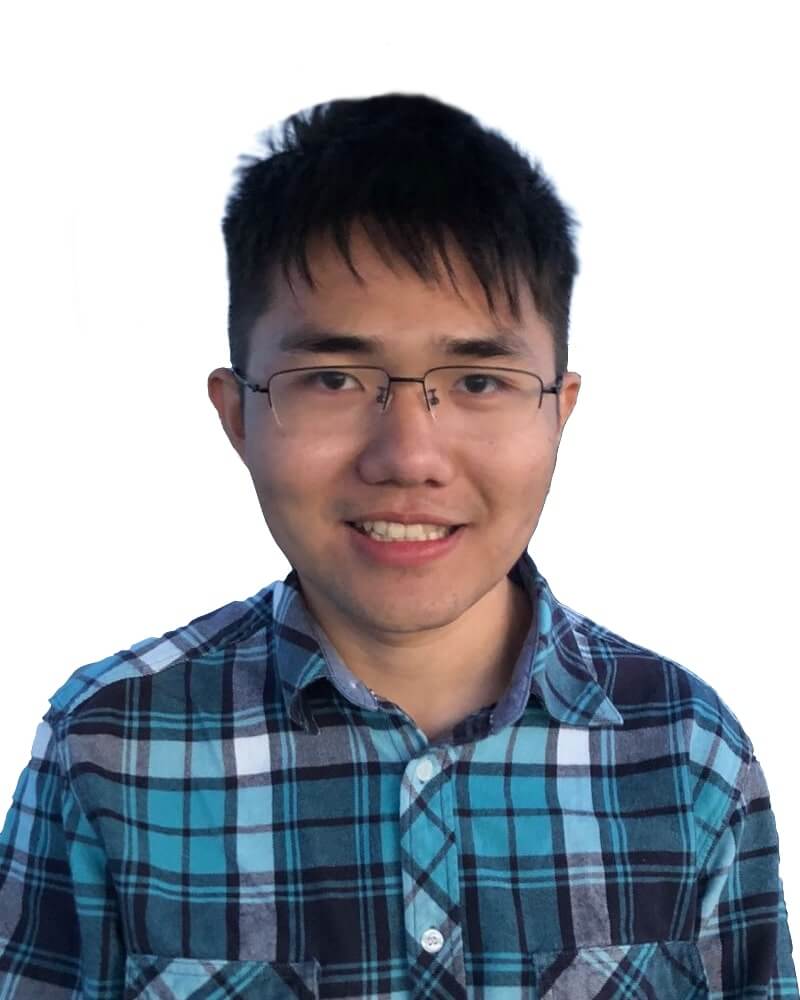}}]{Jun Zeng}
received his Ph.D. in Control and Robotics at the Department of Mechanical Engineering at University of California, Berkeley, USA in 2022 and Dipl. Ing. from Ecole Polytechnique, France in 2017, and a B.S.E degree from Shanghai Jiao Tong University (SJTU), China in 2016.
His research interests lie at the intersection of optimization, control, planning, and learning with applications on various robotics platforms.
\end{IEEEbiography}

\vspace{-3.0em}
\begin{IEEEbiography}
[{\includegraphics[width=1in,height=1.25in,clip,keepaspectratio]{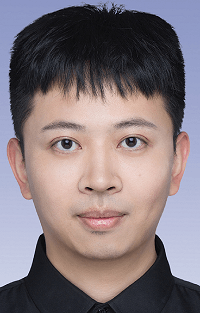}}]{Xuemin Chi}
received the B.Eng. degree in vehicle engineering in 2017 from Shenyang University of Technology, Shenyang, China, and the M.Sc. degree in vehicle engineering in 2019 from Dalian University of Technology, Dalian, China. 

He is currently working toward the Ph.D. degree in control engineering at Zhejiang University, Hangzhou, China.
His research interests include motion planning, safe model predictive control algorithms.
\end{IEEEbiography}

\vspace{-3.0em}
\begin{IEEEbiography}
[{\includegraphics[width=1in,height=1.25in,clip,keepaspectratio]{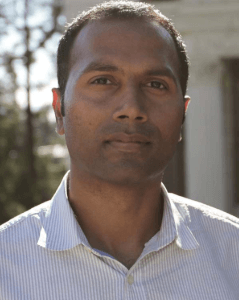}}]{Koushil Sreenath} received the Ph.D. degree in electrical engineering and computer science from the University of Michigan, Ann Arbor, MI, USA, in 2011. He was a Postdoctoral Scholar with the GRASP Laboratory, University of Pennsylvania, from 2011 to 2013, and an Assistant Professor at Carnegie Mellon University, from 2013 to 2017. He is currently an Associate Professor in mechanical engineering with UC Berkeley. His research interests include dynamic robotics, applied nonlinear control, and safety-critical control. He received the NSF CAREER Award, a Hellman Fellowship Award, Best Paper Award at the Robotics: Science and Systems (RSS), and the Google Faculty Research Award in Robotics.
\end{IEEEbiography}

\vspace{-3.0em}
\begin{IEEEbiography}
[{\includegraphics[width=1in,height=1.25in,clip,keepaspectratio]{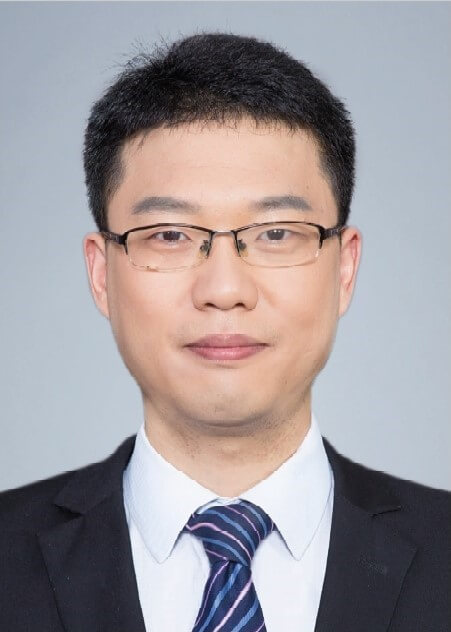}}]{Zhitao Liu}
 (M’13) received the B.S. degree from Shandong University, China, in 2005, and the Ph.D. degree in control science and engineering from Zhejiang University, Hangzhou, China, in 2010. 
 
 From 2011 to 2014, he was a Research Fellow with TUM CREATE, Singapore. He was an Assistant Professor from 2015 to 2016 and an Associate Professor from 2017 to 2021 in Zhejiang University, where he is currently a Professor with the Institute of Cyber-Systems and Control, Zhejiang University. 
 He is also a visiting professor at the Institute of Intelligence Science and Engineering, Shenzhen Polytechnic University.
 His current research interest include robust adaptive control, wireless transfer systems and energy management systems.
\end{IEEEbiography}

\vspace{-3.0em}
\begin{IEEEbiography}
[{\includegraphics[width=1in,height=1.25in,clip,keepaspectratio]{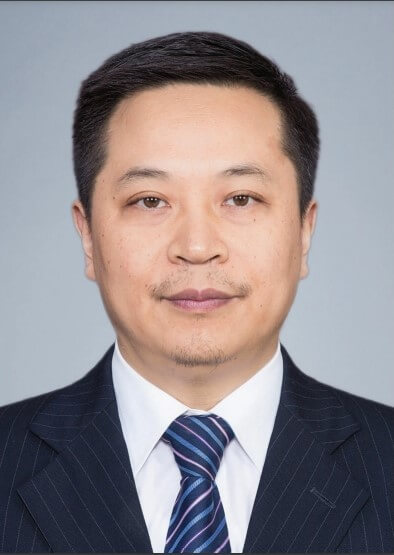}}]{Hongye Su}
 (SM’14) was born in 1969. He received the B.S. degree in industrial automation from the Nanjing University of Chemical Technology, Jiangsu, China, in 1990, and the M.S. and Ph.D. degrees in industrial automation from Zhejiang University, Hangzhou, China, in 1993 and 1995, respectively.

 From 1995 to 1997, he was a Lecturer with the Department of Chemical Engineering, Zhejiang University. From 1998 to 2000, he was an Associate Professor with the Institute of Advanced Process Control, Zhejiang University, where he is currently a Professor with the Institute of Cyber-Systems and Control. His current research interests include robust control, time-delay systems, and advanced process control theory and applications.
\end{IEEEbiography}

\end{document}